\documentclass[journal]{IEEEtran}

\newtheorem{theorem}{Theorem}[section]
\newtheorem{cor}[theorem]{Corollary}
\newtheorem{lem}[theorem]{Lemma}
\newtheorem{prop}[theorem]{Proposition}
\newtheorem{as}[theorem]{Assumption}
\newtheorem{problem}{Problem}
\newtheorem{definition}[theorem]{Definition}
\newtheorem{rem}[theorem]{Remark}
\newcommand{\set}[1]{\left\{#1\right\}}

\IEEEoverridecommandlockouts
\usepackage{graphicx} % for pdf, bitmapped graphics files
\usepackage{epsfig} % for postscript graphics files
\usepackage{amsmath}
\usepackage{amssymb}
\usepackage{epstopdf}
\usepackage{cite}
\usepackage[noend,ruled,linesnumbered]{algorithm2e}
\SetKwComment{Comment}{$\triangleright$\ } {}
\usepackage{multirow}
\usepackage{rotating}
\usepackage{subfigure}
\usepackage{color}
\usepackage{mysymbol}
\usepackage{url}
\usepackage{ltlfonts}	

\begin{document}

\title{Sampling-Based Optimal Control Synthesis for Multi-Robot Systems under Global Temporal Tasks}

\author{Yiannis~Kantaros,~\IEEEmembership{Student Member,~IEEE,} and Michael~M.~Zavlanos,~\IEEEmembership{Member,~IEEE}
\thanks{Yiannis Kantaros and Michael M. Zavlanos are with the Department of Mechanical Engineering and Materials Science, Duke University, Durham, NC 27708, USA. $\left\{\text{yiannis.kantaros,michael.zavlanos}\right\}$@duke.edu. 
This work is supported in part by NSF under grant IIS $\#1302283$ and by ONR under grant $\#N000141812374$.}
}
\maketitle
%\markboth{IEEE~Transactions~on~Automatic~Control,~June~2017.~Note:~First~Manuscript.}{Yiannis~Kantaros~and~Michael~M.~Zavlanos}	
%\thispagestyle{empty}
%\pagestyle{empty}
\begin{abstract}
This paper proposes a new optimal control synthesis algorithm for
multi-robot systems under global temporal logic tasks. 
Existing planning approaches under global temporal goals rely on graph
search techniques applied to a product automaton constructed among the
robots. In this paper, we propose a new sampling-based algorithm that
builds incrementally trees that approximate the state-space and
transitions of the synchronous product automaton. By approximating the
product automaton by a tree rather than representing it explicitly, we
require much fewer memory resources to store it and motion plans can be
found by tracing sequences of parent nodes without the need for
sophisticated graph search methods. This significantly increases the
scalability of our algorithm compared to existing optimal control synthesis
methods. We also show that the proposed algorithm is probabilistically
complete and asymptotically optimal. Finally, we present numerical
experiments showing that our approach can synthesize
optimal plans from product automata with billions of states, which is
not possible using standard optimal control synthesis algorithms or
model checkers.
\end{abstract}

\section{Introduction}\label{sec:introduction}
%choset2005principles
\IEEEPARstart{R}{obot} motion planning is a fundamental problem that has received considerable research attraction \cite{lavalle2006planning}. The basic motion planning problem consists of generating robot trajectories that reach a desired goal region starting from an initial configuration while avoiding obstacles. More recently, a new class of planning approaches have been developed that can handle a richer class of tasks than the classical point-to-point navigation, and can capture temporal and boolean specifications. Such tasks can be, e.g., sequencing or coverage \cite{fainekos2005temporal}, data gathering \cite{guo2017distributed}, intermittent communication \cite{kantaros2016distributedInterm}, or persistent surveillance \cite{leahy2016persistent}, and can be modeled using formal languages, such as Linear Temporal Logic (LTL) \cite{baier2008principles,clarke1999model}, that are developed in concurrency theory. Given a task described by a formal language, model checking algorithms can be employed to synthesize correct-by-construction controllers that satisfy the assigned tasks.

Control synthesis for mobile robots under complex tasks, captured by Linear Temporal Logic (LTL) formulas, build upon either bottom-up approaches when independent LTL expressions are assigned to robots \cite{kress2009temporal,kress2007s,bhatia2010sampling} or top-down approaches when a global LTL formula describing a collaborative task is assigned to a team of robots \cite{chen2011synthesis,chen2012formal}, as in this work. Common in the above works is that they rely on model checking theory \cite{baier2008principles,clarke1999model} to find paths that satisfy LTL-specified tasks, without optimizing task performance. Optimal control synthesis under local and global LTL specifications has been addressed in \cite{smith2011optimal,guo2015multi} and \cite{kloetzer2010automatic,ulusoy2013optimality,ulusoy2014optimal}, respectively. In top-down approaches \cite{kloetzer2010automatic,ulusoy2013optimality,ulusoy2014optimal}, optimal discrete plans are derived for every robot using the individual transition systems that capture robot mobility and a Non-deterministic B$\ddot{\text{u}}$chi Automaton (NBA) that represents the global LTL specification. Specifically, by taking the synchronous product among the transition systems and the NBA, a synchronous product automaton can be constructed. Then, representing the latter automaton as a graph and using graph-search techniques, optimal motion plans can be derived that satisfy the global LTL specification and optimize a cost function. As the number of robots or the size of the NBA increases, the state-space of the product automaton grows exponentially and, as a result, graph-search techniques become intractable. Consequently, these motion planning algorithms scale poorly with the number of robots and the complexity of the assigned task. 

To mitigate these issues, we propose an optimal control synthesis algorithm for multi-robot systems under global temporal specifications that avoids the explicit construction of the product among the transition systems and the NBA. Motivated by the $\text{RRT}^*$ algorithm \cite{karaman2011sampling}, we build incrementally through a B$\ddot{\text{u}}$chi-guided sampling-based algorithm directed trees that approximately represent the state-space and transitions among states of the synchronous product automaton. Specifically, first a tree is built  until a path from an initial to an \textit{accepting} state is constructed. This path corresponds to the prefix part of the motion plan and is executed once. Then, a new tree rooted at an accepting state is constructed in a similar way until a cycle-detection method discovers a loop around the root. This cyclic path corresponds to the suffix part of the motion plan and is executed indefinitely. The advantage of the proposed method is that approximating the product automaton by a tree rather than representing it explicitly by an arbitrary graph, as existing works do, results in significant savings in resources both in terms of memory to save the associated data structures and in terms of computational cost in applying graph search techniques. In this way, our proposed model-checking algorithm scales much better compared to existing top-down approaches. Also, we show that the proposed LTL-based planning algorithm is probabilistically complete and asymptotically optimal. We present numerical simulations that show that the proposed approach can synthesize optimal motion plans from product automata with billions of states, which is impossible using existing optimal control synthesis algorithms or the off-the-shelf symbolic model checkers PRISM \cite{kwiatkowska2002prism} and NuSMV \cite{cimatti2002nusmv}. 

%However, the syntax of $\mu$-calculus is not user-friendly, since it is based on fixed point operators, compared to the most natural syntax and semantics of LTL that we use in this paper.
To the best of our knowledge, the most relevant works are presented in \cite{karaman2012sampling,vasile2013sampling,kantaros15asilomar,kantaros2017sampling}. In \cite{karaman2012sampling}, a sampling-based algorithm is proposed which builds incrementally a Kripke structure until it is expressive enough to generate a motion plan that satisfies a task specification expressed in deterministic $\mu$-calculus. Specifically, in \cite{karaman2012sampling}, since control synthesis from $\omega$-regular languages requires cyclic patterns, an RRG-like algorithm is employed to construct a Kripke structure. However, building arbitrary graph structures to represent transition systems, compromises scalability of temporal  logic planning methods since, as the number of samples increases, so does the density of the constructed graph increasing in this way the required resources to save the associated structure and search for optimal plans using graph search methods. Therefore, the algorithm proposed in \cite{karaman2012sampling} can be typically used only for single-robot motion planning problems in simple environments and for LTL tasks associated with small NBA. Motivated by this limitation, in \cite{vasile2013sampling}, a sampling-based temporal logic path planning algorithm is proposed, that also builds upon the RRG algorithm, but constructs incrementally sparse graphs representing transition systems that are then used to construct a product automaton. Then, correct-by-construction discrete plans are synthesized applying graph search methods on the product automaton. However, similar to \cite{karaman2012sampling}, as the number of samples increases, the sparsity of the constructed graph is lost and as a result this method does not scale well to large planning problems either. Common in the works \cite{karaman2012sampling,vasile2013sampling} is that a discrete abstraction of the environment is built until it becomes expressive enough to generate a motion plan that satisfies the LTL specification. To the contrary, our proposed sampling-based approach, given a discrete abstraction of the environment \cite{conner2003composition,belta2004constructing,belta2005discrete,kloetzer2006reachability,boskos2015decentralized}, builds trees, instead of arbitrary graphs, to approximate the product automaton. Therefore, it is more economical in terms of memory requirements and does not require the application of expensive graph search techniques to find the optimal motion plan, but instead it tracks sequences of parent nodes starting from desired accepting states. In this way, we can handle more complex planning problems with more robots and LTL tasks that correspond to larger NBA, compared to the ones that can be solved using the approach in \cite{vasile2013sampling}. Moreover, we show that our proposed planning algorithm is asymptotically optimal which is not the case in \cite{vasile2013sampling}. %Finally, notice that common in existing sampling-based algorithms \cite{karaman2011sampling}, used for point-to-point navigation, is that samples are drawn from a continuous state-space, which is not the case in this work; therefore, they cannot be applied straightforwardly in our problem setting.

In our previous work \cite{kantaros15asilomar,kantaros2017sampling}, we have proposed sampling-based planning algorithms for multi-robot systems under global temporal logic tasks. Specifically, \cite{kantaros15asilomar} transforms given transition systems that abstract robot mobility into trace-included transition systems with smaller state-spaces that are still rich enough to construct motion plans that satisfy the global LTL specification. However, this algorithm does not scale well with the number of robots, since it relies on the construction of a product automaton among all agents. A more tractable and memory-efficient approach is proposed in \cite{kantaros2017sampling} that builds trees incrementally, similar to the approach proposed here, that approximate the product automaton. Here, we extend the work in \cite{kantaros2017sampling} by improving the sampling process of the algorithm so that samples are drawn among the sets of nodes that are reachable from the current tree rather than drawn arbitrarily from the state-space of the product automaton, as in \cite{kantaros2017sampling}. Since the state-space of the product automaton can be arbitrarily large drawing samples randomly can result in very slow growth of the tree, since these samples are not necessarily reachable from the current tree, as we show through numerical experiments. As in \cite{kantaros2017sampling}, here too we show that the proposed algorithm is probabilistically complete and asymptotically optimal. Nevertheless, the completeness and optimality proofs are entirely different due to the different sampling process. To the best of our knowledge, this is the first optimal control synthesis algorithm for global temporal task specifications that is probabilistically complete, asymptotically optimal, and scalable, as it is resource efficient both in terms of memory requirements and computational time.

The rest of the paper is organized as follows. In Section \ref{sec:prel}, we provide a brief overview of LTL and in Section \ref{sec:problem} we present the problem formulation. In Section \ref{sec:solution} we describe our proposed sampling-based planning algorithm and we examine its correctness and optimality in Section \ref{sec:corr}. Numerical experiments are presented in Section \ref{sec:sim}.

%To mitigate these issues, we propose a novel technique that approximately decomposes the global LTL formula into local ones and assigns them to robots. Since the approximate decomposition of the global LTL formula can result in conflicting robot behaviors we develop a distributed conflict resolution scheme that generates discrete motion plans for every robot based on the assigned local LTL expressions. Then  robot mobility along the generated discrete motion plans is performed by a continuous controller giving rise to a hybrid robot system. By appropriately introducing delays in the execution of the generated motion plans we show the proposed controllers can also be executed in an asynchronous fashion while ensuring that there are no deadlocks and the global LTL is satisfied. In contrast, most relevant literature assumes that robot control is performed in a synchronous way \cite{chen2011synthesis,kloetzer2008distributed}. To the best of our knowledge, although specific to the problem under consideration, this is the first distributed, scalable, and asynchronous LTL-based framework for the coordination of teams of multiple robots.

\section{Preliminaries}\label{sec:prel} 
In this section we formally describe Linear Temporal Logic (LTL) by presenting its syntax and semantics. Also, we briefly review preliminaries of automata-based LTL model checking. A detailed overview of this theory can be found in \cite{baier2008principles}.

Linear temporal logic is a type of formal logic whose basic ingredients are a set of atomic propositions $\mathcal{AP}$, the boolean operators, i.e., conjunction $\wedge$, and negation $\neg$, and two temporal operators, next $\bigcirc$ and until $\mathcal{U}$. LTL formulas over a set $\mathcal{AP}$ can be constructed based on the following grammar: $\phi::=\text{true}~|~\pi~|~\phi_1\wedge\phi_2~|~\neg\phi~|~\bigcirc\phi~|~\phi_1~\mathcal{U}~\phi_2$, where $\pi\in\mathcal{AP}$. For the sake of brevity we abstain from presenting the derivations of other Boolean and temporal operators, e.g., \textit{always} $\square$, \textit{eventually} $\lozenge$, \textit{implication} $\Rightarrow$, which can be found in \cite{baier2008principles}. 

An infinite \textit{word} $\sigma$ over the alphabet $2^{\mathcal{AP}}$ is defined as an infinite sequence  $\sigma=\pi_0\pi_1\pi_2\dots\in (2^{\mathcal{AP}})^{\omega}$, where $\omega$ denotes an infinite repetition and $\pi_k\in2^{\mathcal{AP}}$, $\forall k\in\mathbb{N}$. The language $\texttt{Words}(\phi)=\left\{\sigma\in (2^{\mathcal{AP}})^{\omega}|\sigma\models\phi\right\}$ is defined as the set of words that satisfy the LTL formula $\phi$, where $\models\subseteq (2^{\mathcal{AP}})^{\omega}\times\phi$ is the satisfaction relation.

Any LTL formula $\phi$ can be translated into a Nondeterministic B$\ddot{\text{u}}$chi Automaton (NBA) over $(2^{\mathcal{AP}})^{\omega}$ denoted by $B$ \cite{vardi1986automata}, which is defined as follows:
\begin{definition}
A \textit{Nondeterministic B$\ddot{\text{u}}$chi Automaton} (NBA) $B$ over $2^{\mathcal{AP}}$ is defined as a tuple $B=\left(\ccalQ_{B}, \ccalQ_{B}^0,\Sigma,\rightarrow_B,\mathcal{Q}_B^F\right)$, where $\ccalQ_{B}$ is the set of states, $\ccalQ_{B}^0\subseteq\ccalQ_{B}$ is a set of initial states, $\Sigma=2^{\mathcal{AP}}$ is an alphabet, $\rightarrow_{B}\subseteq\ccalQ_{B}\times \Sigma\times\ccalQ_{B}$ is the transition relation, and $\ccalQ_B^F\subseteq\ccalQ_{B}$ is a set of accepting/final states. 
\end{definition}

An \textit{infinite run} $\rho_B$ of $B$ over an infinite word $\sigma=\pi_0\pi_1\pi_2\dots$, where $\pi_k\in\Sigma$, for all $ k\in\mathbb{N}$, is a sequence $\rho_B=q_B^0q_B^1q_B^2\dots$ such that $q_B^0\in\ccalQ_B^0$ and $(q_B^{k},\pi_k,q_B^{k+1})\in\rightarrow_{B}$, $\forall k\in\mathbb{N}$. An infinite run $\rho_B$ is called \textit{accepting} if $\texttt{Inf}(\rho_B)\cap\ccalQ_B^F\neq\varnothing$, where $\texttt{Inf}(\rho_B)$ represents the set of states that appear in $\rho_B$ infinitely often. The words $\sigma$ that result in an accepting run of $B$ constitute the accepted language of $B$, denoted by $\ccalL_B$. It is proven \cite{baier2008principles} that the accepted language of $B$ is equivalent to the words of $\phi$, i.e., $\ccalL_B=\texttt{Words}(\phi)$. 

\section{Problem Formulation}\label{sec:problem}

Consider $N$ mobile robots that evolve in a complex workspace $\mathcal{W}\subset\mathbb{R}^d$ according to the following dynamics: $\dot{\bf{x}}_i(t)=f_i({\bf{x}}_i(t),{\bf{u}}_i(t))$,
%\begin{equation}\label{eq:dynamics}
%\dot{\bf{x}}_i(t)=f_i({\bf{x}}_i(t),{\bf{u}}_i(t)),
%\end{equation}
where ${\bf{x}}_i(t)$ and ${\bf{u}}_i(t)$ are the position and the control input associated with robot $i$, $i\in\{1,\dots,N\}$. We assume that there are $W$ disjoint regions of interest in $\mathcal{W}$ that are worth investigation or surveillance. The $j$-th region is denoted by $\ell_j$ and it can be of any arbitrary shape. Given the robot dynamics, robot mobility in the workspace $\ccalW$ can be represented by a weighted transition system (wTS) obtained through an abstraction process \cite{conner2003composition,belta2004constructing,belta2005discrete,kloetzer2006reachability,boskos2015decentralized}.The wTS for robot $i$ is defined as follows
\begin{definition}[wTS]
A \textit{weighted Transition System} (wTS) for robot $i$, denoted by $\text{wTS}_{i}$ is a tuple $\text{wTS}_{i}=\left(\mathcal{Q}_{i}, q_{i}^0,\rightarrow_{i}, w_i, \mathcal{AP}_i,L_{i}\right)$ where: 
(a) $\mathcal{Q}_{i}=\bigcup_{j=1}^W\{q_{i}^{\ell_j}\}$ is the set of states, where a state $q_{i}^{\ell_j}$ indicates that robot $i$ is at location $\ell_j$; (b) $q_{i}^0\in\mathcal{Q}_{i}$ is the initial state of robot $i$; $\rightarrow_{i}\subseteq\mathcal{Q}_{i}\times\mathcal{Q}_{i}$ is the transition relation for robot $i$. Given the robot dynamics, if there is a control input ${\bf{u}}_i$ that can drive robot $i$ from location $\ell_j$ to $\ell_e$, then there is a transition from state $q_i^{\ell_j}$ to $q_i^{\ell_e}$ denoted by $(q_i^{\ell_j}, q_i^{\ell_e})\in\rightarrow_i$; (c) $w_{i}:\mathcal{Q}_{i}\times\mathcal{Q}_{i}\rightarrow \mathbb{R}_+$\footnote{$\mathbb{R}_+$ and $\mathbb{N}_+$ stand for the positive real and natural numbers, respectively.} is a cost function that assigns weights/cost to each possible transition in wTS. For example, such costs can be associated with the distance that needs to be traveled by robot $i$ in order to move from state $q_i^{\ell_j}$ to state $q_i^{\ell_k}$; 
(d) $\mathcal{AP}_i=\bigcup_{j=1}^W\{\pi_{i}^{\ell_j}\}$ is the set of atomic propositions, where $\pi_{i}^{\ell_j}$ is true if robot $i$ is inside region $\ell_j$ and false otherwise; and (e) $L_{i}:\mathcal{Q}_{i}\rightarrow 2^{\mathcal{AP}_i}$ is an observation/output relation giving the set of atomic propositions that are satisfied in a state, i.e., $L_i(q_i^{\ell_j})=\pi_i^{\ell_j}$. 
\label{defn:wTS}
\end{definition} 

%\begin{figure}[t]
%\centering
%  \includegraphics[width=0.5\linewidth]{space2ts_v2.eps}
%  \caption{Graphical depiction of a wTS that abstracts robot mobility in an indoor environment. Black disks stand for the states of wTS, red edges capture transitions among states and numbers on these edges represent the cost $w_i$ for traveling from one state to another one.}
%  \label{fig:wts}
%\end{figure}

Given the definition of the wTS, we can define the synchronous \textit{Product Transition System} (PTS), which captures all the possible combinations of robots' states in their respective $\text{wTS}_{i}$, as follows \cite{baier2008principles}:

\begin{definition}[PTS]
Given $N$ transition systems $\text{wTS}_i=(\mathcal{Q}_{i}, q_{i}^0, \rightarrow_{i},w_i,\mathcal{AP}_i, L_{i})$, the \textit{product transition system} $\text{PTS}=\text{wTS}_{1}\otimes\text{wTS}_{2}\otimes\dots\otimes\text{wTS}_{N}$ is a tuple $\text{PTS}=(\mathcal{Q}_{\text{PTS}}, q_{\text{PTS}}^0,\longrightarrow_{\text{PTS}},w_{\text{PTS}},\mathcal{AP},L_{\text{PTS}})$ where
(a) $\mathcal{Q}_{\text{PTS}}=\mathcal{Q}_{1}\times\mathcal{Q}_{2}\times\dots\times\mathcal{Q}_{N}$ is the set of states;
(b) $q_{\text{PTS}}^0=(q_{1}^0,q_{2}^0,\dots,q_{N}^0)\in\mathcal{Q}_{\text{PTS}}$ is the initial state,
	%\item $\ccalA_{\text{TS}_{\bbv_i}}=\ccalA_{ij_1}\times\ccalA_{ij_2}\times\dots\times\ccalA_{ij_{|\ccalN_i|}}$ is a set of actions,
(c) $\longrightarrow_{\text{PTS}}\subseteq\mathcal{Q}_{\text{PTS}}\times\mathcal{Q}_{\text{PTS}}$ is the transition relation defined by the rule\footnote{The notation of this rule is along the lines of the notation used in \cite{baier2008principles}. In particular, it means that if the proposition above the solid line is true, then so does the proposition below the solid line.} $\frac{\bigwedge _{\forall i}\left(q_{i}\rightarrow_{i}q_{i}'\right)}{q_{\text{PTS}}\rightarrow_{\text{PTS}}q_{\text{PTS}}'}$, where with slight abuse of notation $q_{\text{PTS}}=(q_{1},\dots,q_{N})\in\mathcal{Q}_{\text{PTS}}$, $q_{i}\in\mathcal{Q}_{i}$. The state $q_{\text{PTS}}'$ is defined accordingly. In words, this transition rule says that there exists a transition from $q_{\text{PTS}}$ to $q_{\text{PTS}}'$ if there exists a transition from $q_i$ to $q_i'$ for all $i\in\left\{1,\dots,N\right\}$; (d) $w_{\text{PTS}}:\mathcal{Q}_{\text{PTS}}\times\mathcal{Q}_{\text{PTS}}\rightarrow \mathbb{R}_+$ is a cost function that assigns weights/cost to each possible transition in PTS, defined as $	w_{\text{PTS}}(q_{\text{PTS}},q_{\text{PTS}}')=\sum_{i=1}^N w_i(\Pi|_{\text{wTS}_{i}}q_{\text{PTS}},\Pi|_{\text{wTS}_{i}}q_{\text{PTS}}^{'})\geq0$, 
%	\begin{equation}\label{eq:wpts}
%	w_{\text{PTS}}(q_{\text{PTS}},q_{\text{PTS}}^{'})=\sum_{i=1}^N w_i(\Pi|_{\text{wTS}_{i}}q_{\text{PTS}},\Pi|_{\text{wTS}_{i}}q_{\text{PTS}}^{'}),
%	\end{equation}
where $q_{\text{PTS}}',q_{\text{PTS}}\in\mathcal{Q}_{\text{PTS}}$, and $\Pi|_{\text{wTS}_i}q_{\text{PTS}}$ stands for the projection of state $q_{\text{PTS}}$ onto the state space of $\text{wTS}_i$. The state $\Pi|_{\text{wTS}_i}q_{\text{PTS}}\in\mathcal{Q}_i$ is obtained by removing all states in $q_{\text{PTS}}$ that do not belong to $\mathcal{Q}_i$; (e) $\mathcal{AP}=\bigcup_{i=1}^N\mathcal{AP}_i$ is the set of atomic propositions; and, (f) $L_{\text{PTS}}=\bigcup_{\forall i}L_{i}: \mathcal{Q}_{\text{PTS}}\rightarrow2^{\mathcal{AP}}$ is an observation/output relation giving the set of atomic propositions that are satisfied at a state $q_{\text{PTS}}\in\mathcal{Q}_{\text{PTS}}$. 
\label{def:pts}
\end{definition} 

In what follows, we give definitions related to the $\text{PTS}$, that we will use throughout the rest of the paper. An \textit{infinite path} $\tau$ of a $\text{PTS}$ is an infinite sequence of states, $\tau=\tau(1)\tau(2)\tau(3)\dots$ such that $\tau(1)=q_{\text{PTS}}^0$, $\tau(k)\in\mathcal{Q}_{\text{PTS}}$, and $(\tau(k),\tau(k+1))\in\rightarrow_{\text{PTS}}$, $\forall k\in\mathbb{N}_+$, where $k$ is an index that points to the $k$-th entry of $\tau$ denoted by $\tau(k)$.
%\begin{defn}[Cost of infinite path]\label{def:cost}
%The cost of an infinite path $\tau_i$ is defined as $J(\tau_i)=\sum_{k=1}^{\infty}w_i(\tau_i(k),\tau_i(k+1))$. %\left\|\tau_i(k)-\tau_{i}(k+1)\right\|where with slight abuse of notation $\left\|\tau_i(k)-\tau_{i}(k+1)\right\|$ stands for the distance between location in the workspace associated with the states $\tau_i(k)$ and $\tau_i(k+1)$. 
%\end{defn}
A \textit{finite path} of a $\text{PTS}$ can be defined accordingly. The only difference with the infinite path is that a finite path is defined as a finite sequence of states of a $\text{PTS}$. Given the definition of the weights $w_{\text{PTS}}$ in Definition \ref{def:pts}, the \textit{cost} of a finite path $\tau$, denoted by $\hat{J}(\tau)$, can be defined as
\begin{equation}\label{eq:cost}
\hat{J}(\tau)=\sum_{k=1}^{|\tau|-1}w_{\text{PTS}}(\tau(k),\tau(k+1)).
\end{equation} 
In \eqref{eq:cost}, $|\tau|$ stands for the number of states in $\tau$. In words, the cost \eqref{eq:cost} captures the total cost incurred by all robots during the execution of the finite path $\tau$. 
%
%Throughout the paper, we make the following assumption for the cost function $J_f(\tau)$:\footnote{Note that such an assumption is commonly made in sampling-based algorithms, see e.g., \cite{karaman2011sampling,karaman2009sampling,karaman2012sampling}. Also, this assumption is reasonable to make, since it is valid for several selections of the weight $w_{\text{PTS}}$. For instance, this assumption holds if $w_{\text{PTS}}$ represents energy consumption or distance between locations associated with the states $\tau(k)$ and $\tau(k+1)$, which are commonly considered in robotic applications.}
%
%\begin{rem}[Cost function $\hat{J}(\cdot)$]
Notice that the cost function $\hat{J}(\cdot)$ is additive, i.e., $\hat{J}(\tau_1|\tau_2)=\hat{J}(\tau_1)+\hat{J}(\tau_2)$, where $\tau_1$ and $\tau_2$ are two finite paths of the PTS so that there is a feasible transition from the last state in $\tau_1$ to the first state in $\tau_2$, according to $\rightarrow_{\text{PTS}}$, and $|$ stands for concatenation. Also, since $\hat{J}(\cdot)$ is additive and $w_{\text{PTS}}(q_{\text{PTS}},q_{\text{PTS}}')\geq0$ by Definition \ref{def:pts}, we get that $\hat{J}(\cdot)$ is monotone i.e., $\hat{J}(\tau_1)\leq \hat{J}(\tau_1|\tau_2)$.
%\end{rem}

%\begin{as}[Additivity and Monoticity of the Cost function $\hat{J}(\cdot)$]\label{as:additive}
%The cost function $\hat{J}(\cdot)$ in \eqref{eq:cost} is monotone,i.e., 
%$\hat{J}(\tau_1)\leq \hat{J}(\tau_1|\tau_2)$, where $\tau_1$ and $\tau_2$ are two finite paths of the PTS and $\tau_1$, it holds that %$\hat{J}(\tau_1|\tau_2)=\hat{J}(\tau_1)+\hat{J}(\tau_2)$, where $|$ stands for the concatenation of paths. 
%additive, i.e., given any two finite paths $\tau_1$, $\tau_2$ of the PTS, it holds that $\hat{J}(\tau_1|\tau_2)=\hat{J}(\tau_1)+\hat{J}%(\tau_2)$, where $|$ stands for the concatenation of paths. 
%\end{as}

The \textit{trace} of an infinite path $\tau=\tau(1)\tau(2)\tau(3)\dots$ of a PTS, denoted by $\texttt{trace}(\tau)\in\left(2^{\mathcal{AP}}\right)^{\omega}$, where $\omega$ denotes infinite repetition, is an infinite word that is determined by the sequence of atomic propositions that are true in the states along $\tau$, i.e., $\texttt{trace}(\tau)=L(\tau(1))L(\tau(2))\dots$. 
%\end{defn}
%\begin{defn}[Motion Plan]\label{def:motionplan}

Given the $\text{PTS}$ and the NBA $B$ that corresponds to the LTL $\phi$, we can now define the \textit{Product B$\ddot{\text{u}}$chi Automaton} (PBA) $P=\text{PTS}\otimes B$ \cite{baier2008principles}, as follows:

\begin{definition}[PBA]\label{defn:pba}
Given the product transition system $\text{PTS}=(\mathcal{Q}_{\text{PTS}},\allowbreak q_{\text{PTS}}^0,\longrightarrow_{\text{PTS}},w_{\text{PTS}},\mathcal{AP},L_{\text{PTS}})$ and the NBA $B=(\mathcal{Q}_{B}, \mathcal{Q}_{B}^0,\Sigma,\rightarrow_B,\mathcal{Q}_{B}^{F})$, we can define the \textit{Product B$\ddot{\text{u}}$chi Automaton} $P=\text{PTS}\otimes B$ as a tuple $P=(\mathcal{Q}_P, \mathcal{Q}_P^0,\longrightarrow_{P},w_P,\mathcal{Q}_P^F)$ where
(a) $\mathcal{Q}_P=\mathcal{Q}_{\text{PTS}}\times\mathcal{Q}_{B}$ is the set of states; (b) $\mathcal{Q}_P^0=q_{\text{PTS}}^0\times\mathcal{Q}_B^0$ is a set of initial states;
(c) $\longrightarrow_{P}\subseteq\mathcal{Q}_P\times 2^{\mathcal{AP}}\times\mathcal{Q}_P$ is the transition relation defined by the rule: $\frac{\left(q_{\text{PTS}}\rightarrow_{\text{PTS}} q_{\text{PTS}}'\right)\wedge\left( q_{B}\xrightarrow{L_{\text{PTS}}\left(q_{\text{PTS}}\right)}q_{B}'\right)}{q_{P}=\left(q_{\text{PTS}},q_{B}\right)\longrightarrow_P q_{P}'=\left(q_{\text{PTS}}',q_{B}'\right)}$. Transition from state $q_P\in\mathcal{Q}_P$ to $q_P'\in\mathcal{Q}_P$, is denoted by $(q_P,q_P')\in\longrightarrow_P$, or $q_P\longrightarrow_P q_P'$; (d) $w_P(q_{\text{P}},q_{\text{P}}')=w_{\text{PTS}}(q_{\text{PTS}},q_{\text{PTS}}')\geq0$,  where $q_{\text{P}}=(q_{\text{PTS}},q_B)$ and $q_{\text{P}}'=(q_{\text{PTS}}',q_B')$; and
(e) $\mathcal{Q}_P^F=\mathcal{Q}_{\text{PTS}}\times\mathcal{Q}_B^F$ is a set of accepting/final states. 
\end{definition} 

In what follows, we assume that the robots have to accomplish a complex collaborative task captured by a global LTL statement $\phi$ defined over the set of atomic propositions $\mathcal{AP}=\bigcup_{i=1}^N\mathcal{AP}_i$. Given such an LTL formula $\phi$, we define the \textit{language} $\texttt{Words}(\phi)=\left\{\sigma\in (2^{\mathcal{AP}})^{\omega}|\sigma\models\phi\right\}$, where $\models\subseteq (2^{\mathcal{AP}})^{\omega}\times\phi_i$ is the satisfaction relation, as the set of infinite words $\sigma\in (2^{\mathcal{AP}})^{\omega}$ that satisfy the LTL formula $\phi$. Given such a global LTL formula $\phi$ and the PBA an infinite path $\tau$ of a $\text{PTS}$ satisfies $\phi$ if and only if $\texttt{trace}(\tau)\in\texttt{Words}(\phi)$, which is equivalently denoted by $\tau\models\phi$. 

Given an LTL formula $\phi$, if there is a path satisfying $\phi$, then there exists a path  $\tau\models\phi$ that can be written in a finite representation, called prefix-suffix structure, i.e., $\tau=\tau^{\text{pre}}[\tau^{\text{suf}}]^{\omega}$, where the prefix part $\tau^{\text{pre}}$ is executed only once followed by the indefinite execution of the suffix part $\tau^{\text{suf}}$ \cite{guo2013revising,guo2015multi}. The prefix part $\tau^{\text{pre}}$ is the projection of a finite path of the PBA, i.e., a finite sequence of states of the PBA, denoted by $p^{\text{pre}}$, onto $\ccalQ_{\text{PTS}}$, which has the following structure $p^{\text{pre}}=(q_{\text{PTS}}^0,q_B^0)(q_{\text{PTS}}^1,q_B^1)\dots (q_{\text{PTS}}^K,q_B^K)$ with $(q_{\text{PTS}}^K,q_B^K)\in\ccalQ_B^F$. The suffix part $\tau^{\text{suf}}$ is the projection of a finite path of the PBA, denoted by $p^{\text{suf}}$, onto $\ccalQ_{\text{PTS}}$, which has the following structure $p^{\text{suf}}=(q_{\text{PTS}}^{K},q_B^K)(q_{\text{PTS}}^{K+1},q_B^{K+1})\dots (q_{\text{PTS}}^{K+S},q_B^{K+S})(q_{\text{PTS}}^{K+S+1},q_B^{K+S+1})$, where $(q_{\text{PTS}}^{K+S+1},q_B^{K+S+1})=(q_{\text{PTS}}^{K},q_B^{K})$. Then our goal is to compute a plan $\tau=\tau^{\text{pre}}[\tau^{\text{suf}}]^{\omega}=\Pi|_{\text{PTS}}p^{\text{pre}}[\Pi|_{\text{PTS}}p^{\text{suf}}]^{\omega}$, where $\Pi|_{\text{PTS}}$ stands for the projection on the state-space $\ccalQ_{\text{PTS}}$, so that the following objective function is minimized
\begin{align}\label{eq:cost2}
J(\tau)=\hat{J}(\tau^{\text{pre}})+\hat{J}(\tau^{\text{suf}}),
\end{align}
which captures the total cost incurred by all robots during the execution of the prefix and a single execution of the suffix part. In \eqref{eq:cost2}, $\hat{J}(\tau^{\text{pre}})$ and $\hat{J}(\tau^{\text{suf}})$ stands for the cost of the prefix and suffix part, where the cost function $\hat{J}(\cdot)$ is defined in \eqref{eq:cost}, i.e., $\hat{J}(\tau^{\text{pre}})=\sum_{k=1}^{K-1}w_{\text{PTS}}(\Pi|_{\text{PTS}}p^{\text{pre}}(k),\Pi|_{\text{PTS}}p^{\text{pre}}(k+1)),~
\hat{J}(\tau^{\text{suf}})=\sum_{k=K}^{K+S}w_{\text{PTS}}(\Pi|_{\text{PTS}}p^{\text{suf}}(k)\Pi|_{\text{PTS}}p^{\text{suf}}(k+1))$. Specifically, in this paper we address the following problem.

\begin{problem}\label{pr:problem}
Given a global LTL specification $\phi$, and transition systems $\text{wTS}_i$, for all robots $i$, determine a discrete team plan $\tau$ that satisfies $\phi$, i.e., $\tau\models\phi$, and minimizes the cost function \eqref{eq:cost2}.%, and minimizes the cost $J(\tau)$ defined in \eqref{eq:cost}. 
\end{problem}

\subsection{A Solution to Problem \ref{pr:problem}}\label{sec:prelim}
Problem \ref{pr:problem} is typically solved by applying graph-search methods to the PBA, see e.g., \cite{guo2013revising,guo2015multi}. Specifically, to generate a motion plan $\tau$ that satisfies $\phi$, the PBA is viewed as a weighted directed graph $\mathcal{G}_P=\{\mathcal{V}_P, \mathcal{E}_P, w_P\}$, where the set of nodes $\mathcal{V}_P$ is indexed by the set of states $\mathcal{Q}_P$, the set of edges $\mathcal{E}_P$ is determined by the transition relation $\longrightarrow_P$, and the weights assigned to each edge are determined by the function $w_P$. Then we find the shortest paths from the initial states to all reachable final/accepting states $q_P\in\mathcal{Q}_P^F$ and projecting these paths onto the $\text{PTS}$ results in the prefix parts $\tau^{\text{pre},a}$, where $a\in\{1,\dots,|\mathcal{Q}_P^F|\}$. The respective suffix parts $\tau^{\text{suf},a}$ are constructed similarly by computing the shortest cycle around the $a$-th accepting state. All the resulting motion plans $\tau^a=\tau^{\text{pre},a}[\tau^{\text{suf},a}]^{\omega}$ satisfy the LTL specification $\phi$. Among all these plans, we can easily compute the optimal plan that minimizes the cost function defined in \eqref{eq:cost2} by computing the cost $J(\tau^a)$ for all plans and selecting the one with the smallest cost; see e.g. \cite{guo2015multi,ulusoy2013optimality,ulusoy2014optimal}. 

\section{Sampling-based Optimal Control Synthesis}\label{sec:solution}
Since the size of the PBA can grow arbitrarily large with the number of robots and the complexity of the task, constructing this PBA and applying graph-search techniques to find optimal plans, as discussed in Section \ref{sec:prelim}, is resource demanding and computationally expensive. In this section, we propose a sampling-based planning algorithm that is scalable and constructs a discrete motion plan $\tau$ in prefix-suffix structure, i.e., $\tau=\tau^{\text{pre}}[\tau^{\text{suf}}]^{\omega}$, that satisfies a given global LTL specification $\phi$. The procedure is based on the incremental construction of a directed tree that approximately represents the state-space $\mathcal{Q}_P$ and the transition relation $\rightarrow_P$ of the PBA defined in Definition \ref{defn:pba}. The construction of the prefix and the suffix part is described in Algorithm \ref{alg:plans}. In Algorithm \ref{alg:plans}, first the LTL formula is translated to a NBA $B=\{\ccalQ_B,\ccalQ_B^0,\rightarrow_B,\ccalQ_B^F\}$ [line \ref{alg1:line1}, Alg. \ref{alg:plans}]. Then, in lines \ref{alg1:line2}-\ref{alg1:line6}, the prefix parts $\tau^{\text{pre},a}$ are constructed, followed by the construction of their respective suffix parts $\tau^{\text{suf},a}$ in lines \ref{alg1:line7}-\ref{alg1:line15}. Finally, using the constructed prefix and suffix parts, the optimal plan $\tau=\tau^{\text{pre},a^{*}}[\tau^{\text{suf},a^{*}}]^{\omega}\models\phi$ is synthesized in lines \ref{alg1:line15a}-\ref{alg1:line16}.

\begin{algorithm}[t]
\caption{Construction of Optimal plans $\tau\models\phi$}
\LinesNumbered
\label{alg:plans}
%\vspace{0.5cm}
\KwIn {Logic formula $\phi$, Transition systems $\text{wTS}_1,\dots,\text{wTS}_N$, Initial location $q_{\text{PTS}}^0\in\ccalQ_{PTS}$, maximum numbers of iterations $n_{\text{max}}^{\text{pre}}$, $n_{\text{max}}^{\text{suf}}$}
\KwOut {Optimal plans $\tau\models\phi$}
Convert $\phi$ to a NBA $B=\left(\ccalQ_B,\ccalQ_B^0,\rightarrow_B,\ccalQ_B^F\right)$\;\label{alg1:line1} 
%\Comment*[r]{Construction of Prefix Parts $\tau^{\text{pre},s}$}
Define goal set: $\mathcal{X}_{\text{goal}}^{\text{pre}}$\;\label{alg1:line2} 
\For{$b_0=1:|\ccalQ_B^0|$}{\label{alg1:line2a} 
Initial NBA state: $q_B^0=\ccalQ_B^0(b_0)$\;\label{alg1:initNBA}
Root of the tree: $q_P^r=(q_\text{PTS}^0,q_B^0)$\;\label{alg1:line3}
\footnotesize{ 
$\left[\ccalG_T,\ccalP\right]=\texttt{ConstructTree}(\mathcal{X}_{\text{goal}}^{\text{pre}},\text{wTS}_1,\dots,\text{wTS}_N,B,q_P^r,n_{\text{max}}^{\text{pre}})$\;\label{alg1:line4}}
\normalsize
\For {$a=1:|\ccalP|$}{\label{alg1:line5} 
$\tau^{\text{pre},a}=\texttt{FindPath}(\mathcal{G}_T,q_P^r,\mathcal{P}(a))$\;}\label{alg1:line6} 
%
%\Comment*[r]{Construction of Suffix Parts $\tau^{\text{suf},f}$}
\For {$a=1:|\ccalP|$ }{\label{alg1:line7}
Root of the tree: $q_P^r=\ccalP(a)$\;\label{alg1:line8} 
Define goal set: $\mathcal{X}_{\text{goal}}^{\text{suf}}(q_P^r)$\;\label{alg1:line9}
\If{$(q_P^r\in\ccalX_{\text{goal}}^{\text{suf}}) ~\wedge~(w_P(q_P^r,q_P^r)=0)$}{\label{alg1:line9a}%(\Pi|_{\ccalQ_B}(q_P^0),L(\Pi|_{\ccalW^N}(q_P^0)),\Pi|_{\ccalQ_B}(q_P^0))\in\rightarrow_B
$\ccalG_T=(\{q_P^r\},\{q_P^r,q_P^r\},0)$\;\label{alg1:line9b}
$\ccalS_a=\{q_P^r\}$\;}\label{alg1:line9c}
\Else{\label{alg1:line9d} 
\footnotesize{
$\left[\ccalG_T,\ccalS_a\right]=\texttt{ConstructTree}(\mathcal{X}_{\text{goal}}^{\text{suf}},\text{wTS}_1,\dots,\text{wTS}_N,B,q_P^r,n_{\text{max}}^{\text{suf}})$\;}\label{alg1:line10}}
\normalsize 
\For {$e=1:|\ccalS_a|$}{\label{alg1:line11} 
$\tilde{\tau}^{\text{suf},e}=\texttt{FindPath}(\mathcal{G}_T,q_P^r,\mathcal{S}_f(e))$\;\label{alg1:line12} }
$e^{*}=\argmin_e(\hat{J}(\tilde{\tau}^{\text{suf},e}))$\;\label{alg1:line13} 
$\tau^{\text{suf},a}=\tilde{\tau}^{\text{suf},e^{*}}$}\label{alg1:line14}
%Convert $\phi$ to a NBA $B=\left(\ccalQ_B,\ccalQ_B^0,\rightarrow_B,\ccalQ_B^F\right)$
%\Comment*[r]{Construction of optimal plans}
$a_{q_B^0}=\argmin_a(\hat{J}(\tau^{\text{pre},a})+\hat{J}(\tau^{\text{suf},a}))$\;\label{alg1:line15}} 
$a^*=\argmin_{a_{q_B^0}}(\hat{J}(\tau^{\text{pre}}_{q_B^0})+\hat{J}(\tau^{\text{suf}}_{{q_B^0}}))$\;\label{alg1:line15a} 
Optimal Plan: $\tau=\tau^{\text{pre},a^{*}}[\tau^{\text{suf},a^{*}}]^{\omega}$\;\label{alg1:line16} 
\end{algorithm}

\begin{algorithm}[t]
\caption{\texttt{Function} $[\mathcal{G}_T,~\ccalZ]=\texttt{ConstructTree}(\ccalX_{\text{goal}},~\text{wTS}_1,\dots,\text{wTS}_N,~B,~q_P^r,~n_{\text{max}})$}
\LinesNumbered
\label{alg:tree}
%\vspace{0.5cm}
%\KwIn {Initial state $q_P^0$, Goal region $\ccalX_{\text{goal}}$, tuple $E=\left(\mathcal{W}^N,\mathcal{AP},L,C\right)$, NBA $B$, maximum number of iterations $n_{\text{max}}$}
%\KwOut {Tree $\mathcal{G}_T=\{\mathcal{V}_T,\mathcal{E}_T,\texttt{Cost}\}$,  set $\ccalZ=\{q_P\in\ccalV_T,~|~q_P\in\mathcal{X}_{\text{goal}}\}$}
%$\mathcal{V}_T=q_P^0=\{\bbx^0,q_B^0\}$, where $q^0_{\text{PTS}}=\{q_1^0,q_2^0,\dots,q_N^0\}$\;\label{alg1:line1} 
$\mathcal{V}_T=\{q_P^r\}$\;\label{tree:line1} 
$\mathcal{E}_T=\emptyset$\; \label{tree:line2} 
$\texttt{Cost}(q_P^r)=0$\;\label{tree:line3}
%$\ccalZ=\emptyset$\;\label{tree:line4} 
\For {$n=1:n_{\text{max}}$}{\label{tree:line4} 
$q^{\text{new}}_{\text{PTS}}=\texttt{Sample}(\ccalV_T,\text{wTS}_1,\dots,\text{wTS}_N)$\;\label{tree:line5} 
	\For {$b=1:|\mathcal{Q}_B|$}{\label{tree:line6} 
			 $q_B^{\text{new}}=\mathcal{Q}_B(b)$\;\label{tree:line7} 
			 $q_{P}^{\text{new}}=(q^{\text{new}}_{\text{PTS}},q_B^{\text{new}})$\;\label{tree:line8} 
				\If {$q_{P}^{\text{new}}\notin\mathcal{V}_T$}{\label{tree:line9} 
						$[\mathcal{V}_T,~\mathcal{E}_T,\texttt{Cost}]=\texttt{Extend}(q_{P}^{\text{new}},\rightarrow_P)$\;\label{tree:line10}}
				\If{$q_{P}^{\text{new}}\in\mathcal{V}_T$}{\label{tree:line11} 				
						$[\mathcal{E}_T,\texttt{Cost}]=\texttt{Rewire}(q_P^{\text{new}},\mathcal{V}_T,\mathcal{E}_T,\texttt{Cost})$\;} }}\label{tree:line12} 
%$\mathcal{E}_T=\texttt{OptimizeTree}(\mathcal{G}_T)$\; \label{alg1:line13} 
$\ccalZ=\ccalV_T\cap\ccalX_{\text{goal}}$\;\label{tree:line15} 
%\For {$f=1:|\ccalS|$}{\label{alg1:line17} 
%$\tau^{\text{pre},f}=\texttt{FindPath}(\mathcal{G}_T,q_P^0,\mathcal{S}(f))$\;}\label{alg1:line18} 
\end{algorithm}

In what follows, we denote by $\mathcal{G}_T=\{\mathcal{V}_T,\mathcal{E}_T,\texttt{Cost}\}$ the tree that approximately represents the PBA $P$. Also, we denote by $q_P^r$ the root of $\ccalG_T$. The set of nodes $\mathcal{V}_T$ contains the states of $\mathcal{Q}_P$ that have already been sampled and added to the tree structure. The set of edges $\mathcal{E}_T$ captures transitions between nodes in $\mathcal{V}_T$, i.e., $(q_P,q_P')\in\mathcal{E}_T$, if there is a transition from state $q_P\in\mathcal{V}_T$ to state $q_P'\in\mathcal{V}_T$. The function $\texttt{Cost}:\ccalV_T:\rightarrow\mathbb{R}_+$ assigns the cost of reaching node $q_P\in\mathcal{V}_T$ from the root $q_P^r$ of the tree. In other words, 
\begin{equation}\label{eq:relCost}
\texttt{Cost}(q_P)=\hat{J}(\tau_T), 
\end{equation}
where $q_P\in\ccalV_T$ and $\tau_T$ is the path in the tree $\ccalG_T$ that connects the root to $q_P$.

\subsection{Construction of Prefix Parts}\label{sec:prefix}
In this Section, we describe how to construct the tree $\mathcal{G}_T=\{\mathcal{V}_T,\mathcal{E}_T,\texttt{Cost}\}$ that will be used for the synthesis of the prefix part [lines \ref{alg1:line2}-\ref{alg1:line6}, Alg. \ref{alg:plans}]. Since the prefix part connects an initial state $q_P^0=(q_\text{PTS}^0,q_B^0)\in\ccalQ_P^0$ to an \textit{accepting} state $q_P=(q_{\text{PTS}},~q_B)\in\ccalQ_P^F$, with $q_B\in\mathcal{Q}_B^{F}$, we can define the goal region for the tree $\ccalG_T$, as [line \ref{alg1:line2}, Alg. \ref{alg:plans}]
\begin{equation}\label{eq:goalPre}
\mathcal{X}_{\text{goal}}^{\text{pre}}=\{q_P=(q_{\text{PTS}},~q_B)\in\ccalQ_P~|~q_B\in\mathcal{Q}_B^{F}\}.
\end{equation}
The root $q_P^r$ of the tree is an initial state $q_P^0=(q_\text{PTS}^0,q_B^0)$ of the PBA and the following process is repeated for each initial state $q_B^0\in\ccalQ_B^0$ [line \ref{alg1:line2a}-\ref{alg1:line3}, Alg. \ref{alg:plans}]. The construction of the tree is described in Algorithm \ref{alg:tree} [line \ref{alg1:line4}, Alg. \ref{alg:plans}]. In line \ref{alg1:initNBA} of Algorithm \ref{alg:plans}, $\ccalQ_B^0(b_0)$ stands for the $b_0$-th initial state assuming an arbitrary enumeration of the elements of the set $\ccalQ_B^0$. The set $\mathcal{V}_T$ initially contains only the root $q_P^r$, i.e., an initial state of the PBA [line \ref{tree:line1} , Alg. \ref{alg:tree}] and, therefore, the set of edges is initialized as $\mathcal{E}_T=\emptyset$ [line \ref{tree:line2}, Alg. \ref{alg:tree}]. By convention, we assume that the cost of $q_P^r$ is zero [line \ref{tree:line3}, Alg. \ref{alg:tree}]. %The set $\mathcal{F}\subseteq\mathcal{V}_T$ [line \ref{alg1:line4}, Alg. \ref{alg:prefix}] collects all the final states of $P$ that exist in the tree.

\subsubsection{Sampling a state $q_P^{\text{new}}\in\mathcal{Q}_P$}
The first step in the construction of the graph $\mathcal{G}_T$ is to sample a state from the state-space of the PBA. This is achieved by a sampling function $\texttt{Sample}$; see Algorithm \ref{alg:sample}. Specifically, we first create a state $q_{\text{PTS}}^{\text{rand}}=\Pi|_{\text{PTS}}q_P^{\text{rand}}$, where $q_P^{\text{rand}}$ is sampled from a given discrete distribution $f_{\text{rand}}(q_P|\ccalV_T):\ccalV_T\rightarrow[0,1]$ [lines \ref{s:line2}-\ref{s:line3}, Alg. \ref{alg:sample}]. The probability density function $f_{\text{rand}}(q_P|\ccalV_T)$ defines the probability of selecting the state $q_P\in\ccalV_T$ as the state $q_P^{\text{rand}}$ at iteration $n$ of Algorithm \ref{alg:tree} given the set $\ccalV_T$. We make the following assumption for $f_{\text{rand}}(q_P|\ccalV_T)$.%; see also Figure \ref{fig:frand}.

\begin{as}[Probability density function $f_{\text{rand}}$]\label{frand}
(i) The probability density function $f_{\text{rand}}(q_P|\ccalV_T):\ccalV_T\rightarrow[0,1]$ is bounded away from zero on $\ccalV_T$. (ii) The probability density function $f_{\text{rand}}(q_P|\ccalV_T):\ccalV_T\rightarrow[0,1]$ remains the same for all iterations $n$ and for a given state $q_P\in\ccalV_T$ is monotonically decreasing with respect to the size of $|\ccalV_T|$. This also implies that for a given $q_P\in\ccalV_T$, the probability $f_{\text{rand}}(q_P|\ccalV_T)$ remains the same for all iterations $n$ if the set $\ccalV_T$ does not change. 
%(ii) The probability density function $f_{\text{rand}}^n(q_P|\ccalV_T^n)$ is bounded below by a sequence $g^{n}(q_P|\ccalV_T^n)$, such that $\sum_{n=1}^{\infty}g^n(q_P|\ccalV_T^n)=\infty$, for all $q_P\in\ccalV_T^n$. 
(iii) Independent samples $q_P^{\text{rand}}$ can be drawn from $f_{\text{rand}}$.
\end{as}
%\begin{rem}[Probability density function $f_{\text{rand}}$]

By definition of $f_{\text{rand}}(q_P|\ccalV_T)$ we have that $q_P^{\text{rand}}$ always belongs to $\ccalV_T$. Also, by assumption \ref{frand}(i), we have that any node $q_P\in\ccalV_T$ has a non-zero probability to be selected as the node $q_P^{\text{rand}}$. Also, notice that assumption \ref{frand}(ii) is reasonable, as it requires that the probability of selecting a given state $q_P\in\ccalV_T$ decreases as the cardinality of $\ccalV_T$ increases. An example of a distribution that satisfies Assumption \ref{frand} is the discrete uniform distribution
\begin{align}\label{eq:example}
f_{\text{rand}}(q_P|\ccalV_T)=
\begin{cases}
\frac{1}{|\ccalV_T|},~&\mbox{if}~ q_P\in\ccalV_T,\\
0,~&\mbox{otherwise}.
\end{cases}
\end{align}
%Notice that the density function \eqref{eq:example} does not change with respect to iteration $n$, i.e., it is always a discrete uniform distribution. However, the probability that is assigned to each node $q_P\in\ccalV_T$ changes with $n$, as the set $\ccalV_T$ changes. Also, observe that \eqref{eq:example} satisfies Assumption \ref{frand}(ii), as there exists a constant sequence $g^n(q_P)=\frac{1}{\ccalQ_P}$ that satisfies (i) $f_{\text{rand}}^n(q_P|\ccalV_T^n)\geq g^n(q_P)$, since $|\ccalV_T|\leq|\ccalQ_P|$ and (ii) $\sum_{n=1}^{\infty}g^n(q_P|\ccalV_T^n)=\infty$, for all $q_P\in\ccalV_T^n$, since $g^n(q_P)=\frac{1}{\ccalQ_P}$ is a strictly positive constant term.
Notice that other sampling methods for $q_P^{\text{rand}}$ can be employed that do not require the more strict conditions of Assumption \ref{frand}(ii); see Remark \ref{rem:lem1} in Appendix \ref{sec:lemmas1}.
%\end{rem}
%
%\begin{figure}[t]
%\centering
%  \includegraphics[width=1\linewidth]{frand.eps}
%  \caption{Graphical depiction of the probability of selecting $q_P\in\ccalQ_P$ to be $q_P^{\text{rand}}$ with respect to the size of the set of nodes $\ccalV_T$. Notice that it always holds $|\ccalV_T|\leq|\ccalQ_P|$ and, therefore, $f_{\text{rand}}(q_P|\ccalV_T)$ is not defined when $|\ccalV_T|>|\ccalQ_P|$.}
%  \label{fig:frand}
%\end{figure}

% Notice that the sample-space of $f_{\text{rand}}^n$ is the set of nodes $\ccalV_T$, which changes as $n$ increases and, therefore, so does the probability $f_{\text{rand}}^n(q_P)$. %Also, we assume that if the set of nodes $\ccalV_T$ remains the same at any two consecutive iterations $n$, $n+1$ of Algorithm \ref{alg:tree}, then $f_{\text{rand}}^n=f_{\text{rand}}^{n+1}$. Finally, we assume that we can draw independent samples $q_P^{\text{rand}}$ from the probability density function $f_{\text{rand}}^n$. 

Given a state $q_{\text{PTS}}^{\text{rand}}$, we define its \textit{reachable set} in the PTS
\begin{align}\label{reachable1} &\ccalR_{\text{PTS}}(q_{\text{PTS}}^{\text{rand}})=\{q_{\text{PTS}}\in\ccalQ_{\text{PTS}}~|~q_{\text{PTS}}^{\text{rand}}\rightarrow_{\text{PTS}}q_{\text{PTS}}\}
\end{align}
i.e., $\ccalR_{\text{PTS}}(q_{\text{PTS}}^{\text{rand}})\subseteq\ccalQ_{\text{PTS}}$ collects all the states $q_{\text{PTS}}\in\ccalQ_{\text{PTS}}$ that can be reached from $q_{\text{PTS}}^{\text{rand}}$ in one hop. Then, we sample a state $q_{\text{PTS}}^{\text{new}}$ from a discrete distribution $f_{\text{new}}(q_{\text{PTS}}|q_{\text{PTS}}^{\text{rand}}):\ccalR_{\text{PTS}}(q_{\text{PTS}}^{\text{rand}})\rightarrow[0,1]$ [line \ref{s:line5}, Alg. \ref{alg:sample}] that satisfies the following assumption.
%\footnote{Notice that the set $\ccalR_{\text{PTS}}(q_{\text{PTS}}^{\text{rand}})$ does not require the explicit construction of the product transition system, since it can be constructed by the individual wTS.} 

\begin{as}[Probability density function $f_{\text{new}}$]\label{fnew}
(i) The probability density function $f_{\text{new}}(q_{\text{PTS}}|q_{\text{PTS}}^{\text{rand}}):\ccalR_{\text{PTS}}(q_{\text{PTS}}^{\text{rand}})\rightarrow[0,1]$ is bounded away from zero on $\ccalR_{\text{PTS}}(q_{\text{PTS}}^{\text{rand}})$. (ii) %Given a state $q_{\text{PTS}}^{\text{rand}}=q_{\text{PTS}}\in\Pi|_{\text{PTS}}\ccalV_T$, the probability density function $f_{\text{new}}^n(q_{\text{PTS}}|q_{\text{PTS}}^{\text{rand}})$ that is defined for all iterations $n$ at which it holds  $q_P^{\text{rand}}=q_P$, that are collected in a set $\ccalK$, is bounded below by $h^{n}(q_{\text{PTS}}|q_{P}^{\text{rand},n})$, for all $n\in\ccalK$, such that $\sum_{n\in\ccalK}h^{n}(q_{\text{PTS}}|q_{P}^{\text{rand}})=\infty$, for all $q_{\text{PTS}}\in\ccalR_P(q_P^{\text{rand}})$.
For a given $q_{\text{PTS}}^{\text{rand}}$, the distribution $f_{\text{new}}(q_{\text{PTS}}|q_{\text{PTS}}^{\text{rand}})$ remains the same for all iterations $n$. (iii) Given a state $q_{\text{PTS}}^{\text{rand}}$, independent samples $q_{\text{PTS}}^{\text{new}}$ can be drawn from the probability density function $f_{\text{new}}$. 
\end{as}

By definition of $f_{\text{new}}(q_{\text{PTS}}|q_{\text{PTS}}^{\text{rand}})$, we have that $q_{\text{PTS}}^{\text{rand}}$ always belongs to $\ccalR_{\text{PTS}}(q_{\text{PTS}}^{\text{rand}})$. Also, observe that by Assumption \ref{fnew}(i), we have that any node $q_P\in\ccalR_{\text{PTS}}(q_{\text{PTS}}^{\text{rand}})$ has a non-zero probability to be $q_{\text{PTS}}^{\text{new}}$. Moreover, notice that the state $q_P^{\text{rand}}$ can change at every iteration $n$ and, therefore, clearly, so does the reachable set $\ccalR_{\text{PTS}}(q_{\text{PTS}}^{\text{rand}})$. Nevertheless, observe that the set $\ccalR_{\text{PTS}}(q_{\text{PTS}}^{\text{rand}})$ remains the same for all iterations $n$ for a given $q_{\text{PTS}}^{\text{rand}}$. Finally, note that other sampling methods for $q_{\text{PTS}}^{\text{new}}$ can be employed that do not require the more strict conditions of Assumption \ref{fnew}(ii); see Remark \ref{rem:lem2} in Appendix \ref{sec:lemmas1}.

\begin{rem}[Reachable set $\ccalR_{\text{PTS}}(q_{\text{PTS}}^{\text{rand}})$]\label{rem:reachable}
In practice in order to obtain the state $q_{\text{PTS}}^{\text{new}}$ we do not need to construct the reachable set $\ccalR_{\text{PTS}}(q_{\text{PTS}}^{\text{rand}})$, which is a computationally expensive step. In fact, we only need the \textit{reachable} sets $\ccalR_{\text{TS}_i}(q_i^{\text{rand}})$, for all robots $i$, that collect all states that are reachable from the state $q_i^{\text{rand}}=\Pi|_{\text{TS}_i} q_{\text{PTS}}^{\text{rand}}\in\ccalQ_i$ in one hop, defined as $\ccalR_{\text{TS}_i}(q_i^{\text{rand}})=\{q_{i}\in\ccalQ_{i}|q_i^{\text{rand}}\rightarrow_i q_i\}$. Then, we can define the probability density functions $f_{\text{new},i}(q_i|q_i^{\text{rand}}):\ccalR_{\text{TS}_i}(q_i^{\text{rand}})\rightarrow[0,1]$ that are bounded away from zero on $\ccalR_{\text{TS}_i}(q_i^{\text{rand}})$ and use them to draw a sample $q_{i}^{\text{new}}$ that is reachable from $q_i^{\text{rand}}$, for all $i\in\set{1,\dots,N}$. Stacking these samples in a vector we can define $q_{\text{PTS}}^{\text{new}}=[q_{1}^{\text{new}}, q_{2}^{\text{new}}, \dots, q_{N}^{\text{new}} ]$ and the probability density function $f_{\text{new}}$ becomes $f_{\text{new}}=f_{\text{new},1}\cdot f_{\text{new},2}\cdot \dots\cdot f_{\text{new},N}$ that has to satisfy Assumption \ref{fnew}. Clearly, the resulting state $q_{\text{PTS}}^{\text{new}}$ belongs to the reachable set $\ccalR_{\text{PTS}}(q_{\text{PTS}}^{\text{rand}})$. Throughout the paper, for simplicity, our analysis is based on $f_{\text{new}}$ and not on the probability density functions $f_{\text{new},i}$.
\end{rem}

% Assumption \ref{frand} will be used in Section \ref{sec:corr} to show the proposed algorithm is probabilistically complete and asymptotically optimal. 
%As it will be shown in Proposition \ref{prop:growth}, Assumptions \ref{frand}-(i) and \ref{fnew}-(i) ensure that at each iteration $n$ of Algorithm \ref{alg:tree}, at least one state is added to the set $\ccalV_T^n$ under some mild conditions. Assumption \ref{fnew}-(ii) will be used in Section \ref{sec:corr} to show the proposed algorithm is probabilistically complete and asymptotically optimal.

In order to build incrementally a graph whose set of nodes approximates the state-space $\mathcal{Q}_P$ we need to append to $q_{\text{PTS}}^{\text{new}}$ a state from the state-space $\mathcal{Q}_B$ of the NBA $B$. Let $q_B^{\text{new}}=\mathcal{Q}_B(b)$ [line \ref{tree:line7}, Alg. \ref{alg:tree}] be the candidate B$\ddot{\text{u}}$chi state that will be attached to $q_{\text{PTS}}^{\text{new}}$, where $\mathcal{Q}_B(b)$ stands for the $b$-th state in the set $\ccalQ_B$ assuming an arbitrary enumeration of the elements of the set $\ccalQ_B$. The following procedure is repeated for all $q_B^{\text{new}}=\ccalQ_B(b)$ with $b\in\{1,\dots,|\mathcal{Q}_B|\}$. First, we construct the state $q_{P}^{\text{new}}=(q^{\text{new}}_{\text{PTS}},q_B^{\text{new}})\in\mathcal{Q}_P$ [line \ref{tree:line8}, Alg. \ref{alg:tree}] and then we check if this state can be added to the tree $\mathcal{G}_T$ [lines \ref{tree:line9}-\ref{tree:line10}, Alg. \ref{alg:tree}]. If the state $q_{P}^{\text{new}}$ does not already belong to the tree from a previous iteration of Algorithm \ref{alg:tree}, i.e, if $q_{P}^{\text{new}}\notin\mathcal{V}_T$ [line \ref{alg1:line9}, Alg. \ref{alg:tree}], we check which node in $\mathcal{V}_T$ (if there is any) can be the parent of $q_{P}^{\text{new}}$ in the tree $\mathcal{G}_T$. This is achieved by the function $\texttt{Extend}$ described in Algorithm \ref{alg:extend} [line \ref{tree:line10}, Alg. \ref{alg:tree}] and in Section \ref{sec:extend}. If $q_P^{\text{new}}\in\ccalV_T$, then the \textit{rewiring} step follows described in Algorithm \ref{alg:rewire} [lines \ref{tree:line11}-\ref{tree:line12}, Alg. \ref{alg:tree}] and in Section \ref{sec:rewire} that aims to reduce the cost of nodes $q_P\in\ccalV_T$.
%UNCOMMENT 
\begin{figure}[t]
  \centering
  \includegraphics[width=0.6\linewidth]{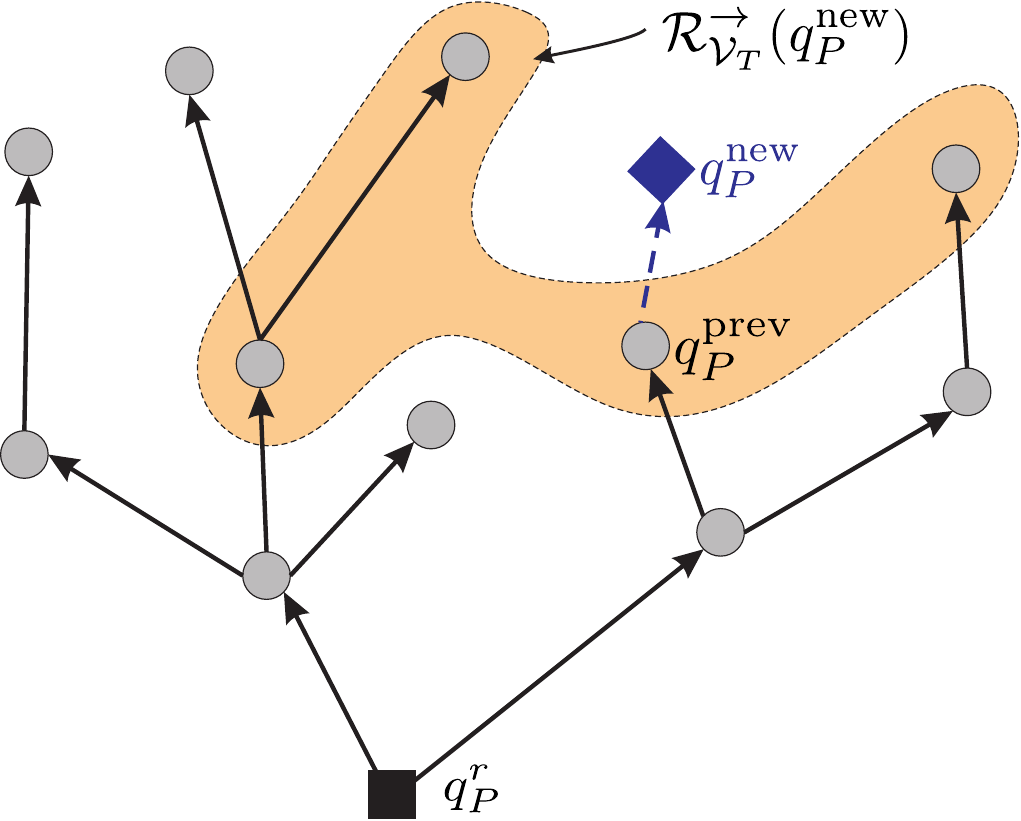}
    \caption{Graphical depiction of Algorithm \ref{alg:extend}. The black square stands for the root of the tree and the gray disks represent nodes in the set $\mathcal{V}_T$. Black arrows represent transitions captured by $\mathcal{E}_T$. The blue diamond stands for the state $q_P^{\text{new}}$ and the dashed blue arrow represents the new edge that will be added to the set $\mathcal{E}_T$ after the execution of Algorithm \ref{alg:extend} (line \ref{alg3:line5}, Alg. \ref{alg:extend}).}
  \label{fig:ext}
\end{figure}

\begin{algorithm}[t]
\caption{Function $\texttt{Sample}(\ccalV_T, \text{wTS}_1,\dots,\text{wTS}_N)$}
\label{alg:sample}
%\vspace{0.5cm}
%$q_{\text{PTS}}^{\text{new}}=[~]$\;\label{s:line1}
Pick a state $q_{P}^{\text{rand}}\in\ccalV_T$ from a given distribution $f_{\text{rand}}(q_P|\ccalV_T):\ccalV_T\rightarrow[0,1]$\;\label{s:line2}
$q_{\text{PTS}}^{\text{rand}}=\Pi|_{\text{PTS}}q_{P}^{\text{rand}}$\;\label{s:line3}
%Compute the reachable set $\ccalR(q_{\text{PTS}}^{\text{rand}})\subseteq\ccalQ_{\text{PTS}}$\;\label{s:line4}
Sample a state $q_{\text{PTS}}^{\text{new}}$ from probability distribution $f_{\text{new}}:\ccalR_{\text{PTS}}(q_{\text{PTS}}^{\text{rand}})\rightarrow[0,1]$\;\label{s:line5}
\Return {$q_{\text{PTS}}^{new}$}\;
\end{algorithm}
\begin{algorithm}[t]
\caption{Function $\texttt{Extend}(q_{P}^{\text{new}})$}
\label{alg:extend}
%\vspace{0.5cm}
Collect in set $\mathcal{R}^{\rightarrow}_{\ccalV_T}(q_P^{\text{new}})$ all states $q_P\in\mathcal{V}_T$ that satisfy the following transition rule: $(q_{P},q_{P}^{\text{new}})\in\rightarrow_{P}$\;\label{alg3:line1}
\If {$\mathcal{R}^{\rightarrow}_{\ccalV_T}(q_P^{\text{new}})\neq\emptyset$}{\label{alg3:line2}
$q_P^{\text{prev}}=\text{argmin}_{q_P\in\mathcal{R}^{\rightarrow}_{\ccalV_T}(q_P^{\text{new}})}[\texttt{Cost}(q_P)+w_{\text{PTS}}(\Pi|_{\text{PTS}}q_P,\Pi|_{\text{PTS}}q_P^{\texttt{new}})]$\;\label{alg3:line3}%$[q^{\text{prev}}_{\text{PTS}},q_B^{\text{prev}}]=;%\texttt{prev}(\mathcal{S}_{\rightarrow q_{P}^{\text{new}}},q^{\text{new}}_{\text{PTS}})$;
%\STATE $q_P^{\text{prev}}=[q^{\text{prev}}_{\text{PTS}},q_B^{\text{prev}}]$;
$\mathcal{V}_T=\mathcal{V}_T\cup\{q_{P}^{\text{new}}\}$\;\label{alg3:line4}
$\mathcal{E}_T=\mathcal{E}_T\cup\{(q_P^{\text{prev}},q_P^{\text{new}})\}$\;\label{alg3:line5}
$\texttt{Cost}(q_P^{\texttt{new}})=\texttt{Cost}(q_P^{\text{prev}})+w_{\text{PTS}}(\Pi|_{\text{PTS}}q_P^{\text{prev}},\Pi|_{\text{PTS}}q_P^{\texttt{new}})$\;}\label{alg3:line6}
\Return {$\mathcal{V}_T$, $\mathcal{E}_T$, $\texttt{Cost}$}\;
\end{algorithm}

\subsubsection{Adding a new edge to $\mathcal{E}_T$}\label{sec:extend}
Assume that $q_P^{\text{new}}\notin\ccalV_T$ [lines \ref{tree:line9}, Alg. \ref{alg:tree}]. Then the function $\texttt{Extend}$ described in Algorithm \ref{alg:extend} [line \ref{tree:line10}, Alg. \ref{alg:tree}] is used to check if the tree can be extended towards $q_P^{\text{new}}$. The first step in Algorithm \ref{alg:extend} is to construct the set $\mathcal{R}^{\rightarrow}_{\ccalV_T}(q_P^{\text{new}}) \subseteq\mathcal{V}_T$ defined as 
\begin{equation}\label{eq:Stonew}
%\mathcal{R}^{\rightarrow q_P^{\text{new}}}_P(q_{P}^{\text{new}})=\{q_P\in\ccalV_T|q_P\rightarrow_P q_P^{\text{new}}\},
\mathcal{R}^{\rightarrow}_{\ccalV_T}(q_P^{\text{new}})=\{q_P\in\ccalV_T|q_P\rightarrow_P q_P^{\text{new}}\},
\end{equation}
that collects all states $q_P\in\mathcal{V}_T$ that satisfy the transition rule $(q_{P},q_{P}^{\text{new}})\in\rightarrow_{P}$, i.e., all states that can directly reach $q_P^{\text{new}}$ [line \ref{alg3:line1}, Alg. \ref{alg:extend}]. If the resulting set $\mathcal{R}^{\rightarrow}_{\ccalV_T}(q_P^{\text{new}})$ is empty then the sample $q_{P}^{\text{new}}$ is not added to the tree. On the other hand, if $\mathcal{R}^{\rightarrow}_{\ccalV_T}(q_P^{\text{new}})\neq \emptyset$, then the state $q_{P}^{\text{new}}$ is added to the tree [lines \ref{alg3:line3}-\ref{alg3:line6}, Alg. \ref{alg:extend}]. The parent of $q_{P}^{\text{new}}$ is selected as
\begin{align}
&q_P^{\text{prev}}=\nonumber\\&\text{argmin}_{q_P\in\mathcal{R}^{\rightarrow}_{\ccalV_T}(q_P^{\text{new}})}[\texttt{Cost}(q_P)+w_{\text{PTS}}(\Pi|_{\text{PTS}}q_P,\Pi|_{\text{PTS}}q_P^{\texttt{new}})],\nonumber
\end{align} 
where $\texttt{Cost}(q_P)+w_{\text{PTS}}(\Pi|_{\text{PTS}}q_P,\Pi|_{\text{PTS}}q_P^{\texttt{new}})$ captures the cost of the node $q_P^{\text{new}}$ if it gets connected to the root through the node $q_P$. In other words, the parent $q_P^{\text{prev}}$ of node $q_P^{\text{new}}$ is selected among all states in $\mathcal{R}^{\rightarrow}_{\ccalV_T}(q_P^{\text{new}})$ so that the cost of $q_{P}^{\text{new}}$ is minimized [line \ref{alg3:line3}, Alg. \ref{alg:extend}]. The set of nodes and edges is updated as $\mathcal{V}_T=\mathcal{V}_T\cup\{q_{P}^{\text{new}}\}$ and $\mathcal{E}_T=\mathcal{E}_T\cup\{(q_P^{\text{prev}},q_P^{\text{new}})\}$ [lines \ref{alg3:line4} and \ref{alg3:line5}, Alg. \ref{alg:extend}]. Given the parent $q_P^{\text{prev}}$ of the node $q_P^{\text{new}}$, the cost of $q_P^{\text{new}}$ is [line \ref{alg3:line6}, Alg. \ref{alg:extend}]:
\begin{align} \texttt{Cost}(q_P^{\texttt{new}})=&\underbrace{\texttt{Cost}(q_P^{\texttt{prev}})}_{\text{Cost of reaching $q_P^{\texttt{prev}}$ from the root of the tree $\mathcal{G}_T$}}\nonumber\\&+\underbrace{w_{\text{PTS}}(\Pi|_{\text{PTS}}q_P^{\texttt{prev}},\Pi|_{\text{PTS}}q_P^{\texttt{new}})}_{\text{cost of reaching $q_P^{\texttt{new}}$ from $q_P^{\texttt{prev}}$}},
\end{align}
due to \eqref{eq:cost} and \eqref{eq:relCost}. Algorithm \ref{alg:extend} is illustrated in Figure \ref{fig:ext}, as well.

\subsubsection{Rewiring}\label{sec:rewire}
Once a new state $q_{P}^{\text{new}}=(q^{\text{new}}_{\text{PTS}},q_B^{\text{new}})$ has been added to the tree or if the new sample $q_{P}^{\text{new}}$ already belongs to the tree [line \ref{tree:line11}, Alg. \ref{alg:tree}], the rewiring step follows [line \ref{tree:line12}, Alg. \ref{alg:tree}]. Specifically, we \textit{rewire} the nodes in $q_P\in\mathcal{V}_T$ that can get connected to the root $q_P^r$ through the node $q_{P}^{\text{new}}$ if this rewiring can decrease their cost $\texttt{Cost}(q_P)$. The rewiring process is described in Algorithm \ref{alg:rewire} and is illustrated in Figure \ref{fig:rew}.

In Algorithm \ref{alg:rewire} we first construct the reachable set $\mathcal{R}_{\ccalV_T}^{\leftarrow}(q_P^{\text{new}})\subseteq\mathcal{V}_T$ defined as
\begin{equation}
\mathcal{R}_{\ccalV_T}^{\leftarrow}(q_P^{\text{new}})=\{q_P\in\ccalV_T|q_P^{\text{new}}\rightarrow_P q_P\},
\end{equation}
that collects all states of $q_P\in\mathcal{V}_T$ that satisfy the transition rule $(q_{P}^{\text{new}},q_{P})\in\rightarrow_{P}$, i.e., all states that can be directly reached by $q_P^{\text{new}}$ [line \ref{alg4:line1}, Alg. \ref{alg:rewire}]. Then, for all states $q_P\in\mathcal{R}_{\ccalV_T}^{\leftarrow}(q_P^{\text{new}})$ we check if their current cost $\texttt{Cost}(q_P)$ is greater than their cost if they were connected to the root through $q_P^{\text{new}}$ [line \ref{alg4:line3}, Alg. \ref{alg:rewire}]. If this is the case for a node $q_P\in\mathcal{R}_{\ccalV_T}^{\leftarrow}(q_P^{\text{new}})$, then the new parent of $q_P$ becomes $q_P^{\text{new}}$, i.e, a directed edge is drawn from  $q_P^{\text{new}}$ to $q_P$, and the edge that was connecting  $q_P$ to its previous parent is deleted [lines \ref{alg4:line4}-\ref{alg4:line5}, Alg. \ref{alg:rewire}]. The cost of node $q_P$ is updated as $\texttt{Cost}(q_P)=\texttt{Cost}(q_P^{\text{new}})+w_{\text{PTS}}(\Pi|_{\text{PTS}}q_P^{\texttt{new}},\Pi|_{\text{PTS}}q_P)$ to take into account the new path through which it gets connected to the root [line \ref{alg4:line6}, Alg. \ref{alg:rewire}]. Once a state $q_P$ gets rewired, the cost of all its \textit{successor} nodes in $\ccalG_T$, collected in the set 
\begin{align}\label{eq:successor}
\ccalS(q_P)=&\{q_P'\in\ccalV_T|q_P'~\text{is connected to}~ q_P~\text{through}\nonumber\\& \text{a multi hop path in}~\ccalG_T\},
\end{align} 
is updated to account for the change in the cost of $q_P$ [line \ref{alg4:line6}, Alg. \ref{alg:rewire}]. 

%UNCOMMENT
\begin{figure}[t]
\centering
  \includegraphics[width=0.6\linewidth]{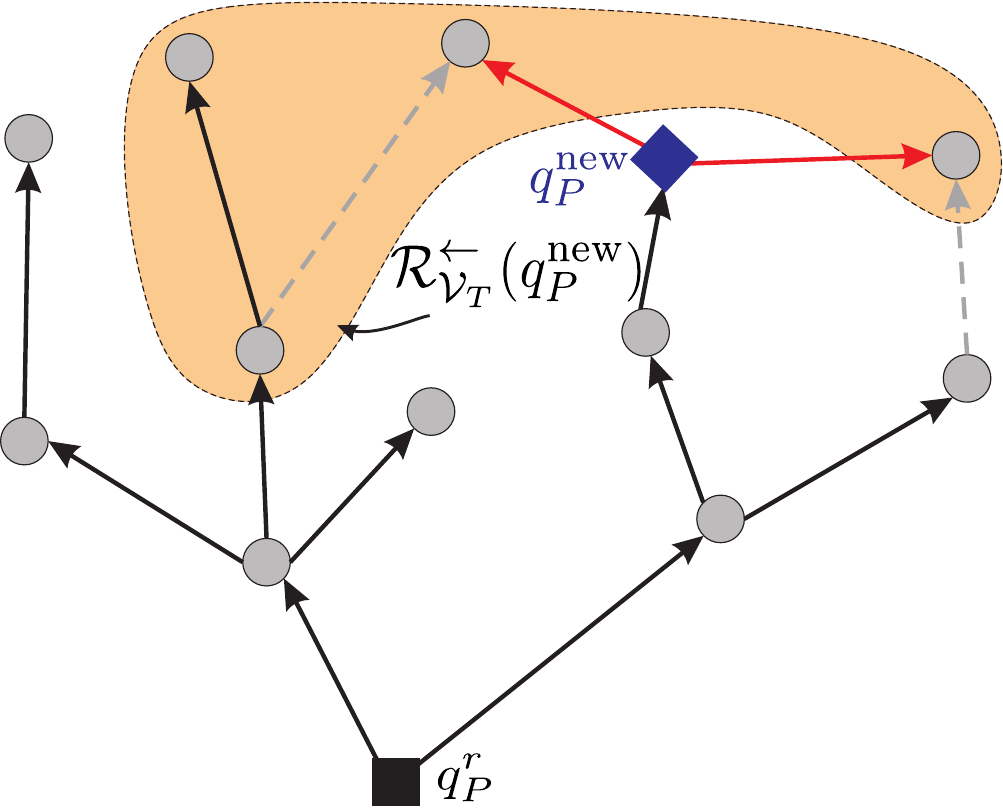}
  \caption{Graphical depiction of Algorithm \ref{alg:rewire}. The black square stands for the root of the tree and the gray disks and the blue diamond represent nodes in the set $\mathcal{V}_T$. Black arrows represent transitions captured by $\mathcal{E}_T$. The blue diamond stands for the state $q_P^{\text{new}}$. Dashed gray arrows stand for the edges that will be deleted from the set $\mathcal{E}_T$ during the execution of Algorithm \ref{alg:rewire} (line \ref{alg4:line4}, Alg. \ref{alg:rewire}). Red arrows stand for the new edges that will be added to $\mathcal{E}_T$ during the execution of Algorithm \ref{alg:rewire} (line \ref{alg4:line5}, Alg. \ref{alg:rewire}).}
  \label{fig:rew}
\end{figure}
%\subsubsection{Optimizing $\mathcal{G}_T$}
%The above procedure terminates after $n^{\text{pre}}_{\text{max}}$ iterations. Once this happens, we modify the resulting set of edges $\mathcal{E}_T$, as per Algorithm \ref{alg:optimize} so that the cost of each node in $\mathcal{V}_T$ is minimized [line \ref{alg1:line16}, Alg. \ref{alg:prefix}]. To achieve that, it f to rewire all nodes in the graph $\mathcal{G}_T$ [lines \ref{alg5:line2}-\ref{alg5:line3}, Alg. \ref{alg:optimize}]. After rewiring all nodes, a new set of edges is constructed denoted by $\mathcal{E}_T^k$, where $k=2,3,\dots$ and $\mathcal{E}_T^1:=\mathcal{E}_T$. This rewiring process is repeated until the set of edges stops changing, i.e, until  $\mathcal{E}_T^k=\mathcal{E}_T^{k-1}$ [lines \ref{alg5:line5}-\ref{alg5:line8}, Alg. \ref{alg:optimize}]. In this way, we minimize the cost of all nodes and, consequently, of nodes that represent final states, as well, since $\mathcal{F}\subset\mathcal{V}_T$. Moreover, notice that Algorithm \ref{alg:optimize} will terminate after a finite number of iterations, since the set $\mathcal{V}_T$ is finite and there is a finite number of possible transitions among these nodes captured by the transition rule $\longrightarrow_P$.

\subsubsection{Construction of Paths}
The construction of the tree $\ccalG_T$ ends after $n_{\text{max}}^{\text{pre}}$ iterations, where $n_{\text{max}}^{\text{pre}}$ is user specified [line \ref{tree:line4}, Alg. \ref{alg:tree}]. Then, we construct the set $\ccalP=\ccalV_T\cap\ccalX_{\text{goal}}^{\text{pre}}$ [line \ref{tree:line15}, Alg. \ref{alg:tree}] that collects all the states $q_P\in\ccalV_T$ that belong to the goal region $\ccalX_{\text{goal}}^{\text{pre}}$. Given the tree $\ccalG_T$ and the set $\ccalP$ [line \ref{alg1:line4}, Alg. \ref{alg:plans}] that collects all states $q_P\in\ccalX_{\text{goal}}^{\text{pre}}\cap\ccalV_T$, we can compute the prefix plans [lines \ref{alg1:line5}-\ref{alg1:line6}, Alg. \ref{alg:plans}]. In particular, the path that connects the $a$-th state in the set $\mathcal{P}$, denoted by $\ccalP(a)$, to the root $q_P^r$ constitutes the $\alpha$-th prefix plan and is denoted by $\tau^{\text{pre},a}$ [line \ref{alg1:line6}, Algorithm \ref{alg:plans}]. Its computation is described in Algorithm \ref{alg:findpath}. Specifically, the prefix part $\tau^{\text{pre},a}$ is constructed by tracing the sequence of parents of nodes starting from the node that represents the accepting state $\mathcal{P}(a)$ and ending at the root of the tree [lines \ref{alg6:line1}-\ref{alg6:line7}, Alg. \ref{alg:findpath}]. The parent of each node is computed by the function $\texttt{parent}:\mathcal{V}_T\rightarrow\mathcal{V}_{T}$ that maps a node $q_P\in\mathcal{V}_{T}$ to a unique vertex $q_P'\in\mathcal{V}_{T}$ if $(q_P',q_P)\in\mathcal{E}_{T}$, i.e., $\texttt{parent}(q_P)=q_P'$ if $(q_P',q_P)\in\mathcal{E}_{T}$. By convention, we assume that $\texttt{parent}(q_P^r)=q_P^r$. In line \ref{alg6:line7}, $\Pi|_{\text{PTS}}p_T$ stands for the projection of the path $p_T$ onto the state-space of the $\text{PTS}$. In line \ref{alg6:line4} of Algorithm \ref{alg:findpath}, $|$ stands for the concatenation of paths. Thus, for the resulting prefix plan $\tau^{\text{pre},a}$, it holds that $\tau^{\text{pre},a}(1)=\Pi|_{\text{PTS}}q_P^r$ and $\tau^{\text{pre},a}(|\tau^{\text{pre},a}|)=\Pi|_{\text{PTS}}\mathcal{P}(a)$. 

\begin{algorithm}[t]
\caption{Function $\texttt{Rewire}(q_P^{\text{new}},\mathcal{V}_T,\mathcal{E}_T,\texttt{Cost})$}
\label{alg:rewire}
%\vspace{0.5cm}
Collect in set $\mathcal{R}_{\ccalV_T}^{\leftarrow}(q_P^{\text{new}})$ all states of $q_P\in\mathcal{V}_T$ that abide by the following transition rule: $(q_{P}^{\text{new}},q_{P})\in\rightarrow_{P}$\;\label{alg4:line1}
%\IF {$\mathcal{S}_{\leftarrow q_{P}^{\text{new}}}\neq\emptyset$}
	\For {$q_P\in\mathcal{R}_{\ccalV_T}^{\leftarrow}(q_P^{\text{new}})$}{\label{alg4:line2}
		\If {$\texttt{Cost}(q_P)>\texttt{Cost}(q_P^{\text{new}})+w_{\text{PTS}}(\Pi|_{\text{PTS}}q_P^{\texttt{new}},\Pi|_{\text{PTS}}q_P)$}{\label{alg4:line3}
			$\mathcal{E}_T=\mathcal{E}_T\setminus \{(\texttt{Parent}(q_P),q_P)\}$\;\label{alg4:line4}
			$\mathcal{E}_T=\mathcal{E}_T\cup\{(q_P^{\text{new}},q_P)\}$\;\label{alg4:line5} 			$\texttt{Cost}(q_P)=\texttt{Cost}(q_P^{\text{new}})+w_{\text{PTS}}(\Pi|_{\text{PTS}}q_P^{\texttt{new}},\Pi|_{\text{PTS}}q_P)$\;\label{alg4:line6}
			Update the cost of all successor nodes of $q_P\in\ccalV_T$\;\label{alg4:line7}}}
\Return {$\mathcal{E}_T$, $\texttt{Cost}$}\;
\end{algorithm}

%\begin{algorithm}[t]
%\caption{Function $\texttt{OptimizeTree}(\mathcal{G}_T)$}
%\label{alg:optimize}
%%\vspace{0.5cm}
%$\mathcal{E}_T^{1}=\mathcal{E}_T$\;\label{alg5:line1}
%\For {$q_P\in\mathcal{V}_T$}{\label{alg5:line2}
%$[\mathcal{E}_T^2,\texttt{Cost}]=\texttt{Rewire}(\mathcal{V}_T,\mathcal{E}_T^{1},\texttt{Cost})$\;}
%$k=2$\;\label{alg5:line3}
%\While {$\mathcal{E}_T^{k}\neq \mathcal{E}_T^{k-1}$}{\label{alg5:line5}
%\For {$q_P\in\mathcal{V}_T$}{\label{alg5:line6}
%$[\mathcal{E}_T^{k+1},\texttt{Cost}]=\texttt{Rewire}(q_P,\mathcal{V}_T,\mathcal{E}_T^k,\texttt{Cost})$\;}\label{alg5:line7}
%$k=k+1$\;}\label{alg5:line8}
%\Return {$\mathcal{E}_T=\mathcal{E}_T^k$}\;
%\end{algorithm}

\begin{algorithm}[t]
\caption{Function $\texttt{FindPath}(\mathcal{G}_T,q_P^{\text{initial}},q_P^{\text{goal}})$}
\label{alg:findpath}
%\vspace{0.5cm}
$p_T=\{q_P^{\text{goal}}\}$\;\label{alg6:line1}
$q_P^{\text{prev}}=\texttt{Parent}(q_P^{\text{goal}})$\;\label{alg6:line2}
\While {$q_P^{\text{prev}}\neq q_P^{\text{initial}}$}{\label{alg6:line3}
$p_T=p_T|\{q_P^{\text{prev}}\}$\;\label{alg6:line4}
$q_P^{\text{prev}}=\texttt{Parent}(q_P^{\text{prev}})$\;}\label{alg6:line5}
$p_T=p_T|\{q_P^{\text{initial}}\}$\;\label{alg6:line6}
$p_T=\Pi|_{\text{PTS}}p_T$\;\label{alg6:line7}
\Return {$p_T$}\;
\end{algorithm}

\subsection{Construction of Suffix Parts}\label{sec:suffix}
%UNCOMMENT
\begin{figure}[t]
  \centering
    \label{fig:wait}
  \includegraphics[width=0.45\linewidth]{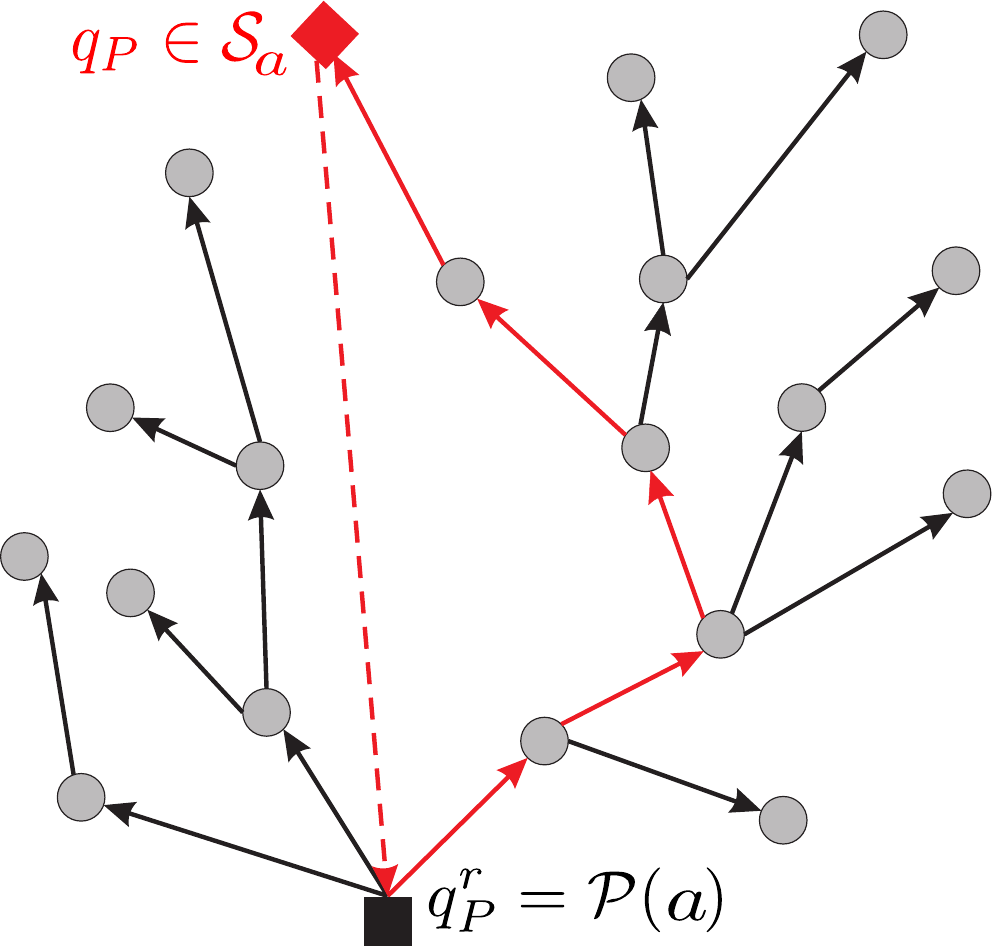}
  \caption{Graphical depiction of detecting cycles around a final/accepting state $\mathcal{P}(a)$ (black square) which acts as the root of the tree. The red diamond stands for a state $q_P\in\mathcal{S}_a$. Solid red arrows stand for the path that connects the state $q_P\in\mathcal{S}_a$ to the root $\mathcal{P}(a)$. The dashed red arrow implies that a transition from $q_P$ to $\mathcal{P}(a)$ is feasible according to the transition rule $\longrightarrow_P$; however, such a transition is not included in the set $\mathcal{E}_T$. The cycle around the accepting state $\mathcal{P}(a)$ is illustrated by solid and dashed red arrows.}
  \label{fig:findSuffix}
\end{figure}
Once the prefix plans $\tau^{\text{pre},a}$ for all $a\in\{1,\dots,|\ccalP|\}$ are constructed, the corresponding suffix plans $\tau^{\text{suf},a}$ are constructed [lines \ref{alg1:line7}-\ref{alg1:line14}, Alg. \ref{alg:plans}]. Specifically, every suffix part $\tau^{\text{suf},a}$ is a sequence of states in $\ccalQ_P$ that starts from the state $\mathcal{P}(a)$ and ends at the same state $\mathcal{P}(a)$, i.e., a cycle around state $\mathcal{P}(a)$ where any two consecutive states in $\tau_i^{\text{suf},a}$ respect the transition rule $\rightarrow_P$. 
%Although cycles around a final state $\mathcal{F}(f)$ may exist in the state-space of the PBA $P$, t
To construct the suffix plan $\tau_i^{\text{suf},a}$ we build a tree $\mathcal{G}_T=\{\mathcal{V}_T, \mathcal{E}_T, \texttt{Cost}\}$ that approximates the PBA $P$, in a similar way as in Section \ref{sec:prefix}, and implement a cycle-detection mechanism to identify cycles around the state $P(a)$. The only differences are that: (i) the root of the tree is now $q_P^{r}=\ccalP(a)$, i.e., it is an \textit{accepting/final} state [line \ref{alg1:line8}, Alg. \ref{alg:plans}] detected during the construction of the prefix plans, (ii) the goal region corresponding to the root $q_P^{r}=\ccalP(a)$, is defined as
\begin{align}\label{eq:goalSuf}
\mathcal{X}_{\text{goal}}^{\text{suf}}(q_P^r)=&\{q_P=(q_{\text{PTS}},~q_B)\in\ccalQ_P~|\nonumber\\&(q_P,L(q_{\text{PTS}}),q_P^r)\in\rightarrow_P\},
\end{align}
and, (iii) we first check if $q_P^r\in\ccalX_{\text{goal}}^{\text{suf}}$, i.e., if $(\Pi|_{B}q_P^r,L(\Pi|_{\text{PTS}}q_P^r),\Pi|_{B}q_P^r)$ and if the cost of such a self loop has zero cost, i.e., if $w_P(q_P^r,q_P^r)=0$ [line \ref{alg1:line9a}, Alg. \ref{alg:plans}]. If so, the construction of the tree is trivial, as it consists of only the root, and a loop around it with zero cost [line \ref{alg1:line9b}, Alg. \ref{alg:plans}].\footnote{Clearly, any other suffix part will have non-zero cost and, therefore, it will not be optimal and it will be discarded by Algorithm 1 [lines \ref{alg1:line13}-\ref{alg1:line14}, Alg. \ref{alg:plans}]. For this reason, the construction of the tree $\ccalG_T$ is terminated if a self-loop around $q_P^r$ is detected.}  If $q_P^r\notin\ccalX_{\text{goal}}^{\text{suf}}$, then the tree $\ccalG_T$ is constructed by Algorithm \ref{alg:tree} [line \ref{alg1:line10}, Alg. \ref{alg:plans}].
Once a tree rooted at $q_P^r=\ccalP(a)$ is constructed, a set $\ccalS_a\subseteq\ccalV_T$ is formed that collects all states $q_P\in\ccalV_T\cap\ccalX_{\text{goal}}^{\text{suf}}(q_P^r)$ [lines \ref{alg1:line9c}, \ref{alg1:line10}, Alg. \ref{alg:plans}]. 
Then for each state $q_P\in\ccalS_a$, we compute the cost $\hat{J}(\tilde{\tau}^{\text{suf},e})$ of each possible suffix plan $\tilde{\tau}^{\text{suf},e}$, for all $e\in\{1,\dots,|\ccalS_a|\}$, associated with the root $q_P^r$. By construction of the cost functions $\texttt{Cost}$ and $\hat{J}(\cdot)$, it holds that $\hat{J}(\tilde{\tau}^{\text{suf},e})=\texttt{Cost}(\ccalS_a(e))+w_{\text{PTS}}(\Pi|_{\text{PTS}}\ccalS_a(e),\Pi|_{\text{PTS}}q_P^r)$, where $\ccalS_a(e)$ stands for the $e$-th state in the set $\ccalS_a$. Among all detected suffix plans $\tilde{\tau}^{\text{suf},e}$ associated with the \textit{accepting} state $\mathcal{P}(a)$, we select the suffix plan with the minimum cost, which constitutes the suffix plan $\tau^{\text{suf},a}$ [lines \ref{alg1:line13}-\ref{alg1:line14}, Alg. \ref{alg:plans}]. This process is repeated for all $a\in\{1,\dots,|\ccalP|\}$ [line \ref{alg1:line7}, Alg. \ref{alg:plans}]. In this way, for each prefix plan $\tau^{\text{pre},a}$ we construct its corresponding suffix plan $\tau^{\text{suf},a}$, if it exists.
%Next, we optimize the structure of the tree by modifying the set of edges $\mathcal{E}_T$, as we did in Algorithm \ref{alg:prefix} [line \ref{alg7:line20}, Alg. \ref{alg:suffix}]. Given the optimized tree structure, among all detected cycles around $\mathcal{F}(f)$, we pick the optimal cycle for all $f\in\{1,\dots,|\mathcal{F}|\}$ [lines \ref{alg7:line22}-\ref{alg7:line26}, Alg. \ref{alg:suffix}]. For this purpose, given a final state $\mathcal{F}(f)$, if $\mathcal{S}_f\neq\emptyset$, first we find all possible cycles around the state $\mathcal{F}(f)$ projected onto the state-space of the PTS [lines \ref{alg7:line22}-\ref{alg7:line24}, Alg. \ref{alg:suffix}]. Among them we pick the one that has the minimum cost, in terms of the cost function \eqref{eq:cost} [lines \ref{alg7:line25}-\ref{alg7:line26}, Alg. \ref{alg:suffix}]. The resulting cycle around state $\mathcal{F}(f)$ is denoted by $\tau^{\text{suf},f}$.
 
\subsection{Construction of Optimal Discrete Plans}\label{sec:plan}
By construction, any motion plan $\tau^a=\tau^{\text{pre},a}[\tau^{\text{suf},a}]^{\omega}$, with $\mathcal{S}_a\neq\emptyset$, and $a\in\{1,\dots,|\mathcal{P}|\}$ satisfies the global LTL specification $\phi$. The cost $J(\tau^a)$ of each plan $\tau^a$  is defined in \eqref{eq:cost2}. Given an initial state $q_B^0\in\ccalQ_B^0$, among all the motion plans $\tau^a\models\phi$, we select the one with the smallest cost $J(\tau^a)$ [line \ref{alg1:line15}, Alg. \ref{alg:plans}]. The plan with the smallest cost given an initial state $q_B^0$ is denoted by $\tau_{q_B^0}$. Then, among all plans $\tau_{q_B^0}$, we select again the one with smallest cost $J(\tau_{q_B^0})$, i.e., $\tau=\tau^{a_*}$, where $a_*=\text{argmin}_{a_{q_B^0}}{J(\tau_{q_B^0})}$ [lines \ref{alg1:line15a}-\ref{alg1:line16}, Alg. \ref{alg:plans}].

\begin{rem}[Execution of plan $\tau$]
The motion plan $\tau$ generated by Algorithm \ref{alg:plans} satisfies the global LTL formula, if all robots pick their next states either synchronously or asynchronously as in \cite{ulusoy2013optimality}. For asynchronous execution of the plan $\tau$, we only need to add `traveling states' to the transition systems that capture cases where robots $i$ are traveling from states $q_i^{\ell_j}\in\ccalQ_i$ to $q_i^{\ell_e}\in\ccalQ_i$ that satisfy $(q_i^{\ell_j},q_i^{\ell_e})\in\rightarrow_i$. More details about the traveling states can be found in  \cite{ulusoy2013optimality}. Note that adding traveling states to the trees increases the computational cost of synthesizing plans that can be executed asynchronously and satisfy the assigned LTL formula.
\end{rem}

%\cite{ulusoy2014optimal}

%\begin{remark}
%Given the graph $\mathcal{G}_T^n$, the complexity of Algorithm \ref{alg:findpath} that constructs the prefix parts $\tau^{\text{pre},f}$ is $O(n_{\text{max}})$. Since this algorithm is executed  $|\mathcal{F}$ times, the computational complexity for constructing all prefix parts is $O(|\mathcal{F}|n_{\text{max}})$. Similarly, for the construction of all prefix an suffix parts and since we have to check the cost of all resulting plans $\tau^f$ 
%\end{remark}

\subsection{Complexity Analysis}
The memory resources needed to store the PBA as a graph structure $\ccalG_P=\set{\ccalV_P, \ccalE_P, w_P}$, defined in Section \ref{sec:prelim}, using its adjacency list is $O(|\ccalV_P|+|\ccalE_P|)$ \cite{sedgewick2011algorithms}. On the other hand, the memory needed to store a tree, constructed by Algorithm \ref{alg:tree}, that approximates the PBA is $O(|\ccalV_T|)$, since $|\ccalE_T|=|\ccalV_T|-1$. Due to the incremental construction of the tree we get that $|\ccalV_T|\leq|\ccalV_P|<|\ccalV_P|+|\ccalE_P|$ which shows that our proposed algorithm requires fewer memory resources compared to existing optimal control synthesis algorithms that rely on the construction of the PBA \cite{ulusoy2013optimality,ulusoy2014optimal}.

Next, observe that the time complexity of sampling the state $q_\text{PTS}^{\text{new}}$ in Algorithm \ref{alg:sample} is $O(\sum_i|\ccalQ_i|)$; see also Remark \ref{rem:reachable}. Moreover, the time complexity of extending the graph towards $q_P^{\text{new}}$ is $O(|\ccalV_T|(N+1))$; see Algorithm \ref{alg:extend}. The reason is that Algorithm \ref{alg:extend} can be equivalently written as a for-loop over the set $\ccalV_T$ where we first examine if $q_P\in\ccalV_T$ can reach $q_P^{\text{new}}$, based on the transition rule $\rightarrow_P$, and then we compute the cost of such a transition while keeping track of the node $q_P\in\ccalV_T$ that incurs the minimum cost. These calculations have cost $O(N+1)$ and the time complexity of the for-loop over  $\ccalV_T$ is $O(|\ccalV_T|)$. With this implementation of Algorithm \ref{alg:extend}, we do not need to construct the set $\ccalR_{\ccalV_T}^{\rightarrow}(q_P^{\text{new}})$.\footnote{Definition of $\ccalR_{\ccalV_T}^{\rightarrow}(q_P^{\text{new}})$ is only used in Section \ref{sec:corr} to simplify the proof of completeness and optimality of the proposed algorithm.} For the same reason, the time complexity of the rewiring step is $O(|\ccalV_T|(N+1))$; see Algorithm \ref{alg:rewire}. Finally the time complexity of Algorithm \ref{alg:findpath} that finds a path in the tree $\ccalG_T$ is $O(|\ccalV_T|)$. On the other hand, using the Dijkstra algorithm to find the shortest path from an initial to a final state or a cycle around a final state of a PBA is $O(|\mathcal{E}_P|+|\mathcal{V}_P|\log(|\mathcal{V}_p|))$; clearly, it holds that $|\mathcal{E}_P|+|\mathcal{V}_P|\log(|\mathcal{V}_P|)>|\mathcal{V}_T|$. If the PBA is represented as an implicit graph using its transition rule $\rightarrow_P$, then we can apply the uniform cost search algorithm \cite{korf2004best,russell1995modern} to find the optimal prefix and suffix paths with time and space complexity $O(b^{1+C^*/\epsilon})$, where $b$ is the branching factor of $\ccalG_P$, $C^*$ is the optimal cost of either the prefix or suffix path, and $\epsilon>0$ is the minimum increase in the cost of the path as we move from one node of $\ccalG_P$ to another. Note that this approach is also memory efficient since it does not require the explicit construction of the PBA but it can become computationally intractable (i) for dense graphs, i.e., as $b$ increases or (ii) for long paths, i.e.,  as $C^*/\epsilon$ increases. In comparison, the computational cost per iteration of our algorithm depends linearly only on $|\ccalV_T|$ and the size of the network and not on the structure of $\ccalG_P$.

\section{Correctness and Optimality} \label{sec:corr}
In this section, we first characterize the rate at which the constructed trees grow and then we provide the main results pertaining to the probabilistic completeness and optimality of the proposed Algorithm \ref{alg:plans}. In what follows, we denote by $\mathcal{G}_T^n=\{\mathcal{V}_T^{n}, \mathcal{E}_T^{n}, \texttt{Cost}\}$ the tree that has been built by Algorithm \ref{alg:tree} at the $n$-th iteration for the construction of either a prefix or suffix part. Also, we denote the nodes $q_P^{\text{rand}}$ and $q_P^{\text{new}}$ at iteration $n$ by $q_P^{\text{rand},n}$ and $q_P^{\text{new},n}$, respectively. Moreover, in the following results, we denote the reachable set of $q_P\in\ccalQ_P$ in the state-space of the PBA by:
\begin{equation}\label{eq:reachableP}
\ccalR_{P}(q_P)=\{q_{P}'\in\ccalQ_P~|~q_P\rightarrow_{P}q_{P}'\},
\end{equation}
that collects all states $q_P'\in\ccalQ_P$ that can be reached from $q_P\in\ccalQ_P$ in one hop.

 %Then, we define the probabilistic completeness of an algorithm as follows:

\begin{prop}[Growth Rate of $\ccalG_T^n$]\label{prop:growth}
Assume that the reachable set of $q_P^{\text{rand},n}=(q_{\text{PTS}}^{\text{rand},n},q_B^{\text{rand},n})$ in the PBA is non-empty, i.e., that $\ccalR_P(q_P^{\text{rand},n})=\{q_P'\in\ccalQ_P~|~q_P^{\text{rand},n}\rightarrow q_P'\}\neq\emptyset$. Then, there is at least one $b\in\{1,\dots,|\ccalQ_B|\}$ so that either the state $q_P^{\text{new},n}=(q_{\text{PTS}}^{\text{new}},\ccalQ_B(b))$ will be added to $\ccalV_T^n$ at iteration $n$ if $q_P^{\text{new},n}\notin\ccalV_T^n$, or rewiring to $q_P^{\text{new},n}$ will occur if $q_P^{\text{new},n}\in\ccalV_T^n$ and $\ccalR_{\ccalV_T^n}^{\leftarrow}(q_P^{\text{new},n})\neq\emptyset$. If $\ccalR_P(q_P^{\text{rand},n})=\emptyset$, then the tree may remain unaltered at iteration $n$.
\end{prop}

\begin{proof}
%To prove this result, we examine two cases about the reachable set $\ccalR_P(q_P^{\text{rand},n})$, i.e., $\ccalR_P(q_P^{\text{rand},n})=\emptyset$ and $\ccalR_P(q_P^{\text{rand},n})\neq\emptyset$. Then, using the definition of the transition rule $\rightarrow_P$, given in Definition \ref{defn:pba}, we show that if $\ccalR_P(q_P^{\text{rand},n})=\emptyset$, then the tree will not necessarily change at iteration $n$, and if $\ccalR_P(q_P^{\text{rand},n})\neq\emptyset$, then the tree changes due to either addition of a new state or rewiring. 
%To show this result, 
%first, recall that the state $q_{\text{PTS}}^{\text{new},n}$ always belongs to the reachable set $\ccalR_{\text{PTS}}(q_{\text{PTS}}^{\text{rand},n})\subseteq\ccalQ_{\text{PTS}}$, defined in \eqref{reachable1}, by Assumption \ref{fnew}-(i). Second, recall that given the state $q_{\text{PTS}}^{\text{new},n}$, the states $q_P^{\text{new},n}=(q_{\text{PTS}}^{\text{new},n},q_B^{\text{new},n})$, with $q_B^{\text{new},n}=\ccalQ_B(b)$ are constructed, for all $b\in\{1,\dots,|\ccalQ_B|\}$. 
To show this result, recall that a state $q_P'=(q_{\text{PTS}}',q_B')$ belongs to $\ccalR_P(q_P^{\text{rand},n})$ if $q_P^{\text{rand},n}\rightarrow_P q_P'$, i.e., if (i) $q_{\text{PTS}}^{\text{rand},n}\rightarrow_{\text{PTS}}q_{\text{PTS}}'$ and (ii) $q_B^{\text{rand},n}\stackrel{L(q_{\text{PTS}}^{\text{rand},n})}{\longrightarrow_B}q_B'$. In what follows, we first examine the case $\ccalR_P(q_P^{\text{rand},n})=\emptyset$ and then the case $\ccalR_P(q_P^{\text{rand},n})\neq\emptyset$.

Assume that $\ccalR_P(q_P^{\text{rand},n})=\emptyset$. This means that for the state $q_P^{\text{rand},n}=(q_{\text{PTS}}^{\text{rand},n},q_B^{\text{rand},n})$, there are either no states $q_{\text{PTS}}'$ that satisfy condition (i), or no states $q_B'$ that satisfy condition (ii), or possibly both. %In other words, this means that either $q_{\text{PTS}}^{\text{rand},n}$ is a terminal state of PTS, i.e., $\ccalR_{\text{PTS}}(q_{\text{PTS}}^{\text{rand},n})=\emptyset$, where $\ccalR_{\text{PTS}}(q_{\text{PTS}}^{\text{rand},n})$ is defined in \eqref{reachable1}, or, $q_{\text{PTS}}^{\text{rand},n}$ is a state at which the LTL formula $\phi$ is violated, or possibly both at the same time. 
If there are no states $q_{\text{PTS}}'$ that satisfy condition (i), then $q_{\text{PTS}}^{\text{rand},n}$ is a terminal state of the PTS, i.e., $\ccalR_{\text{PTS}}(q_{\text{PTS}}^{\text{rand},n})=\emptyset$, where $\ccalR_{\text{PTS}}(q_{\text{PTS}}^{\text{rand},n})$ is defined in \eqref{reachable1}. As a result, this implies that we cannot create any state $q_{\text{PTS}}^{\text{new},n}\in\ccalR_{\text{PTS}}(q_{\text{PTS}}^{\text{rand},n})$ and, therefore, it is trivial to see that no states will be added to $\ccalV_T^n$ at iteration $n$ and the tree will not change at iteration $n$. On the other hand, if there are no states $q_B'$ that satisfy condition (ii), i.e., if there is no $b$ such that $q_B^{\text{rand},n}\stackrel{L(q_{\text{PTS}}^{\text{rand},n})}{\longrightarrow_B}\ccalQ_B(b)$, then $q_{\text{PTS}}^{\text{rand},n}$ is a state at which the LTL formula $\phi$ is violated, by construction of the NBA. However, this does not necessarily mean that the tree will remain the same at iteration $n$. The reason is that if $\ccalR_{\text{PTS}}(q_{\text{PTS}}^{\text{rand},n})\neq\emptyset$, then there exist states $q_{\text{PTS}}^{\text{new},n}\in\ccalR_{\text{PTS}}(q_{\text{PTS}}^{\text{rand},n})$ and, consequently, states $q_P^{\text{new},n}=(q_{\text{PTS}}^{\text{new},n},\ccalQ_B(b))$ can be constructed. Such states can possibly be added to the tree if there exists a state $q_P\in\ccalV_T^n$ such that $q_P\in\ccalR^{\rightarrow}_{\ccalV_T^n}(q_P^{\text{new},n})$; see line \ref{alg3:line2} in Algorithm \ref{alg:extend}. Also, if $q_P^{\text{new},n}\in\ccalV_T^n$ and if $\ccalR_{\ccalV_T^n}^{\leftarrow}(q_P^{\text{new},n})\neq\emptyset$ then rewiring to these states may occur; see line \ref{alg4:line2} in Algorithm \ref{alg:rewire}. Clearly if $q_{\text{PTS}}^{\text{rand},n}$ is both a terminal state of PTS and a state at which $\phi$ is violated, then the tree will remain unchanged at iteration $n$.\footnote{Observe that if $\ccalR_P(q_P^{\text{rand},n})=\emptyset$, then the state $q_P^{\text{rand},n}\in\ccalV_T^n$ will remain forever a leaf node in the tree $\ccalG_T^n$.} 

Next, assume that $\ccalR_P(q_P^{\text{rand},n})\neq\emptyset$. Following the same logic as in the previous case, this means that $q_{\text{PTS}}^{\text{rand},n}$ is not a terminal state of the PTS, i.e., $\ccalR_{\text{PTS}}(q_{\text{PTS}}^{\text{rand},n})\neq\emptyset$, and $q_{\text{PTS}}^{\text{rand},n}$ is not a state at which the LTL formula $\phi$ is violated. Since $\ccalR_{\text{PTS}}(q_{\text{PTS}}^{\text{rand},n})\neq\emptyset$ and since  $q_{\text{PTS}}^{\text{new},n}\in\ccalR_{\text{PTS}}(q_{\text{PTS}}^{\text{rand},n})$ by construction, we get that $q_P^{\text{new},n}=(q_{\text{PTS}}^{\text{new},n},\ccalQ_B(b))$
satisfies condition (i) for all $b$. Next, since $q_{\text{PTS}}^{\text{rand},n}$ is not a state at which the LTL formula $\phi$ is violated, this means that there is at least one value for $b$, denoted hereafter by $\bar{b}$, such that $q_B^{\text{rand},n}\stackrel{L(q_{\text{PTS}}^{\text{rand},n})}{\longrightarrow_B}\ccalQ_B(\bar{b})$, by construction of the NBA. Thus, we get that $(q_{\text{PTS}}^{\text{new},n},\ccalQ_B(\bar{b}))\in\ccalR_P(q_P^{\text{rand},n})$. This result along with the fact that $q_P^{\text{rand},n}\in\ccalV_T^n$, by definition of $f_{\text{rand}}$, entail that $q_P^{\text{new},n}=(q_{\text{PTS}}^{\text{new},n},\ccalQ_B(\bar{b}))$ satisfies $q_P^{\text{rand},n}\in\ccalR^{\rightarrow}_{\ccalV_T^n}(q_P^{\text{new},n})$, i.e., $\ccalR^{\rightarrow}_{\ccalV_T^n}(q_P^{\text{new},n})\neq\emptyset$. This equivalently means that if $q_P^{\text{new},n}\notin\ccalV_T^n$ then the state $q_P^{\text{new},n}=(q_{\text{PTS}}^{\text{new},n},\ccalQ_B(\bar{b}))$ will be added to $\ccalV_T^n$ at iteration $n$; see line \ref{alg3:line2} in Algorithm \ref{alg:extend}. Otherwise, if $q_P^{\text{new},n}\in\ccalV_T^n$ and $\ccalR_{\ccalV_T^n}^{\leftarrow}(q_P^{\text{new},n})\neq\emptyset$ rewiring to $q_P^{\text{new},n}$ will follow (see line \ref{alg4:line2} in Algorithm \ref{alg:rewire}), completing the proof.
\end{proof}

\begin{rem}[Emptiness of $\ccalR_P(q_P^{\text{rand},n})$]
Observe that $\ccalR_P(q_P^{\text{rand},n})=\emptyset$ for a state $q_P^{\text{rand},n}=(q_{\text{PTS}}^{\text{rand},n},q_B^{\text{rand},n})$, if $q_{\text{PTS}}^{\text{rand},n}$ is either a terminal state of the PTS, i.e., $\ccalR_{\text{PTS}}(q_{\text{PTS}}^{\text{rand},n})=\emptyset$, where $\ccalR_{\text{PTS}}(q_{\text{PTS}}^{\text{rand},n})$ is defined in \eqref{reachable1}, or a state at which the LTL formula $\phi$ is violated, or both. For example, if the PTS has no terminal states and the LTL formula does not include the negation operator $\neg$, i.e., there are no states in $\ccalQ_{\text{PTS}}$ that can violate $\phi$, then $\ccalR_P(q_P)\neq\emptyset$, $\forall q_P\in\ccalQ_P$. 
\end{rem}

To show that Algorithm \ref{alg:plans} is probabilistically complete and asymptotically optimal we need first to show the following results that rely on the second Borel-Cantelli lemma \cite{grimmett2001probability} presented below. The proofs of the following lemmas are provided in Appendix \ref{sec:lemmas1}.

\begin{lem}[Borel-Cantelli \cite{grimmett2001probability}]\label{lem:bc}
Consider a sequence of independent events $A=\{A^n\}_{n=1}^{\infty}$. If $\sum_{n=1}^{\infty}\mathbb{P}(A^n)=\infty$ then $\mathbb{P}(\limsup_{n\to\infty} A^n)=1$, i.e., the probability that infinitely many events $A^n$ occur is 1.  
%i.e., the events $A^n$ occur infinitely often meaning that $A^n$ is true for an infinite number of indices $n\in\mathbb{N}_{+}$.
\end{lem}

\begin{lem}[Sampling $q_P^{\text{rand},n}$]
Consider any state $q_P\in\ccalV_T^n$ and any fixed iteration index $n$. Then, there exists an infinite number of subsequent iterations $n+k$,  where $k\in\mathcal{K}$ and $\mathcal{K}\subseteq\mathbb{N}$ is a subsequence of $\mathbb{N}$, at which the state $q_P\in\ccalV_T^n$ is selected by Algorithm \ref{alg:sample} to be the node $q_P^{\text{rand},n+k}$. %\footnote{The subsequences $\ccalK$ and $\ccalK'$}
\label{lem:qrand}
\end{lem}

Using Lemma \ref{lem:qrand} we can show the following result for the node $q_{\text{PTS}}^{\text{new},n}$.
\begin{lem}[Sampling $q_{\text{PTS}}^{\text{new},n}$]
Consider any state $q_P^{\text{rand},n}=(q_{\text{PTS}}^{\text{rand},n},q_B^{\text{rand},n})\in\ccalV_T^n$ selected by Algorithm \ref{alg:sample} and any fixed iteration index $n$. Then, for any state $q_{\text{PTS}}\in\ccalR_{\text{PTS}}(q_{\text{PTS}}^{\text{rand},n})$, where $\ccalR_{\text{PTS}}(q_{\text{PTS}}^{\text{rand},n})$ is defined in \eqref{reachable1}, there exists an infinite number of subsequent iterations $n+k$, where $k\in\mathcal{K}'$ and $\mathcal{K}'\subseteq\mathcal{K}$ is a subsequence of the sequence of $\mathcal{K}$ defined in Lemma \ref{lem:qrand}, at which the state $q_{\text{PTS}}\in\ccalR_{\text{PTS}}(q_{\text{PTS}}^{\text{rand},n})$ is selected by Algorithm \ref{alg:sample} to be the node $q_{\text{PTS}}^{\text{new},n+k}$.
\label{lem:qnew}
\end{lem}

By Lemma \ref{lem:qnew}, we have the following corollary for the state $q_P^{\text{new},n}$. 
\begin{cor}[Sampling $q_{P}^{\text{new},n}$]
Consider any state $q_P^{\text{rand},n}=(q_{\text{PTS}}^{\text{rand},n},q_B^{\text{rand},n})\in\ccalV_T^n$ selected by Algorithm \ref{alg:sample} and any fixed iteration index $n$. Then, for any state $q_{P}\in\ccalR_{P}(q_{P}^{\text{rand},n})$, where $\ccalR_{P}(q_{P}^{\text{rand},n})$ is defined in \eqref{eq:reachableP}, there exists an infinite number of iterations $n+k$, where $k\in\mathcal{K}'$ and $\mathcal{K}'$ is the subsequence defined in Lemma \ref{lem:qnew}, at which the state $q_{P}\in\ccalR_{P}(q_{P}^{\text{rand},n})$, is selected by Algorithm \ref{alg:sample} to be the node $q_{P}^{\text{new},n+k}$.
\label{cor:qPnew}
\end{cor}

Using Corollary \ref{cor:qPnew}, we can show the following result for the reachable set $\mathcal{R}_P(q_P^{\text{rand},n})$.
\begin{lem}[Reachable set $\mathcal{R}_P(q_P^{\text{rand},n})$]\label{lem:reach1}
Consider any state $q_P^{\text{rand},n}=(q_{\text{PTS}}^{\text{rand},n},q_B^{\text{rand},n})\in\ccalV_T^n$ selected by Algorithm \ref{alg:sample} and any fixed iteration index $n$. Then, Algorithm \ref{alg:tree} will add to $\ccalV_T^{n+k}$ all states that belong to the reachable set $\ccalR_{P}(q_P^{\text{rand},n})$, where $\ccalR_{P}(q_{P}^{\text{rand},n})$ is defined in \eqref{eq:reachableP}, as $k\to\infty$, with probability 1, i.e.,
\begin{equation}\label{eq:prob2}
\lim_{k\rightarrow\infty} \mathbb{P}\left(\{\mathcal{R}_P(q_P^{\text{rand},n})\subseteq\mathcal{V}_T^{n+k}\}\right)=1,
\end{equation}
\end{lem}

Using Lemma \ref{lem:qrand}, we can show that Lemma \ref{lem:reach1} holds for all nodes $q_P\in\ccalV_T^n$. This result is stated in the following corollary.
\begin{cor}[Reachable set $\mathcal{R}_P(q_P)$]\label{cor:reach2}
Given any state $q_P=(q_{\text{PTS}},q_B)\in\ccalV_T^n$, Algorithm \ref{alg:tree} will add to $\ccalV_T^n$ all states that belong to the reachable set $\ccalR_{P}(q_{P})$, as $n\to\infty$, with probability 1, i.e.,
\begin{equation}\label{eq:reach2}
\lim_{n\rightarrow\infty} \mathbb{P}\left(\{\mathcal{R}_P(q_P)\subseteq\mathcal{V}_T^n\}\right)=1,
\end{equation}
\end{cor}

Using Corollary \ref{cor:reach2}, in the next theorem, we show that Algorithm \ref{alg:plans} is probabilistically complete.
\begin{theorem}[Probabilistic Completeness]\label{prop:compl}
If there exists a solution to Problem \ref{pr:problem}, then Algorithm \ref{alg:plans} is probabilistically complete, i.e., it will find with probability 1 a motion plan $\tau$ that satisfies the LTL specification $\phi$, as $n_{\text{max}}^{\text{pre}}\to\infty$ and $n_{\text{max}}^{\text{suf}}\to\infty$.
\end{theorem}
\begin{proof}
to show this result, we need to show that Algorithm \ref{alg:tree} satisfies
\begin{equation}\label{eq:probcompl}
\lim_{n\rightarrow\infty} \mathbb{P}\left(\{\mathcal{V}_T^n\cap\mathcal{X}_{\text{goal}}\neq\emptyset\}\right)=1,
\end{equation}
for both goal regions $\ccalX_{\text{goal}}$ defined in \eqref{eq:goalPre} and in \eqref{eq:goalSuf}. 

To show \eqref{eq:probcompl}, it suffices to show that the set of nodes $\ccalV_T^n$ will eventually contain all states in the state-space $\ccalQ_P$ that are reachable from the root $q_P^r$ through a multi-hop path that respects the transition rule $\rightarrow_P$.\footnote{Recall that the root for the construction of the prefix parts is $q_P^r=(q_{\text{PTS}}^0,q_B^0)$, where $q_B^0$ is each possible state in $\ccalQ_B^0$, and for the construction of the suffix parts the root $q_P^r$ is each possible final state detected during the construction of the prefix parts.} We collect these states in the set $\ccalR^{\infty}_P(q_P^r)$, defined as 
\begin{equation}\label{Rinf}
%\ccalR^{\infty}_P(q_P^r)=\ccalR_P(...\ccalR_P(\ccalR_P(q_P^r))).\footnote{The superscript $\infty$ means that the reachable set %collects all states that are reachable from $q_P^r$ in the state-space $\ccalQ_P$ in any number of hops.} 
\ccalR^{\infty}_P(q_P^r)=\cup_{m=1}^{\infty}\ccalR_P^m(q_P^r),
\end{equation}
where the reachable set $\ccalR_P^m(q_P^r)$ collects all states $q_P\in\ccalQ_P$ that are reachable from the root $q_P^r$ in $m$-hops.\footnote{The superscript $\infty$ in \eqref{Rinf} means that the reachable set collects all states that are reachable from $q_P^r$ in the state-space $\ccalQ_P$ in any number of hops.} 

In mathematical terms, we need to show that
\begin{equation}\label{eq:prob3}
\lim_{n\rightarrow\infty} \mathbb{P}\left(\{\ccalR^{\infty}_P(q_P^r)=\mathcal{V}_T^n\}\right)=1.
\end{equation}
%where $\ccalR_{P}(q_P^r)$ is the reachable set of the root $q_P^r$, i.e., $\ccalR_{P}(q_P^r)=\{q_{P}\in\ccalQ_P|q_P\rightarrow_{P}q_{P}^r\}$; see \eqref{eq:reachableP}. 
Since we assume that there exists a solution to Problem \ref{pr:problem}, if \eqref{eq:prob3} holds, then $\ccalR_P^{\infty}(q_P^r)\cap\mathcal{X}_{\text{goal}}\neq\emptyset$, which equivalently means that \eqref{eq:probcompl} holds, as well.

To show that \eqref{eq:prob3} holds it suffices to show that the event $\{\mathcal{R}_P(q_P)\subseteq\mathcal{V}_T^n,~\forall q_P\in\ccalV_T^n\}$, is equivalent to the event $\{\ccalR^{\infty}_P(q_P^r)=\mathcal{V}_T^n\}$. Then the result follows due to \eqref{eq:reach2} in Corollary \ref{cor:reach2}.
To show this, %we need to show that if $q\in\mathcal{R}(q_P)$, with $q_P\in\ccalV_T^n$, then $q\in\ccalR^{\infty}_P(q_P^r)$ and vice versa, i.e., if $q\in\ccalR^{\infty}_P(q_P^r)=\ccalV_T^n$, then  $q\in\mathcal{R}(q_P)$, for some $q_P\in\ccalV_T^n$.
observe that if $q\in\mathcal{R}_P(q_P)$, with $q_P\in\ccalV_T^n$ then, clearly $q$ is reachable from the root $q_P^r$ through a multi-hop path, since by construction of the tree there is a multi-hop path that connects $q_P$ to the root $q_P^r$, i.e., $q\in\ccalR^{\infty}_P(q_P^r)$. Next, if $q\in\ccalR^{\infty}_P(q_P^r)=\ccalV_T^n$ then there exists a node $q_P\in\ccalR^{\infty}_P(q_P^r)=\ccalV_T^n$, so that $q$ is reachable from $q_P$ due to \eqref{Rinf}, i.e., $q\in\ccalR_P(q_P)$, completing the proof. 
%Next, we show by contradiction that \eqref{eq:prob3} holds. Particularly, assume that \eqref{eq:prob3} does not hold, i.e., there exists a state $q_P'$ for which it holds that (i) $q_P'\in\ccalR_P^{\infty}(q_P^r)$, (ii) $q_P'\notin\ccalV_T^n$, for all $n\in\mathbb{N}$. Due to (i), we get that there exists a state $q_P''\in\ccalV_T^n$ such that $q_P'\in\ccalR_P(q_P'')$,  which contradicts the result shown in Corollary \ref{cor:reach2}. Therefore, \eqref{eq:prob3} holds and, as a result, so does \eqref{eq:probcompl} completing the proof. 
\end{proof}

To show that Algorithm \ref{alg:plans} is asymptotically optimal, we need the following corollary that is proved in Appendix \ref{sec:lemmas1}.
\begin{cor}[Sampling $q_{P}^{\text{new},n}$]\label{cor:qPnew2}
Consider any state $q_P\in\ccalV_T^n$ and any fixed iteration $n$. Then, there exists and infinite number of subsequent iterations $n+k$,
where $k\in\mathcal{K}'$ is the subsequence defined in Lemma \ref{lem:qnew}, at which the state $q_{P}\in\ccalV_T^n$ is selected by Algorithm \ref{alg:sample} to be the node $q_{P}^{\text{new},n+k}$.
\end{cor}

\begin{theorem}[Asymptotic Optimality]\label{prop:opt}
Assume that there exists an optimal solution to Problem \ref{pr:problem}. Then, Algorithm \ref{alg:plans} is asymptotically optimal, i.e., the optimal motion plan for a given LTL formula $\phi$ will be found with probability 1, as $n_{\text{max}}^{\text{pre}}\to\infty$ and $n_{\text{max}}^{\text{suf}}\to\infty$. In other words, the discrete motion plan $\tau$ that is generated by this algorithm for a given global LTL specification $\phi$ satisfies
%_{n_{\text{max}}^{\text{pre}}}^{n_{\text{max}}^{\text{suf}}}
\begin{equation}\label{eq:opt1}
\mathbb{P}\left(\left\{\lim_{n_{\text{max}}^{\text{pre}}\to\infty, n_{\text{max}}^{\text{suf}}\to\infty}J(\tau)=J^*\right\}\right)=1,
\end{equation} 
where $J$ is the cost function \eqref{eq:cost2}, $J^*$ is the optimal cost, and $n_{\text{max}}^{\text{pre}}$ and $n_{\text{max}}^{\text{suf}}$ are the maximum number of iterations of Algorithm \ref{alg:tree} for the prefix and suffix part synthesis, respectively.
%As the maximum number of samples $n_{\text{max}}$ for Algorithms \ref{alg:prefix} and \ref{alg:suffix} goes to infinity the cost of the resulting motion plan $\tau_{n_{\text{max}}}$ 
\end{theorem}
\begin{proof}
To show that Algorithm \ref{alg:tree} is asymptotically optimal, we will show that as $n_{\text{max}}^{\text{pre}}\to\infty$ and $n_{\text{max}}^{\text{suf}}\to\infty$, all states $q_P\in\ccalV_T^n$ are connected to the root $q_P^r$ through the path in the PBA that has the minimum cost. A necessary and sufficient condition for this is to show that as $n_{\text{max}}^{\text{pre}}\to\infty$ and $n_{\text{max}}^{\text{suf}}\to\infty$, the set of edges $\ccalE_T^n$ of the tree $\ccalG_T^n$ constructed by Algorithm \ref{alg:tree} contains the transitions between states in the plans $\tau^{\text{pre},*}$ and $\tau^{\text{suf},*}$, where $\tau^*=\tau^{\text{pre},*}[\tau^{\text{suf},*}]^{\omega}$ is the optimal motion plan. Specifically, to prove that, we first show that every node $q_P\in\ccalV_T^n$ will get rewired only a \textit{finite} number of times as $n\to\infty$, which means that the cost of each node will converge to a finite number, as $n\to\infty$ (necessary condition that guarantees convergence). Then we show that this means that the cost of every node $q_P$ has converged to its optimal cost, which equivalently means that the set of edges $\ccalE_T^n$ of the tree $\ccalG_T^n$ constructed by Algorithm \ref{alg:tree} contains the transitions between states in the optimal prefix $\tau^{\text{pre},*}$ and suffix $\tau^{\text{suf},*}$ part (sufficient condition).

To show that every node $q_P\in\ccalV_T^n$ will get rewired only a finite number of times as $n\to\infty$ we use contradiction. Assume that as $n\to\infty$, there exists a node $q_P$ for which rewiring will take place infinitely often. Equivalently, this means that the path that connects $q_P$ to the root $q_P^r$ of the tree will change infinitely often. This can happen if: (i) The sets $\ccalV_T^n$ and $\ccalE_T^n$ are infinite as $n\to\infty$. This is not possible by construction of the PBA, since $|\ccalV_T^n|\leq|\ccalQ_P|<\infty$ for all $n\in\mathbb{N}_{+}$; (ii) The sets  $\ccalV_T^n$ and $\ccalE_T^n$ are finite but the cost of each node $q_P\in\ccalV_T^n$ is not bounded below. This is not possible, by definition of the PBA, since the cost of all states $q_P\in\ccalQ_P$ is bounded below by $0$; and (iii) The sets  $\ccalV_T^n$ and $\ccalE_T^n$ are finite but the path that connects $q_P$ to $q_P^r$ in the tree $\ccalG_T^n$, denoted by $\pi^{n}(q_P)$, reoccurs periodically. This means that there exist constants $\bar{n}$, $K>0$ so that $\pi^{m\bar{n}}(q_P)=\pi^{m\bar{n}+K}(q_P)$ for all $n>m\bar{n}$ and $m\in\mathbb{N}_{+}$
% there exists an iteration $\bar{n}$, such that as $n\to\infty$ and $n\geq\bar{n}$, at iterations $n=m\bar{n}$ and $n=m\bar{n}+K$, where $K>0$ is the period and $m\in\mathbb{N}_{+}$, we have that $\pi^{m\bar{n}}(q_P)=\pi^{m\bar{n}+K}(q_P)$ for all $m\in\mathbb{N}_{+}$. 

In what follows, we show by contradiction that case (iii) is not possible either. With slight abuse of notation, we denote by $\texttt{Cost}^n(q_P)$ the cost of $q_P$ at iteration $n$. Since, by assumption, $q_P$ gets rewired indefinitely, we have that
\begin{equation}
\texttt{Cost}^{m\bar{n}}(q_P)>\texttt{Cost}^{m\bar{n}+1}(q_P)>\dots>\texttt{Cost}^{m\bar{n}+K}(q_P),\nonumber
\end{equation}
for any $\bar{n},K>0$ and for all $m\in\mathbb{N}_{+}$. Clearly, this contradicts the fact that $\pi^{m\bar{n}}(q_P)=\pi^{m\bar{n}+K}(q_P)$ which implies that $\texttt{Cost}^{m\bar{n}}(q_P)=\texttt{Cost}^{m\bar{n}+K}(q_P)$. Therefore, as $n\to\infty$, every node $q_P\in\ccalV_T^n$ will get rewired only a finite number of times. %This also means that, as $n\to\infty$, after a finite number of rewiring steps for each node, none of the nodes will ever get rewired again. 

Next, we show by contradiction that when rewiring has ended for all nodes in $\ccalV_T^n$, every node in the constructed tree has achieved its optimal cost. Specifically, assume that rewiring has ended for all nodes and that there exists at least one node $q_P\in\ccalV_T^n$ that has not reached its optimal cost. This means that there exists at least one pair of nodes $q_P\in\ccalV_T^n$ and $q_P'\in\ccalV_T^n$ such that (i) $q_P\rightarrow_P q_P'$ and (ii) if $q_P$ gets rewired to $q_P'$ the cost of $q_P$ will decrease. However, by Corollary \ref{cor:qPnew2}, $q_P'$ will be selected by Algorithm \ref{alg:sample} to be $q_P^{\text{new},n+k}$ infinitely often, meaning that $q_P'$ will eventually get rewired to $q_P$ by Algorithm 2. This contradicts the fact that rewiring has ended for all nodes in $\ccalV_T^n$. Therefore, when all nodes have been rewired finitely many times and the rewiring process has terminated, every node $q_P$ has achieved its optimal cost. 
%since by construction of the rewiring step, for every node $q_P\in\ccalV_T^n$ there is not any other node $q_P'\in\ccalV_T^n$ such $q_P\rightarrow_P q_P'$ such that if $q_P$ gets rewired to $q_P'$. 
%Now it remains to show that once this happens the node $q_P$ will achieve its optimal cost.  We show that using contradiction. Specifically, assume that no rewiring occurs 

%Also, since the cost of all nodes is bounded below, as discussed before, and since every rewiring decreases the cost of nodes $q_P$ that get rewired to $q_P^{\text{new},n+k}$, i.e., $\texttt{Cost}^{n+k}(q_P)\geq\texttt{Cost}^{n+k+1}(q_P)$, $\forall n,~k\in\mathbb{N}_{+}$, we conclude that after a \textit{finite} number of iterations $I(\ccalV_T^n)$, the optimal cost for each state in $\ccalV_T^n$ will be achieved.% Clearly, the required number of iterations to achieve the optimal cost depends on the set $\ccalV_T^n$. Hereafter, we denote by $I(\ccalV_T^n)$ the aforementioned number of iterations.}

Finally note that by Theorem \ref{prop:compl}, as $n\to\infty$ we have that $\ccalV_T^n=\ccalR^{\infty}_P(q_P^r)$ with probability 1. Since $\ccalR_P^\infty$ is fixed (because $\ccalQ_P$ is finite), so is $\ccalV_T^n$ as $n\to\infty$. Moreover, since Problem \ref{pr:problem} has a solution, $\ccalV_T^n$ also includes the final state $q_P^F$ that appears in $\tau^{\text{pre},*}$ as $n\to\infty$. Therefore, by the above argument, as $n\to\infty$, every state in $\ccalV_T^n$, including the final state $q_P^F$, will reach its optimal cost with probability 1. This means that the cost of the path that corresponds to the prefix part constructed by Algorithm \ref{alg:tree} that connects the final state $q_P^F$ to the root $q_P^r$ will be $\hat{J}(\tau^{\text{pre},*})$. Following the same logic, the cost of the respective suffix part, i.e., the cycle around the final state $q_P^F$ will be $\hat{J}(\tau^{\text{suf},*})$ completing the proof. %Next, since the optimality criterion for both the algorithm described in Section \ref{sec:prelim} and our proposed algorithm is the same, defined as $J(\tau)=J_f(\tau^{\text{pre}})+J_f(\tau^{\text{suf}})$ for a plan $\tau=\tau^{\text{pre}}[\tau^{\text{suf}}]^{\omega}$, the resulting plan $\tau_{n_{\text{max}}^{\text{pre}}}^{n_{\text{max}}^{\text{suf}}}$ given by our proposed algorithm will coincide with $\tau^*$. Thus, we proved that $J(\tau_{n_{\text{max}}^{\text{pre}}}^{n_{\text{max}}^{\text{suf}}})=J^*$ with probability 1 if $n_{\text{max}}^{\text{pre}}\to\infty$ and $n_{\text{max}}^{\text{suf}}\to\infty$, which completes the proof.
\end{proof}

\section{Numerical Experiments}\label{sec:sim}
In this section, we present two case studies, implemented using MATLAB R2016a on a computer with Intel Xeon CPU at 2.93 GHz and 4 GB RAM, that illustrate our proposed algorithm and compare it to existing methods. The first case study pertains to a motion planning problem with a PBA that has $3,099,363,912$ states. Recall that the state-space of the PBA defined in Definition \ref{defn:pba} has $\Pi_{i=1}^N|\mathcal{Q}_i||\mathcal{Q}_B|$ states. This problem cannot be solved by standard optimal control synthesis algorithms, discussed in Section \ref{sec:introduction}, that rely on the explicit construction of the PBA defined in Section \ref{sec:problem}, due to memory requirements. Representing the PBA as an \textit{implicit} graph and using the uniform-cost search (UCS) algorithm \cite{korf2004best,russell1995modern} to find the optimal plan also failed to detect a final state within 24 hours.
%Notice that in this case, we do not explicitly construct the PBA and, as result, such an approach is not resource demanding in terms of memory requirements. Specifically, at each iteration of UCS a node is selected from a priority list that belongs to the state-space $\ccalQ_P$ of the PBA. Then, we only need to know the neighbors of this node in $\ccalQ_P$. These neighbors can be computed using the transition rule $\rightarrow_P$ of the PBA and the individual wTSs and the NBA. 
In fact, our implementation for both approaches of the algorithm presented in Section \ref{sec:prelim} cannot provide a plan for PBA with more than few millions of states and transitions either due to memory requirement or excessively high runtime. This problem cannot be solved by the off-the-shelf model checker PRISM either, due to excessive memory requirements. Our implementation of \cite{kantaros15asilomar} failed also to provide a motion plan for the considered case study due to the large state-space of the resulting PBA. On the other hand, NuSMV can generate a feasible, but not the optimal, plan that satisfies the considered LTL-based task. A direct comparison with \cite{vasile2013sampling} cannot be made, since in \cite{vasile2013sampling} samples of the robot positions are drawn from the continuous space, which is not the case here. Note, however, that 
as the size of the regions that observe the atomic propositions in \cite{vasile2013sampling} becomes smaller, more samples are needed to construct expressive enough transition systems that are needed to generate a motion plan. In this case, the state space of the PBA may become too large to store, let alone apply graph search methods. This issue becomes more pronounced, as the size of the NBA increases. Also, scalability in \cite{vasile2013sampling} relies on the construction of a sparse graph rather than a tree as in our proposed method. However,  sparsity of the graph is lost as the number of samples increases. Moreover, we also compare the proposed control synthesis algorithm to our previous work \cite{kantaros2017sampling} and we show a significant improvement in terms of scalability, due to the fast exploration of the state-space of the PBA as predicted by Proposition \ref{prop:growth}. In the second case study, we consider a motion planning problem with a PBA that has 6,144 states. This state-space is small enough to manipulate and construct an optimal plan using the standard method described in Section \ref{sec:prelim}. In this simulation study, we examine the performance of the proposed algorithm in terms of runtime and optimality. In what follows, we consider discrete uniform distributions for both $f_{\text{rand}}$ and $f_{\text{new}}$ for all iterations $n$ and also we assume that the weights $w_i$ defined in Definition \ref{defn:wTS} represent distance between locations.%, since it cannot reduce the number of states of the given transition systems. The reason is that in the following simulation studies every state of the transition systems observes atomic propositions that appear in the temporal specification and, therefore, there are no states in any $\text{wTS}_i$ that can be discarded.

\subsection{Case Study I}
In the first simulation study, we consider a team of $N=9$ robots residing in a workspace with $W=9$ regions of interest. The transition system describing the motion of each robot has $|\mathcal{Q}_i|=9$ states and 39 transitions, including self-loops around each state, as shown in Figure \ref{fig:workspace1}. The collaborative task that is assigned to the robots describes an intermittent connectivity problem, that was defined in our previous work \cite{kantaros2016distributedInterm}. Specifically, the robots move along the edges of a mobility graph and communicate only when they meet at the vertices of this graph. The communication network is intermittently connected if communication occurs at the vertices of the mobility graph infinitely often. This intermittent connectivity requirement can be captured by a global LTL formula, which for the case study at hand takes the form $\phi=[\Box\Diamond(\pi_1^{\ell_5}\wedge\pi_2^{\ell_5})]\wedge[\Box\Diamond(\pi_2^{\ell_1}\wedge\pi_3^{\ell_1}\wedge\pi_4^{\ell_1})]\wedge[\Box\Diamond(\pi_4^{\ell_7}\wedge\pi_5^{\ell_7}\wedge\pi_6^{\ell_7})]\wedge[\Box\Diamond(\pi_6^{\ell_8}\wedge\pi_7^{\ell_8})]\wedge[\Box\Diamond(\pi_7^{\ell_4}\wedge\pi_8^{\ell_4})]\wedge[\Box\Diamond(\pi_8^{\ell_3}\wedge\pi_9^{\ell_3})]\wedge [\neg (\pi_1^{\ell_5}\wedge\pi_2^{\ell_5})\mathcal{U}\pi_1^{\ell_7}]$.
%\begin{align}\label{eq:task1}
%\phi=&[\Box\Diamond(\pi_1^{\ell_5}\wedge\pi_2^{\ell_5})]\wedge[\Box\Diamond(\pi_2^{\ell_1}\wedge\pi_3^{\ell_1}\wedge\pi_4^{\ell_1})]\nonumber\\&\wedge[\Box\Diamond(\pi_4^{\ell_7}\wedge\pi_5^{\ell_7}\wedge\pi_6^{\ell_7})]\wedge[\Box\Diamond(\pi_6^{\ell_8}\wedge\pi_7^{\ell_8})]\nonumber\\&\wedge[\Box\Diamond(\pi_7^{\ell_4}\wedge\pi_8^{\ell_4})]\wedge[\Box\Diamond(\pi_8^{\ell_3}\wedge\pi_9^{\ell_3})]\nonumber\\&\wedge (\neg (\pi_1^{\ell_5}\wedge\pi_2^{\ell_5})\mathcal{U}\pi_1^{\ell_7}).
%%\phi=&(\Box\Diamond\pi_1^{\ell_2})\wedge(\Box\Diamond\pi_1^{\ell_6}) \wedge(\Box\Diamond(\pi_1^{\ell_3}\wedge\pi_2^{\ell_3}))\wedge(\Box\Diamond(\pi_3^{\ell_4}\wedge\pi_4^{\ell_4}))\nonumber\\&\wedge(\Box\Diamond\pi_4^{\ell_5})\wedge(\Box\Diamond\pi_5^{\ell_1})\wedge(\Box\Diamond\pi_6^{\ell_8})\wedge(\Box\Diamond(\pi_6^{\ell_7}\wedge\pi_7^{\ell_7})),
%\end{align}
%In this LTL formula, $\Box$, $\Diamond$, and $\mathcal{U}$ stand for the temporal operators `always', `eventually', and `until' respectively, and $\wedge$ and $\neg$ represent the Boolean conjunction and negation operator. 
In words, (a) robots 1 and 2 need to  meet at location $\ell_5$ infinitely often, (b) robots 2, 3 and 4 need to meet at location $\ell_1$, infinitely often, (c) robots 4, 5, and 6 need to meet at location $\ell_7$, infinitely often, (d) robots 6 and 7 need to meet at location $\ell_8$ infinitely often, (e) robots 7 and 8 need to meet at location $\ell_4$, infinitely often, (f) robots 8 and 9 need to meet at location $\ell_3$, infinitely often, and (g) robots 1 and 2 should never meet at location $\ell_5$ until robot $1$ visits location $\ell_7$ to collect some available information. This LTL formula corresponds to a NBA with $|\mathcal{Q}_B|=8$ states, $|\ccalQ_B^0|=1$, $|\ccalQ_B^F|=1$, and $36$ transitions.\footnote{The translation of the LTL formula to a NBA was made by the tool developed in \cite{gastin2001fast}.} 
%UN-COMMENT...
\begin{figure}[t]
  \centering
     \subfigure[$\text{wTS}_i$ for Simulation Study I]{
    \label{fig:workspace1}
  \includegraphics[width=0.26\linewidth]{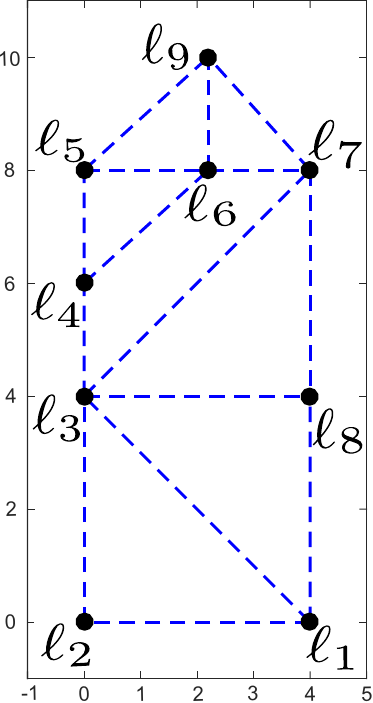}}
  \subfigure[$\text{wTS}_i$ for Simulation Study II]{
    \label{fig:workspace2}
  \includegraphics[width=0.5\linewidth]{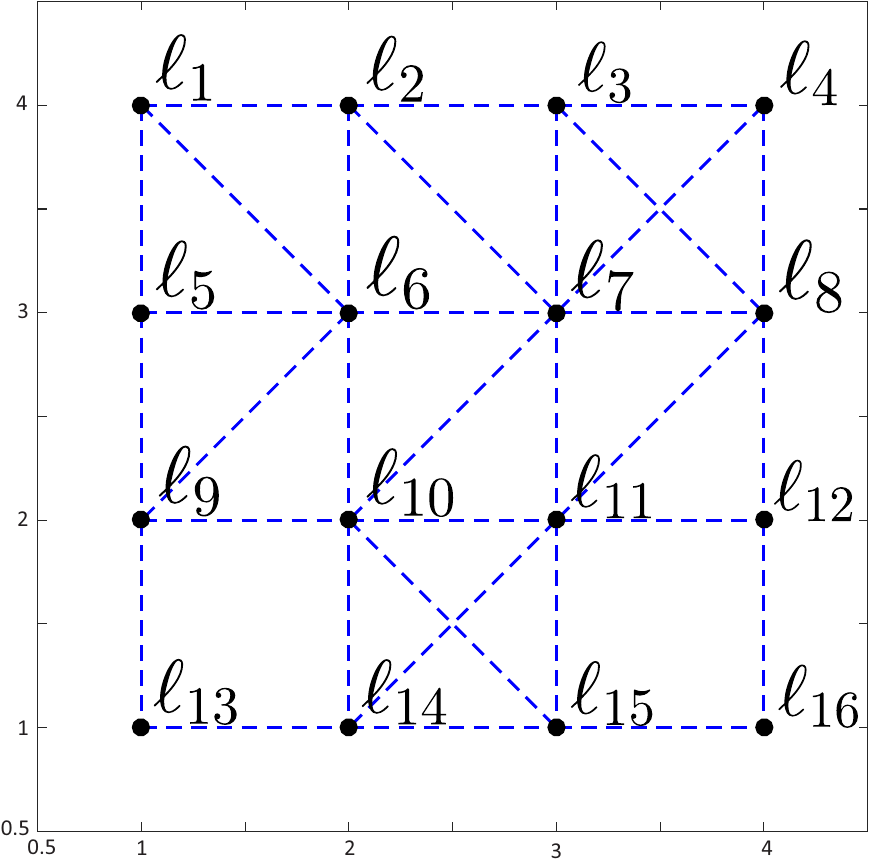}}
  \caption{Graphical depiction of the transition systems $\text{wTS}_i$, for all robots $i$ used in simulation study I (Figure \ref{fig:workspace1}) and II (Figure \ref{fig:workspace2}). Black disks represent the states of $\text{wTS}_i$ and red edges stand for feasible transitions among the states.}
\end{figure}

%\begin{table*}  
%\centering
%\begin{tabular}{ |p{3cm}||p{3cm}|p{3cm}|p{3cm}|p{3cm}| }
% \hline
% \multicolumn{5}{|c|}{Proposed Method - Construction of prefix part with $n^{\text{pre}}_{\text{max}}=6500$} \\
% \hline
%Experiment &Final Size of Tree $|\ccalV_T^{n^{\text{pre}}_{\text{max}}}|$ & Number of Detected Final States & Detection of First Final State at $n=$ & Runtime (mins)\\
% \hline
%1 &24267   & 7    &22895&   22.5\\
%2 &17845 &   0  & -   &14.95\\
%3 &27756 &21 & 17584&  38.1\\
%4 &16749    &0 & -&  14.74\\
%5 &27174 &   21  & 18741&35.01\\
% \hline
%\end{tabular}
% \label{tab:tab1} 
%\end{table*}
%
%\begin{table*} 
%\centering
%\begin{tabular}{ |p{3cm}||p{3cm}|p{3cm}|p{3cm}|p{3cm}| }
% \hline
% \multicolumn{5}{|c|}{Algorithm \cite{kantaros2017sampling} - Construction of prefix part with $n^{\text{pre}}_{\text{max}}=6500$} \\
% \hline
%Experiment &Final Size of Tree $|\ccalV_T^{n^{\text{pre}}_{\text{max}}}|$ & Number of Detected Final States & Detection of First Final State at $n=$ & Runtime (secs)\\
% \hline
%1 &3833   & 0    &-&   59\\
%2 &4230 &   0  & -   &65.1\\
%3 &781 &0 & -&  29.8\\
%4 &3786    &0 & -&  66.3\\
%5 &1373 &   0  & -&35.15\\
% \hline
%\end{tabular}
% \label{tab:tab2} 
%\end{table*}

In Algorithm \ref{alg:plans}, we select $n_{\text{max}}^{\text{pre}}=n_{\text{max}}^{\text{suf}}=6500$. The first final state was detected in $13$ minutes. After 6500 iterations that took approximately 30 minutes, $|\ccalP|=11$ final states were detected and a tree $\mathcal{G}_T$ with $|\mathcal{V}_T|=23893$ nodes was constructed. %presents the total number of times per iteration $n$ that every robot $i$ transmits a set of nodes $\ccalS(\ccalD_i)$ to the robot that has stored the sample $q_P^{\text{new}}$, after rewiring, when the distributed Algorithm \ref{alg:dtree} is applied; see also Remark \ref{rem:com}. The average number of such communication events is 1.2 per iteration $n$, i.e., approximately one robot per iteration $n$ transmits the set of nodes $\ccalS(\ccalD_i)$ to the robot that has stored the sample $q_P^{\text{new}}$. 
%
%\subsubsection{Comparison with \cite{kantaros2017sampling}}
Figure \ref{fig:rejectSim1} depicts the number of rejected states per iteration $n$ of Algorithm \ref{alg:tree}, i.e., samples $q_P^{\text{new},n}\notin\ccalV_T^n$ that cannot be added to $\ccalV_T^n$ at each iteration $n$, during the construction of the prefix part. Observe in Figure \ref{fig:rejectSim1}, that at every iteration $n$, there is at least one sample $q_P^{\text{new}}$ that either is added to the tree if $q_P^{\text{new},n}\notin\ccalV_T^n$ or enables the rewiring operation if $q_P^{\text{new},n}\in\ccalV_T^n$. Given the detected final states, the construction of the suffix part follows, where the average time to compute each suffix part $\tau_{\text{pre},a}$, $a=\{1,\dots,11\}$ was 17 minutes. Given the prefix and suffix parts, the resulting optimal motion plan that satisfies the considered LTL task was synthesized in less than 1 second and its cost is $J(\tau)=\hat{J}(\tau^{\text{pre}})+\hat{J}(\tau^{\text{suf}})=387.2293+387.2293=774.4586$ meters.
Notice also that storage of each of the constructed trees required approximately only 3MBs while the computation of paths over the trees associated with either the prefix or the suffix part required 0.02 seconds on average rendering our approach resource and computationally efficient.

\begin{figure}[t]
  \centering
     \subfigure[Simulation Study I]{
    \label{fig:rejectSim1}
  \includegraphics[width=0.42\linewidth]{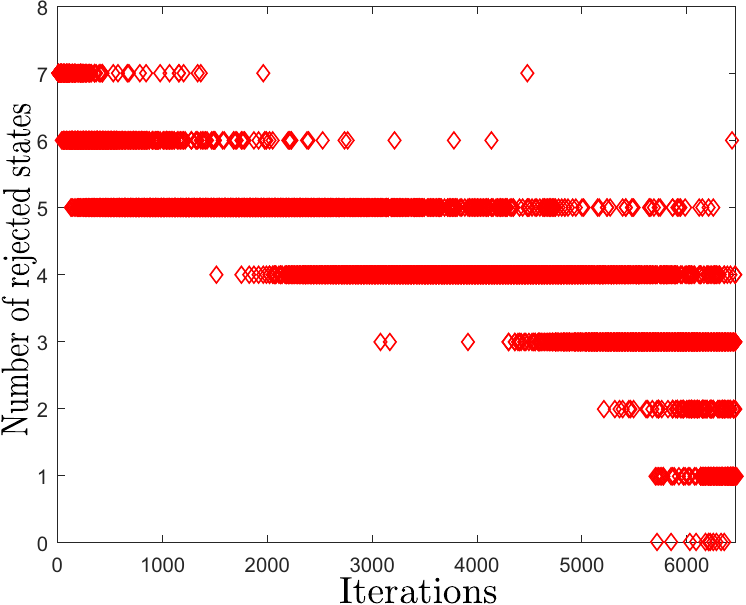}}
  \subfigure[Simulation Study II]{
    \label{fig:rejectSim2}
  \includegraphics[width=0.42\linewidth]{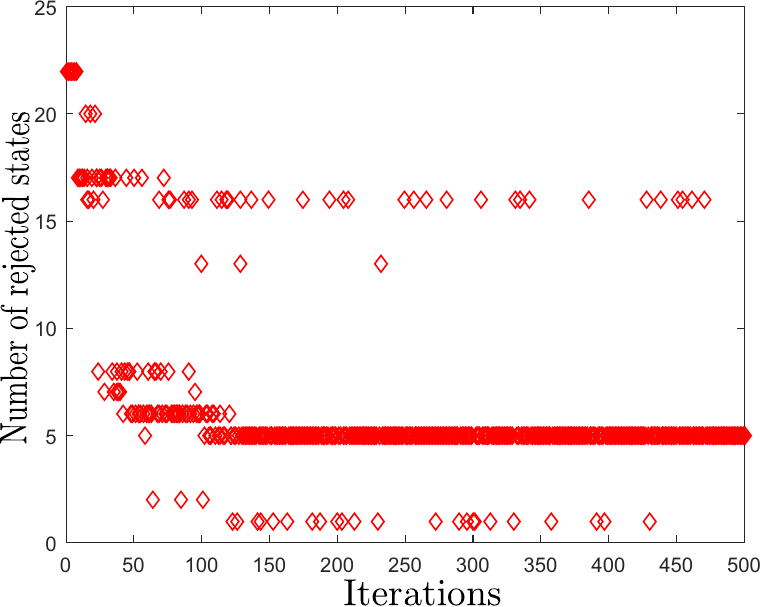}}
  \caption{Graphical depiction of the number of rejected states per iteration $n$ after running Algorithm \ref{alg:tree} for 6500 and 500 iterations for simulation studies I and II, for the synthesis of the prefix part. The resulting trees have 23893 and 3621 nodes, respectively. Red diamonds represent the number of rejected states at each iteration. At iteration $n$ of Algorithm \ref{alg:tree}, a state sampled from $\mathcal{Q}_{\text{PTS}}$ is taken. Given this state, $|\mathcal{Q}_B|$ states that belong to $\mathcal{Q}_P$ are created. Consequently, at iteration $n$ at most $|\mathcal{Q}_B|$ states can be rejected or accepted. Recall that $|\ccalQ_B|=8$ and $|\ccalQ_B|=24$, for simulation studies I and II, respectively.}
\end{figure}

\subsubsection{Comparison with \cite{kantaros2017sampling}}

\begin{figure}[t]
\centering
  \includegraphics[width=0.55\linewidth]{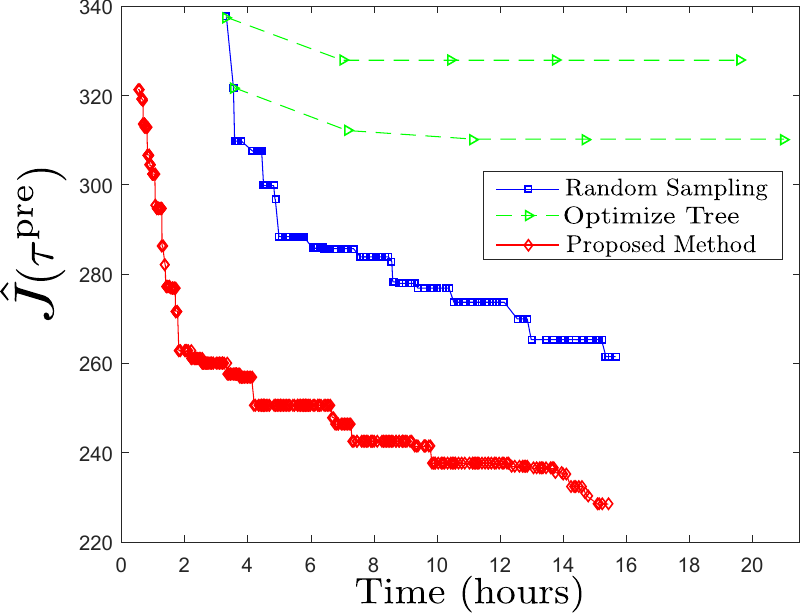}%{CostComparison.eps}%_v2
  \caption{Simulation Study I: Comparison of the average cost of the best prefix part constructed by Algorithm \ref{alg:plans} (red line) and the first part of \cite{kantaros2017sampling} (blue line) with respect to time. The average cost of the best prefix part is reported every time a new final state is detected. Red diamonds and blue squares denote a new final state detected by Algorithm \ref{alg:plans} and \cite{kantaros2017sampling}, respectively. The green dashed lines represent the evolution of the average cost of the best prefix part, when the second part of \cite{kantaros2017sampling} is executed that stops extending the tree and instead it optimizes its structure.}
  \label{fig:compcost}
\end{figure}

\begin{figure}[t]
\centering
\subfigure[]{
  \label{fig:comptrees}
  \includegraphics[width=0.49\linewidth]{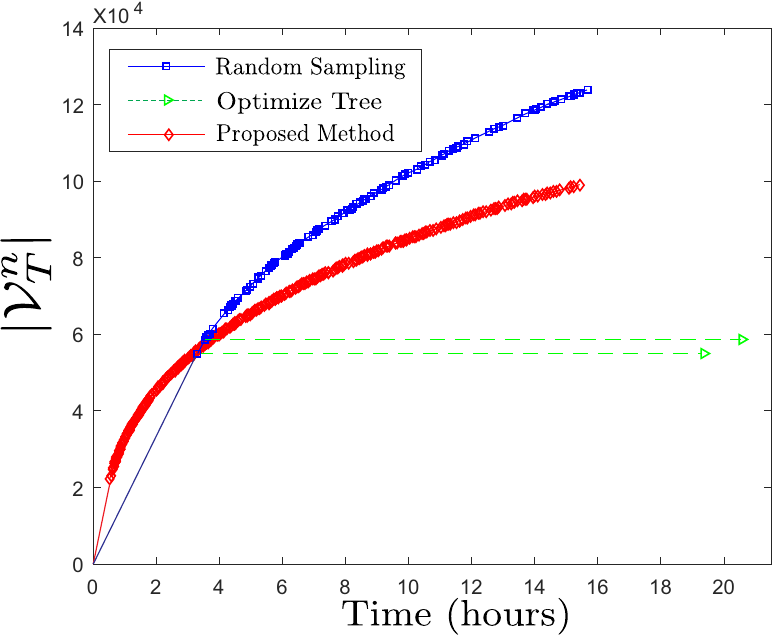}}%{Comp_sizeTree.eps}%_v2
  \subfigure[]{
  \label{fig:compTreeIter}
   \includegraphics[width=0.49\linewidth]{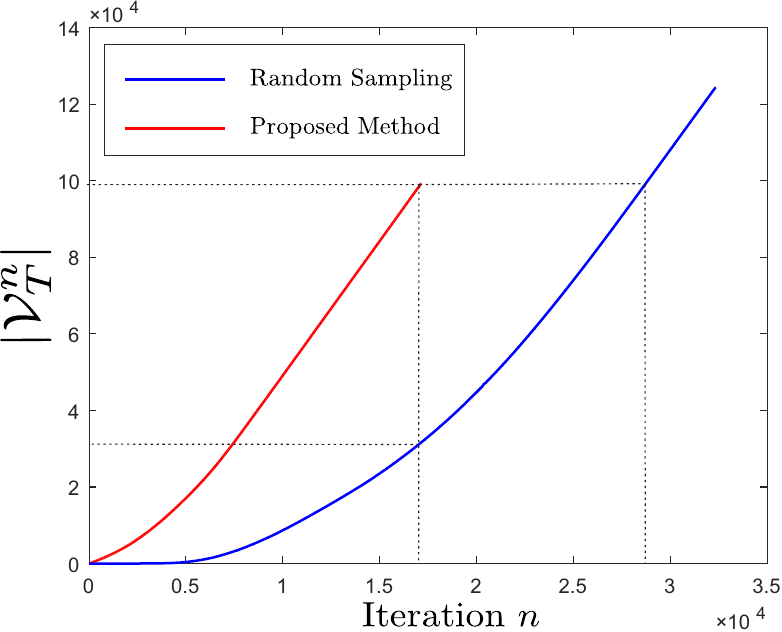}}%_v2}
  \caption{Simulation Study I: Figure \ref{fig:comptrees} compares the average size of the tree built by Algorithm \ref{alg:plans} (red line) and \cite{kantaros2017sampling} (blue line) with respect to time, during the synthesis of the prefix part. The size of the tree is reported every time a new final state is detected. Red diamonds and blue squares denote a new final state detected by Algorithm \ref{alg:plans} and \cite{kantaros2017sampling}, respectively. The green dashed lines represent the average time that it takes for the second part of \cite{kantaros2017sampling} to optimize the tree structure when the first two final states are detected. Figure \ref{fig:compTreeIter} illustrates the effect of sampling in the growth rate of the trees during the synthesis of the prefix part. The red and the blue lines show the evolution of the size of the tree with respect to iterations $n$ when the proposed method and the first part of \cite{kantaros2017sampling} are executed, respectively.  }
  %\label{fig:comptrees}
\end{figure}
%The red and green dashed lines connect the initial zero size of the tree to the average size of the tree when the first final state is detected.

%\begin{figure}[t]
%\centering
 % \includegraphics[width=0.55\linewidth]{TreesIter.eps}%_v2
%  \caption{Simulation Study I: Graphical depiction of the effect of sampling in the growth rate of the trees. The red and the blue lines show the evolution of the size of the tree with respect to iterations $n$ when the proposed method and the first part of \cite{kantaros2017sampling} are executed, respectively.}
%  \label{fig:compTreeIter}
%\end{figure}

Next, we compare the proposed sampling approach with our previous work \cite{kantaros2017sampling} in terms of their ability to minimize cost and grow the tree as a function of runtime. Before presenting the comparative results, recall first that \cite{kantaros2017sampling} consists of two parts. In the \textit{first part}, samples are drawn \textit{randomly} from $\ccalQ_P$. If these samples do not already belong to the tree but they are reachable from the tree, then the tree is extended towards them and the rewiring step follows, as in our proposed algorithm. If they are not reachable from the tree, then they are rejected, as in our proposed algorithm too. On the other hand, if these samples already belong to the tree, then they are rejected in \cite{kantaros2017sampling} while in our proposed algorithm, rewiring to these samples follows. A more detailed description of this first part can be found in Algorithm 1 and 7 in \cite{kantaros2017sampling}. The \textit{second part} of \cite{kantaros2017sampling}, that does not exist in our sampling-based algorithm, pertains to the optimization of the tree structure. Specifically, once the first part of \cite{kantaros2017sampling} has been executed for a user-specified number of iterations, then a rewiring-based algorithm follows that minimizes the cost of each node; see Algorithm 5 in \cite{kantaros2017sampling}. This algorithm is terminated once the set of edges of the tree stops changing and it is necessary to obtain asymptotic optimality of the method in \cite{kantaros2017sampling}. 

To compare Algorithm \ref{alg:plans} and \cite{kantaros2017sampling}, we executed Algorithm \ref{alg:plans} and the \textit{first part} of \cite{kantaros2017sampling} three times for a duration of about 15 hours per experiment. Notice that Algorithm \ref{alg:plans} detected 338 final states in 15.42 hours in average while \cite{kantaros2017sampling} found 108 final states in 15.65 hours in average. In every experiment, every time a new final state was detected, we computed the cost of the best prefix part, i.e., the minimum cost among all detected final states, and the size of the constructed tree. The evolution of the average cost with time is illustrated in Figure \ref{fig:compcost} for both algorithms. Observe in Figure \ref{fig:compcost} that the proposed algorithm can find the first final state in 30 minutes in average, while \cite{kantaros2017sampling} can do so in 3.5 hours, approximately. Also, observe in Figure \ref{fig:compcost} that the prefix part generated by \cite{kantaros2017sampling} has always a higher cost than the one synthesized by Algorithm \ref{alg:plans}. The reason is that the first part of \cite{kantaros2017sampling} only \textit{partially} optimizes the tree, as rewiring occurs only for samples that do not already exist in the tree but they are reachable from it. As discussed before, optimization of the tree structure is accomplished by the \textit{second part} of \cite{kantaros2017sampling} which, however, requires additional computational time. Figure \ref{fig:compcost} also shows how the cost of the best prefix part changes with time during the execution of the second part of \cite{kantaros2017sampling}. Observe that when \cite{kantaros2017sampling} has optimized the tree that was built until the detection of the first final state (the last triangle in the top green dashed line in Figure \ref{fig:compcost}), our proposed method has already detected 338 final states and has also constructed a much better prefix part (the last red rhombus in Figure \ref{fig:compcost}), in terms of the cost function $\hat{J}(\tau^{\text{pre}})$ defined in \eqref{eq:cost}.

Figure \ref{fig:comptrees}, shows how the average size of the tree changes with time during the execution of our proposed algorithm and the first part of \cite{kantaros2017sampling}. Observe that \cite{kantaros2017sampling} explores the state-space of the PBA faster than our proposed method, since at every iteration $n$ the first part of \cite{kantaros2017sampling} executes fewer operations due to the way the rewiring step is triggered, as discussed before. The time required to execute the second part of \cite{kantaros2017sampling} that optimizes the tree structure is also depicted in Figure \ref{fig:comptrees}, that shows that \cite{kantaros2017sampling} is much slower than our proposed method in synthesizing optimal plans. 

Finally, in Figure \ref{fig:compTreeIter}, we present the average size of the constructed trees per iteration $n$ during the execution of the proposed algorithm and the first part of \cite{kantaros2017sampling}. Observe that at any given iteration $n$ the proposed method has built a much larger tree than \cite{kantaros2017sampling}. The reason is that in \cite{kantaros2017sampling} samples are taken arbitrarily from the state space $\ccalQ_{\text{PTS}}$ and, therefore, they are not necessarily reachable from the constructed tree and, as a result, they are rejected. On the other hand, here, samples are drawn from reachable sets accelerating the construction of the tree with respect to iterations $n$, as shown in Proposition \ref{prop:growth}.

\subsubsection{Comparison with off-the-shelf model checkers}
Notice that the off-the-shelf model checker PRISM could not verify the considered LTL specification due to memory requirements. Specifically, PRISM could verify only a smaller part of the considered LTL formula that involved 6 robots, which was $\bar{\phi}=[\Box\Diamond(\pi_1^{\ell_5}\wedge\pi_2^{\ell_5})]\wedge[\Box\Diamond(\pi_2^{\ell_1}\wedge\pi_3^{\ell_1}\wedge\pi_4^{\ell_1})]\wedge[\Box\Diamond(\pi_4^{\ell_7}\wedge\pi_5^{\ell_7}\wedge\pi_6^{\ell_7})]$. In this case, the size of the state-space of the PBA was $|\mathcal{Q}_P|=9,765,625$. PRISM finished the model-checking process in 1.5 minutes while our method found a plan within 17 minutes. We also applied NuSMV to this problem that was able to generate a feasible plan within few seconds with cost equal to  $J(\tau)=J(\tau^{\text{pre}})+J(\tau^{\text{suf}})=336.1216+336.1216=672.2431$ meters while our method found a plan with cost $J(\tau)=J(\tau^{\text{pre}})+J(\tau^{\text{suf}})=298.5286+269.6571=568.1857$ meters. Notice that NuSMV can only generate a feasible plan and not the optimal plan, as our proposed algorithm does. The optimal control synthesis method described in Section \ref{sec:prelim} failed to design a plan that satisfies the considered LTL formula and so did the algorithm presented in \cite{kantaros15asilomar} due to excessive memory requirements.
  
\subsection{Case Study II}

In the second simulation study, we consider a team of $N=2$ robots. The transition system describing the motion of each robot is shown in Figure \ref{fig:workspace2}, and has $|\mathcal{Q}_i|=16$ states and $70$ transitions, including self-loops around each state. The assigned task is expressed in the following temporal logic formula: $\phi= \square\Diamond(\pi_1^{\ell_6}\wedge\Diamond(\pi_2^{\ell_{14}}))\wedge\square(\neg\pi_1^{\ell_9})\wedge\square(\pi_2^{\ell_{14}}\rightarrow \bigcirc(\neg\pi_2^{\ell_{14}}\ccalU\pi_1^{\ell_{4}}))\wedge (\Diamond\pi_2^{\ell_{12}}) \wedge (\square\Diamond\pi_2^{\ell_{10}})$
%\begin{align}\label{eq:task2}
%\phi= \square\Diamond(\pi_1^{\ell_6}\wedge\pi_2^{\ell_4})\wedge\neg(\pi_1^{\ell_7})\wedge(\neg p_2^{\ell_4}\mathcal{U} p_3^{\ell_4})\wedge (\Diamond\pi_3^{\ell_7}) \wedge (\square\Diamond\pi_2^{\ell_2}),
%\end{align}
where the respective NBA has $|\mathcal{Q}_B|=24$ states with $|\ccalQ_B^0|=1$, $|\ccalQ_B^F|=4$, and $163$ transitions. In words, this LTL-based task requires (a) robot 1 to visit location $\ell_6$, (b) once (a) is true robot 2 to visit location $\ell_{14}$, (c) conditions (a) and (b) to occur infinitely often, (d) robot 1 to always avoid location $\ell_9$, (e) once robot 2 visits location $\ell_{14}$, it should avoid this area until robot 1 visits location $\ell_{4}$, (f) robot 2 to visit location $\ell_{12}$ eventually, and (g) robot 2 to visit location $\ell_{10}$ infinitely often. In this simulation study, the state space of the PBA consists of $\Pi_{i=1}^N|\mathcal{Q}_i||\mathcal{Q}_B|=6,144$ states, which is small enough so that the method discussed in Section \ref{sec:prelim} can be used to find the optimal plan. The cost of the optimal plan is $J^*=14.6569$ meters. 

Algorithm \ref{alg:plans} was run for various values of the parameters $n_{\text{max}}^{\text{pre}}$ and $n_{\text{max}}^{\text{suf}}$. Observe in Figure \ref{fig:evolvCost} that as we increase $n_{\text{max}}^{\text{pre}}$ and $n_{\text{max}}^{\text{suf}}$, the cost of the resulting plans decreases and eventually the optimal plan is found, as expected due to Theorem \ref{prop:opt}. The number of detected final states and runtime for each case are also depicted in the same figure. PRISM verified that there exists a motion plan that satisfies the considered LTL formula in few seconds and NuSMV in less than 1 second. However, neither of them can synthesize the optimal motion plan that satisfies the considered LTL task. For instance, the cost of the plan generated by NuSMV is $30.8995$ meters while our algorithm can find the optimal plan with cost $J^*=14.6569$, as shown in Figure \ref{fig:evolvCost}. Figure \ref{fig:rejectSim2} depicts the number of rejected states with respect to the iterations $n$ of Algorithm \ref{alg:tree}. Notice that, as in the previous simulation study, at every iteration $n$ there is at least one state that is added to the tree.
%UNCOMMENT
\begin{figure}[t]
\centering
  \includegraphics[width=0.6\linewidth]{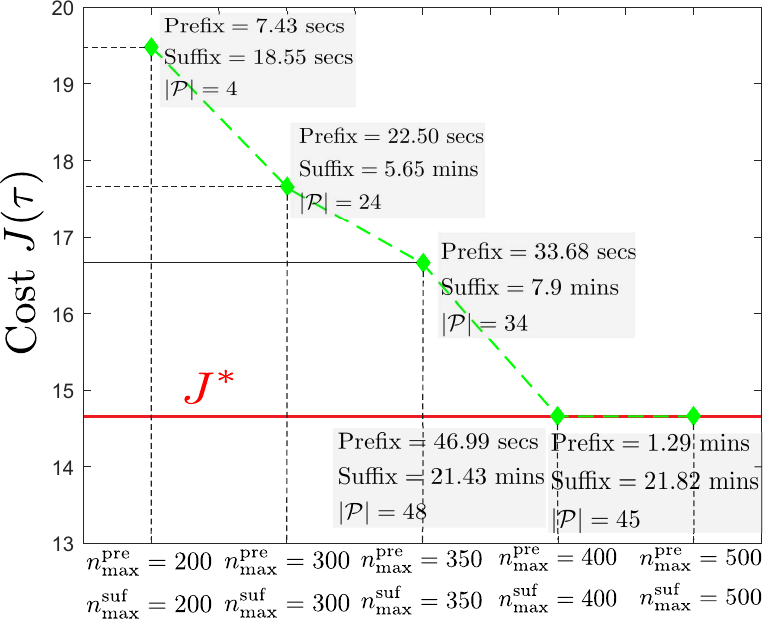}%costIter.pdf
  \caption{Evolution of the cost $J(\tau)$ of the optimal motion plan $\tau$ for various maximum numbers of iterations, $n_{\text{max}}^{\text{pre}}$ and $n_{\text{max}}^{\text{suf}}$, for Algorithm \ref{alg:tree}. The time required for the construction of the optimal prefix and suffix part along with the number of detected final states for each case are included in the gray colored box. The red line denotes the optimal cost $J^*=14.6569$.}
  \label{fig:evolvCost}
\end{figure} 
    
\section{Conclusion}\label{sec:concl}
In this paper we proposed a sampling-based control synthesis algorithm for multi-robot systems under global linear temporal logic (LTL) formulas. Existing planning approaches under global temporal goals rely on graph search techniques applied to a synchronous product automaton constructed among the robots. In this paper, we proposed a new sampling-based algorithm to build incrementally trees that approximated the state-space and transitions of the synchronous product automaton increasing in this way significantly scalability of our method compared to existing model-checking approaches. Moreover, we showed that the proposed algorithm is probabilistically complete and asymptotically optimal. Finally, we presented numerical experiments that show that it can be used to synthesize optimal plans from product automata with billions of states, which was not possible using standard optimal control synthesis algorithms or off-the-shelf model checkers.

\appendices
\section{Proofs of Lemmas}\label{sec:lemmas1}
%In this Appendix, we provide lemmas that are required to show that the centralized Algorithm \ref{alg:plans} is probabilistically complete and asymptotically optimal. The following results rely on the second Borel-Cantelli lemma \cite{grimmett2001probability} stated as follows.

%\begin{lem}[Borel-Cantelli \cite{grimmett2001probability}]\label{lem:bc}
%Consider a sequence of independent events $A=\{A^n\}_{n=1}^{\infty}$. Then if $\sum_{n=1}^{\infty}\mathbb{P}(A^n)=\infty$ then $\mathbb{P}(\limsup_{n\to\infty} A^n)=1$, i.e., the events $A^n$ occur infinitely often meaning that $A^n$ is true
%for an infinite number of indices $n\in\mathbb{N}_{+}$.
%\end{lem}
\subsection{Proof of Lemma \ref{lem:qrand}}
The proof of this results relies on Lemma \ref{lem:bc}. Let $A^{\text{rand},n+k}(q_P)=\{q_P^{\text{rand},n+k}=q_P\}$, with $k\in\mathbb{N}$, denote the event that at iteration $n+k$ of Algorithm \ref{alg:tree} the state $q_P\in\ccalV_T^n$ is selected by the function \texttt{Sample} to be the node $q_P^{\text{rand},n+k}$ [line \ref{s:line2}, Alg. \ref{alg:sample}]. Also, let $\mathbb{P}(A^{\text{rand},n+k}(q_P))$ denote the probability of this event, i.e., $\mathbb{P}(A^{\text{rand},n+k}(q_P))=f_{\text{rand}}(q_P|\ccalV_T^{n+k})$. 

Next, define the infinite sequence of events $A^{\text{rand}}=\{A^{\text{rand},n+k}(q_P)\}_{k=0}^{\infty}$, for a given node $q_P\in\ccalV_T^n$. In what follows, we show that the series $\sum_{k=0}^{\infty}\mathbb{P}(A^{\text{rand},n+k}(q_P))$ diverges and then we complete the proof by applying Lemma \ref{lem:bc}. 
Recall first that the size of $\ccalV_T^{n+k}$ cannot grow arbitrarily large, since it holds that $|\ccalV_T^{n+k}|\leq|\ccalQ_P|<\infty$, for all $k\in\mathbb{N}$. Also, by Assumption \ref{frand}(ii), we have that for a given $q_P\in\ccalV_T^n$, the probability $f_{\text{rand}}(q_P|\ccalV_T^{n+k})$ decreases monotonically with respect to $|\ccalV_T^{n+k}|$. From these two observations we deduce that
\begin{equation}\label{ineq1}
\mathbb{P}(A^{\text{rand},n+k}(q_P))=f_{\text{rand}}(q_P|\ccalV_T^{n+k})\geq f_{\text{rand}}(q_P|\ccalQ_P), 
\end{equation} 
for all $k\in\mathbb{N}$, where $f_{\text{rand}}(q_P|\ccalQ_P)$ is the probability assigned to selecting the state $q_P\in\ccalV_T^n$ as $q_P^{\text{rand},n+k}$ when $\ccalV_T^{n+k}=\ccalQ_P$. Note that $f_{\text{rand}}(q_P|\ccalQ_P)$ is a strictly positive term due to Assumption \ref{frand}(i). Also, $f_{\text{rand}}(q_P|\ccalQ_P)$ is constant, by Assumption \ref{frand}(ii), since $\ccalQ_P$ is a fixed set.
%Notice also in \eqref{ineq1} that we consider a state $q_P$ that belongs to both $\ccalV_T^{n+k}$ and $\ccalQ_P$, rendering \eqref{ineq1} valid. Specifically, since $q_P\in\ccalV_T^n$, it holds that $q_P\in\ccalV_T^{n+k}$, for all $k\in\mathbb{N}$, and, consequently, for $k=k^{*}$ at which $\ccalV_T^{n+k^{*}}=\ccalQ_P$, we get that $\mathbb{P}(A^{\text{rand},n+k}(q_P))=f_{\text{rand}}(q_P|\ccalV_T^{n+k})>0$ and $f_{\text{rand}}(q_P|\ccalQ_P)>0$, by Assumption \ref{frand}(i).

Next, since \eqref{ineq1} holds for all $k\in\mathbb{N}$, we have that 
\begin{equation}\label{ineq2}
\sum_{k=0}^{\infty}\mathbb{P}(A^{\text{rand},n+k}(q_P))\geq\sum_{k=0}^{\infty}f_{\text{rand}}(q_P|\ccalQ_P).
\end{equation}
Since for any state $q_P\in\ccalQ_P$, we have that $f_{\text{rand}}(q_P|\ccalQ_P)$ is a strictly positive constant term, the infinite $\sum_{k=0}^{\infty}f_{\text{rand}}(q_P|\ccalQ_P)$ diverges. Then, %by the comparison test \cite{marsden1993elementary}, 
we conclude that 
\begin{equation}\label{res1}
\sum_{k=0}^{\infty}\mathbb{P}(A^{\text{rand},n+k}(q_P))=\infty.
\end{equation}

Combining \eqref{res1} and the fact that the events $A^{\text{rand},n+k}(q_P)$ are independent by Assumption \ref{frand}(iii), we get that $\mathbb{P}(\limsup_{k\to\infty} A^{\text{rand},n+k}(q_P))=1,$ due to Lemma \ref{lem:bc}. In other words, the events $A^{\text{rand},n+k}(q_P)$ occur infinitely often, for all $q_P\in\ccalV_T^n$. This equivalently means that for every node $q_P\in\ccalV_T^n$, for all $n\in\mathbb{N}_{+}$, there exists an infinite subsequence $\mathcal{K}\subseteq \mathbb{N}$ so that for all $k\in\mathcal{K}$ it holds $q_P^{\text{rand},n+k}=q_P$, completing the proof.\footnote{Note that the subsequence $\ccalK$ is different across the nodes $q_P\in\ccalV_T^n$.}

\begin{rem}[Lemma \ref{lem:qrand}]
The result shown in Lemma \ref{lem:qrand} holds even if Assumption \ref{frand}(ii) does not hold, i.e., if the density function $f_{\text{rand}}(q_P|\ccalV_T^n)$ changes every iteration $n$ for a fixed set $\ccalV_T^n$ and a fixed node $q_P\in\ccalV_T^n$, and even if it does not decrease monotonically with the cardinality of $\ccalV_T^n$,
% as long as (i) $f^{n_1}_{\text{rand}}(q_P|\ccalV_T^{n_1})\leq f^{n_2}_{\text{rand}}(q_P|\ccalV_T^{n_2})$, for all iterations $n_1,~n2$ of Algorithm \ref{alg:tree} at which it holds $|\ccalV_T^{n_1}|\leq|\ccalV_T^{n_2}|$, and (ii)
as long as this varying density function $f_{\text{rand}}^n(q_P|\ccalV_T^n)$ is bounded below by a sequence $g^{n}(q_P|\ccalV_T^n)$, such that $\sum_{n=1}^{\infty}g^n(q_P|\ccalV_T^n)=\infty$, for all $q_P\in\ccalV_T^n$. This will ensure that $\sum_{k=0}^{\infty}\mathbb{P}(A^{\text{rand},n+k}(q_P))=\sum_{k=0}^{\infty} f_{\text{rand}}^{n+k}(q_P|\ccalV_T^{n+k})\geq\sum_{k=0}^{\infty}g^{n+k}(q_P|\ccalV_T^{n+k})=\infty$, which replaces \eqref{ineq2} and still yields \eqref{res1}.
\label{rem:lem1}
\end{rem}

\subsection{Proof of Lemma \ref{lem:qnew}}
This proof relies on Lemma \ref{lem:bc} and resembles the proof of Lemma \ref{lem:qrand}. Let $q_P^{\text{rand},n}\in\ccalV_T^n$ and define the infinite sequence of events $A^{\text{new}}=\{A^{\text{new},n+k}(q_{\text{PTS}})\}_{k=0}^{\infty}$, for any given state $q_{\text{PTS}}\in\ccalR_{\text{PTS}}(q_{\text{PTS}}^{\text{rand},n})$, where $A^{\text{new},n+k}(q_{\text{PTS}})=\{q_{\text{PTS}}^{\text{new},n+k}=q_{\text{PTS}}\}$, for $k\in\mathbb{N}$, denotes the event that at iteration $n+k$ of Algorithm \ref{alg:tree} the state $q_{\text{PTS}}\in\ccalR_{\text{PTS}}(q_{\text{PTS}}^{\text{rand},n})$ is selected by the function \texttt{Sample} to be the node $q_{\text{PTS}}^{\text{new},n+k}$ [line \ref{s:line5}, Alg. \ref{alg:sample}], given a state $q_P^{\text{rand},n+k}=(q_{\text{PTS}}^{\text{rand},n+k},q_B^{\text{rand},n+k})\in\ccalV_T^{n+k}$.\footnote{Recall that the reachable set $\ccalR_{\text{PTS}}(q_{\text{PTS}}^{\text{rand},n})$ defined in \eqref{reachable1} remains the same for all iterations $n$, for a given state
$q_{\text{PTS}}^{\text{rand},n}$.} Moreover, let $\mathbb{P}(A^{\text{new},n+k}(q_{\text{PTS}}))$ denote the probability of this event, i.e., $\mathbb{P}(A^{\text{new},n+k}(q_{\text{PTS}}))=f_{\text{new}}(q_{\text{PTS}}|q_{\text{PTS}}^{\text{rand},n+k})$. 

Now, consider those iterations $n+k$ with $k\in\mathcal{K}$ such that $q_P^{\text{rand},n+k}=q_P^{\text{rand},n}$ by Lemma \ref{lem:qrand}. We will show that the series $\sum_{k\in\ccalK}\mathbb{P}(A^{\text{new},n+k}(q_{\text{PTS}}))$ diverges and then we will use Lemma \ref{lem:bc} to show that $q_{\text{PTS}}\in\ccalR_{\text{PTS}}(q_{\text{PTS}}^{\text{rand},n+k})$ will be selected infinitely often to be node $q_{\text{PTS}}^{\text{new},n+k}$.

Since $q_{\text{PTS}}\in\ccalR_{\text{PTS}}(q_{\text{PTS}}^{\text{rand},n+k})$ we have that $\mathbb{P}(A^{\text{new},n+k}(q_{\text{PTS}}))=f_{\text{new}}(q_{\text{PTS}}|q_{\text{PTS}}^{\text{rand},n+k})$ is a strictly positive constant for all $k\in\ccalK$, by Assumption \ref{fnew}(i) and \ref{fnew}(ii). %Notice that it is possible that $\mathbb{P}(A^{\text{new},n+k}(q_{\text{PTS}}))=f_{\text{new}}(q_{\text{PTS}}|q_P^{\text{rand},n+k})>0$ even if $q_P^{\text{rand},n+k}\neq q_P'$, if $q_{\text{PTS}}\in\ccalR_{\text{PTS}}(q_{\text{PTS}}^{\text{rand},n+k})$. However, hereafter, we consider the worst case scenario that the state $q_{\text{PTS}}$ belongs only to the reachable set $\ccalR_{\text{PTS}}(q_{\text{PTS}}^{\text{rand},n})$
%Also, the probability $\mathbb{P}(A^{\text{new},n+k}(q_{\text{PTS}}))>0$ remains the same for all iterations $n+k$ of Algorithm \ref{alg:tree} at which it holds $q_P^{\text{rand},n+k}=q_P'\in\ccalV_T^n$, by Assumption \ref{fnew}-(ii). 
Therefore, we have that $\sum_{k\in\ccalK}\mathbb{P}(A^{\text{new},n+k}(q_{\text{PTS}}))$ diverges, since it is an infinite sum of a strictly positive constant term. Using this result along with the fact that the events $A^{\text{new},n+k}(q_{\text{PTS}})$ are independent, by Assumption \ref{fnew}(iii), we get that $\mathbb{P}(\limsup_{k\to\infty} A^{\text{new},n+k}(q_{\text{PTS}}))=1,$ 
due to Lemma \ref{lem:bc}. In words, this means that the events $A^{\text{new},n+k}(q_{\text{PTS}})$ for $k\in\mathcal{K}$ occur infinitely often. Thus, for every node $q_{\text{PTS}}\in\ccalR_{\text{PTS}}(q_{\text{PTS}}^{\text{rand},n})$ for all $n\in\mathbb{N}_{+}$, there exists an infinite subsequence $\mathcal{K}' \subseteq \mathcal{K}$ so that for all $k\in\mathcal{K}'$ it holds $q_{\text{PTS}}^{\text{new},n+k}=q_{\text{PTS}}$, completing the proof.

\begin{rem}[Lemma \ref{lem:qnew}]
Lemma \ref{lem:qnew} holds even if Assumption \ref{fnew}(ii) does not hold, i.e., if for any given node $q_P^{\text{rand},n}\in\ccalV_T^n$, the density function $f_{\text{new}}$ changes with iterations $n+k$, where $k\in\ccalK$, for which $q_P^{\text{rand},n}=q_P^{\text{rand},n+k}$, as long as it is bounded below by a sequence $h^{n+k}(q_{\text{PTS}}|q_{\text{PTS}}^{\text{rand},n})$, for all $k\in\ccalK$, such that $\sum_{k\in\ccalK}h^{n+k}(q_{\text{PTS}}|q_{\text{PTS}}^{\text{rand},n})=\infty$, for all $q_{\text{PTS}}\in\ccalR_{\text{PTS}}(q_{\text{PTS}}^{\text{rand},n})$. %Recall from the proof of Lemma \ref{lem:qnew} that for the set $\ccalK$ that collects all indices at which $q_P^{\text{rand}}=q_P^{\text{rand},n}$ we have that $|\ccalK|=\infty$. For example, the sequences $g^n$ and $h^n$ can be selected so that they are strictly positive and constant.
\label{rem:lem2}
\end{rem}

\subsection{Proof of Corollary \ref{cor:qPnew}}
%This result is due to Lemma \ref{lem:qnew}. Before proving this result, 
Recall that a state $q_P=(q_{\text{PTS}},q_B)$ belongs to $\ccalR_P(q_P^{\text{rand},n})$ if $q_P^{\text{rand},n}\rightarrow_P q_P$, i.e., if (i) $q_{\text{PTS}}^{\text{rand},n}\rightarrow_{\text{PTS}}q_{\text{PTS}}$ and (ii) $q_B^{\text{rand},n}\stackrel{L(q_{\text{PTS}}^{\text{rand},n})}{\longrightarrow_B}q_B$. Then, to prove this result, it suffices to show that all states $q_P$ that satisfy both conditions (i) and (ii) are sampled infinitely often, which is a direct result from Lemma \ref{lem:qnew}.

Specifically, observe first that, due to Lemma \ref{lem:qnew}, for all states $q_{\text{PTS}}\in\ccalR_{\text{PTS}}(q_{\text{PTS}}^{\text{rand},n})$, i.e., for all states that satisfy condition (i), there exists an infinite number of iterations $n+k$, at which they will selected to be the nodes $q_{\text{PTS}}^{\text{new},n+k}$ with probability 1, for all $n\in\mathbb{N}$. Second, given a state $q_{\text{PTS}}^{\text{new},n}\in\ccalR_{\text{PTS}}(q_{\text{PTS}}^{\text{rand},n})$, the states $q_P^{\text{new},n}=(q_{\text{PTS}}^{\text{new},n},\ccalQ_B(b))$, for all $b\in\{1,\dots,|\ccalQ_B|\}$ are created, by construction of Algorithm \ref{alg:tree}. Therefore, given a state $q_{\text{PTS}}^{\text{new},n}$, if there exists $b\in\{1,\dots,|\ccalQ_B|\}$ such that $q_B^{\text{rand},n}\stackrel{L(q_{\text{PTS}}^{\text{rand},n})}{\rightarrow_B}\ccalQ_B(b)$, then the state $q_P^{\text{new},n}=(q_{\text{PTS}}^{\text{new},n},\ccalQ_B(b))\in\ccalR_P(q_P^{\text{rand},n})$ that satisfies (ii) will be sampled/constructed infinitely often.\footnote{Recall that the reachable set $\ccalR_{P}(q_{P}^{\text{rand},n})$ defined in \eqref{reachable1} remains the same for all iterations $n$, for a given state $q_{P}^{\text{rand},n}$.} Thus, for all states $q_{P}\in\ccalR_{P}(q_{P}^{\text{rand},n})$ there exists an infinite number of iterations $n+k$, with $k\in\ccalK'\subseteq\ccalK$, at which they will be selected by Algorithm \ref{alg:sample} to be the node $q_{P}^{\text{new},n+k}$ completing the proof.
%Combining these two observations we conclude that for all states $q_P=(q_{\text{PTS}},q_B)$ that satisfy (i) $q_{\text{PTS}}^{\text{rand},n}\rightarrow_{\text{PTS}}q_{\text{PTS}}$ and (ii) $q_B^{\text{rand},n}\stackrel{L(q_{\text{PTS}}^{\text{rand},n})}{\longrightarrow_B}q_B$, i.e., for all states $q_P\in\ccalR_P(q_P^{\text{rand},n})$, there exists an infinite number of iterations $n+k$, at which it holds that $q_P^{\text{new},n+k}=q_P'$ completing the proof.

\subsection{Proof of Lemma \ref{lem:reach1}}
First note that \eqref{eq:prob2} trivially holds for all states $q_P^{\text{rand},n}$ that satisfy $\ccalR_P(q_P^{\text{rand},n})=\emptyset$. Hence, in what follows we consider only states $q_P^{\text{rand},n}\in\ccalV_T^n$ that satisfy $\ccalR_P(q_P^{\text{rand},n})\neq\emptyset$. The proof of this result relies on Corollary \ref{cor:qPnew}. Specifically, recall that, due to Corollary \ref{cor:qPnew}, for all states $q_P\in\ccalR_{P}(q_P^{\text{rand},n})$, there exists an infinite number of iterations $n+k$, $k\in\ccalK'\subseteq\ccalK$, such that $q_P^{\text{new},n+k}=q_P$. This means that with probability 1, there exists an iteration $n+k$ of Algorithm \ref{alg:tree} at which the state $q_P^{\text{new},n+k}=q_P$, will be sampled. %Hereafter, we assume that $n+k$ is the first iteration at which 

Since this iteration $n+k$ satisfies $q_P^{\text{new},n+k}\in\mathcal{R}_P(q_P^{\text{rand},n})$, we get that $q_P^{\text{rand},n}\in\ccalR^{\rightarrow}_{\ccalV_T^{n+k}}(q_P^{\text{new},n+k})\subseteq\ccalV_T^{n+k}$. This follows from the definition of $\ccalR^{\rightarrow}_{\ccalV_T^{n+k}}(q_P^{\text{new},n+k})$ in \eqref{eq:Stonew} and the fact that since $q_P^{\text{rand},n}$ belongs to $\ccalV_T^n$ it also belongs to $\ccalV_T^{n+k}$. %by Assumption \ref{frand}-(i). 
Therefore, $\ccalR^{\rightarrow}_{\ccalV_T^{n+k}}(q_P^{\text{new},n+k})\neq\emptyset$, which means that the state $q_P^{\text{new},n+k}=q_P\in\ccalR_P(q_P^{\text{rand},n})$ will be added to the tree at iteration $n+k$ by construction of Algorithm \ref{alg:extend}; see line \ref{alg3:line2} in Algorithm \ref{alg:extend}. We conclude, that for all states $q_P\in\ccalR_{P}(q_P^{\text{rand},n})$ there exists a subsequent iteration $n+k$ at which they will be added to the tree with probability 1.\footnote{If the states $q_P\in\ccalR_{P}(q_P^{\text{rand},n})$ already belong to the tree then the rewiring step follows.} In mathematical terms, this result can be written as $\lim_{k\rightarrow\infty} \mathbb{P}\left(\{\mathcal{R}_P(q_P^{\text{rand},n})\subseteq\mathcal{V}_T^{n+k}\}\right)=1$ completing the proof. %\footnote{\blue{In seems to me that the second part of the proof of Proposition \ref{prop:growth} and this proof, essentially prove the same thing but maybe in a slightly different way. So we can possibly delete the proof of Proposition \ref{prop:growth} to save some space?}}

\subsection{Proof of Corollary \ref{cor:reach2}}
This result is due to Lemmas \ref{lem:qrand} and \ref{lem:reach1}. Specifically, by Lemma \ref{lem:reach1}, we have that 
\begin{equation}\label{eq:temp1}
\lim_{k\rightarrow\infty} \mathbb{P}\left(\{\mathcal{R}_P(q_P^{\text{rand},n})\subseteq\mathcal{V}_T^{n+k}\}\right)=1,
\end{equation}
for a given state $q_P^{\text{rand},n}\in\ccalV_T^n\subseteq\ccalV_T^{n+k}$ and any iteration $n\in\mathbb{N}_{+}$. Also, by Lemma \ref{lem:qrand}, we have that every state $q_P\in\ccalV_T^n$ will be selected infinitely often to be the node $q_P^{\text{rand},n+k}$, as $k\to\infty$. Therefore, we get that \eqref{eq:temp1} holds for all states $q_P\in\ccalV_T^n$, i.e., 
%
%$\lim_{k\rightarrow\infty} \mathbb{P}\left(\{\mathcal{R}_P(q_P)\subseteq\mathcal{V}_T^{n+k}\}\right)=1,~\forall q_P\in\ccalV_T^n$. Since the latter equation holds for any iteration $n\in\mathbb{N}_{+}$, we can rewrite it as $\lim_{n\rightarrow\infty} \mathbb{P}\left(\{\mathcal{R}_P(q_P)\subseteq\mathcal{V}_T^{n}\}\right)=1,~\forall q_P\in\ccalV_T^n$ completing the proof. 
\begin{equation}\label{eq:tempR}
\lim_{k\rightarrow\infty} \mathbb{P}\left(\{\mathcal{R}_P(q_P)\subseteq\mathcal{V}_T^{n+k}\}\right)=1,~\forall q_P\in\ccalV_T^n.
\end{equation}
Since \eqref{eq:tempR} holds for any iteration $n\in\mathbb{N}_{+}$, we can rewrite \eqref{eq:tempR} as $\lim_{n\rightarrow\infty} \mathbb{P}\left(\{\mathcal{R}_P(q_P)\subseteq\mathcal{V}_T^{n}\}\right)=1,~\forall q_P\in\ccalV_T^n$ completing the proof.
%
%\begin{equation}\label{eq:tempR1}
%\lim_{n\rightarrow\infty} \mathbb{P}\left(\{\mathcal{R}_P(q_P)\subseteq\mathcal{V}_T^{n}\}\right)=1,~\forall q_P\in\ccalV_T^n,
%\end{equation}
%completing the proof.

%$\ccalR_{P}(q_P)=\{q_{P}'\in\ccalQ_P|q_P\rightarrow_{P}q_{P}'\}$
\subsection{Proof of Corollary \ref{cor:qPnew2}}
The proof relies on Corollary \ref{cor:qPnew} and Lemma \ref{lem:qrand}. From Lemma \ref{lem:qrand}, we have that for every state $q_P'\in\ccalV_T^n$ there exists an infinite number of iterations $n+k$, with $k\in\ccalK$, at which the state $q_{P}'$ is selected by Algorithm \ref{alg:sample} to be the node $q_{P}^{\text{rand},n+k}$. Also, for any iteration $n$ and for any state $q_P^{\text{rand},n}$, we know from Corollary \ref{cor:qPnew} that for every state $q_{P}\in\ccalR_{P}(q_{P}^{\text{rand},n})$ there exists an infinite number of subsequent iterations $n+k$, with $k\in\ccalK'\subseteq\ccalK$, at which the state $q_{P}$ is selected by Algorithm \ref{alg:sample} to be the node $q_{P}^{\text{new},n+k}$, given a node  $q_P^{\text{rand},n}$.  Combining these two results, we get that for every state $q_{P}\in\ccalR_{P}(q_{P}')$, and for all $q_P'\in\ccalV_T^n$, there exists an infinite number of subsequent iterations $n+k$, with $k\in\ccalK'$, at which the state $q_{P}$ is selected to be the node $q_{P}^{\text{new},n+k}$. Equivalently, this means that for every state $q_P\in\cup_{q_P'\in\ccalV_T^n}\ccalR_{P}(q_P')$ there exists an infinite number of iterations $n+k$, with $k\in\ccalK'$, at which the state $q_{P}$, is selected to be the node $q_{P}^{\text{new},n+k}$. Then, it suffices to show that the set $\ccalV_T^n$ is a subset of the set $\cup_{q_P'\in\ccalV_T^n}\ccalR_{P}(q_P')$. This would mean that for any state $q_P\in\ccalV_T^n$, there exists and infinite number of iterations $n+k$, with $k\in\ccalK'$, at which the state $q_{P}\in\ccalV_T^n$, is selected by Algorithm \ref{alg:sample} to be the node $q_{P}^{\text{new},n+k}$.

%In particular, from Corollary \ref{cor:qPnew}, we know that for every state $q_{P}\in\ccalR_{P}(q_{P}^{\text{rand},n})$ there exists an infinite number of subsequent iterations $n+k$, with $k\in\ccalK'\subseteq\ccalK$, at which the state $q_{P}$ is selected by Algorithm \ref{alg:sample} to be the node $q_{P}^{\text{new},n+k}$, given a node  $q_P^{\text{rand},n}$. From Lemma \ref{lem:qrand}, we have that for every state $q_P'\in\ccalV_T^n$ there exists an infinite number of iterations $n+k$, with $k\in\ccalK$, at which the state $q_{P}'$ is selected by Algorithm \ref{alg:sample} to be the node $q_{P}^{\text{rand},n+k}$. \red{Combining these two results, we get that for every state $q_{P}\in\ccalR_{P}(q_{P}')$, and for all $q_P'\in\ccalV_T^n$, there exists an infinite number of subsequent iterations $n+k$, with $k\in\ccalK'$, at which the state $q_{P}$ is selected to be the node $q_{P}^{\text{new},n+k}$. Equivalently, this means that for every state $q_P\in\cup_{q_P'\in\ccalV_T^n}\ccalR_{P}(q_P')$ there exists an infinite number of iterations $n+k$, with $k\in\ccalK'$, at which the state $q_{P}$, is selected to be the node $q_{P}^{\text{new},n+k}$.} Then, it suffices to show that the set $\ccalV_T^n$ is a subset of the set $\cup_{q_P'\in\ccalV_T^n}\ccalR_{P}(q_P')$. This would mean that for any state $q_P\in\ccalV_T^n$, there exists and infinite number of iterations $n+k$, with $k\in\ccalK'$, at which the state $q_{P}\in\ccalV_T^n$, is selected by Algorithm \ref{alg:sample} to be the node $q_{P}^{\text{new},n+k}$. 

To show that $\ccalV_T^n\subseteq\cup_{q_P'\in\ccalV_T^n}\ccalR_{P}(q_P')$, we will show that if $q\in\ccalV_T^n$ then $q\in\cup_{q_P'\in\ccalV_T^n}\ccalR_{P}(q_P')$. Consider a node $q\in\ccalV_T^n$. Then, by construction of the tree, we have that there exists another node $q'\in\ccalV_T^n$, such that $\texttt{parent}(q)=q'\in\ccalV_T^n$, which means $q\in\ccalR_P(q')$. Observe that $\ccalR_P(q')\subseteq\cup_{q_P'\in\ccalV_T^n}\ccalR_{P}(q_P')$, since $q'\in\ccalV_T^n$ by assumption. Therefore, we have that $q\in\ccalR_P(q')\subseteq\cup_{q_P'\in\ccalV_T^n}\ccalR_{P}(q_P')$, i.e., $q\in\cup_{q_P'\in\ccalV_T^n}\ccalR_{P}(q_P')$ completing the proof.
%
%the Recall that the set $\ccalR^{\infty}_P(q_P^r)$, defined in \eqref{Rinf} as $\ccalR^{\infty}_P(q_P^r)=\ccalR(...\ccalR(\ccalR(q_P^r)))$ all the states that can be reached by the root $q_P^r$ through a multi-hop path in $\ccalQ_P$. Then, due to Theorem \ref{prop:compl} (see also \eqref{eq:prob3}), we have that $\ccalV_T^n\subseteq\ccalR^{\infty}_P(q_P^r)$. Hence, we get that for any state $q_P\in\ccalV_T^n$, there exists and infinite number of iterations $n+k$, with $k\in\mathbb{N}$, at which the state $q_{P}\in\ccalV_T^n$, is selected by Algorithm \ref{alg:sample} to be the node $q_{P}^{\text{new},n+k}$ completing the proof.

\bibliographystyle{IEEEtran}
\bibliography{YK_bib}
\end{document}